\documentclass[twoside,11pt]{article}

\usepackage{blindtext}

\usepackage[abbrvbib,preprint]{jmlr2e}

\usepackage{smile}
\usepackage{algorithm} 
\usepackage{algorithmic}  
\usepackage[algo2e,ruled]{algorithm2e} 
\usepackage[inline]{enumitem}
\usepackage[T1]{fontenc}
\usepackage{mathtools}
\usepackage{physics}
\usepackage{tablefootnote}
\usepackage{xcolor}
\usepackage{subcaption}

\usepackage[capitalize,noabbrev]{cleveref}
\newtheorem{assumption}{Assumption}
\Crefname{assumption}{Assumption}{Assumptions}

\newcommand*\prob{\PP}
\newcommand*\dist{d}
\DeclareMathOperator{\ReLU}{ReLU}
\DeclareMathOperator{\Conv}{Conv}
\DeclareMathOperator{\Area}{Area}

\DeclareMathOperator{\supp}{supp}

\newcommand{\pitheta}[1]{{\pi_{#1}}}

\newcommand{\lossQ}{\cL_{\mathrm{critic}}}
\newcommand{\lossQhat}{\hat{\cL}_{\mathrm{critic}}}
\newcommand{\losspi}[1]{\cL_{\mathrm{actor}}}
\newcommand{\losspihat}[1]{\hat{\cL}_{\mathrm{actor}}}

\newcommand{\chisqr}[2]{\chi^2( #1,#2 )}

\usepackage{lastpage}
\jmlrheading{??}{2024}{1-\pageref{LastPage}}{?/??; Revised ?/??}{?/??}{??-????}{Zhenghao Xu, Xiang Ji, Minshuo Chen, Mengdi Wang, and Tuo Zhao}

\ShortHeadings{Sample Complexity of NPMD on Low-Dimensional Manifolds}{Xu, Ji, Chen, Wang and Zhao}
\firstpageno{1}

\begin{document}

\title{Sample Complexity of Neural Policy Mirror Descent for Policy Optimization on Low-Dimensional Manifolds}

\author{\name Zhenghao Xu \email zhenghaoxu@gatech.edu \\ 
        \addr H. Milton Stewart School of Industrial and Systems Engineering\\ 
        Georgia Institute of Technology\\
        Atlanta, GA 30332, USA
        \AND
        \name Xiang Ji \email xiangj@princeton.edu \\ 
        \addr Department of Electrical and Computer Engineering\\ 
        Princeton University\\ 
        Princeton, NJ 08544, USA
        \AND
        \name Minshuo Chen \email minshuochen@princeton.edu \\ 
        \addr Department of Electrical and Computer Engineering\\ 
        Princeton University\\ 
        Princeton, NJ 08544, USA
        \AND
        \name Mengdi Wang \email mengdiw@princeton.edu \\ 
        \addr Department of Electrical and Computer Engineering\\ 
        Princeton University\\ 
        Princeton, NJ 08544, USA
        \AND
        \name Tuo Zhao \email tourzhao@gatech.edu \\
        \addr H. Milton Stewart School of Industrial and Systems Engineering\\ 
        Georgia Institute of Technology\\
        Atlanta, GA 30332, USA
       }

\editor{My editor}

\maketitle

\begin{abstract}

Policy gradient methods equipped with deep neural networks have achieved great success in solving high-dimensional reinforcement learning (RL) problems. However, current analyses cannot explain why they are resistant to the curse of dimensionality. 
In this work, we study the sample complexity of the neural policy mirror descent (NPMD) algorithm with deep convolutional neural networks (CNN). Motivated by the empirical observation that many high-dimensional environments have state spaces possessing low-dimensional structures, such as those taking images as states, we consider the state space to be a $d$-dimensional manifold embedded in the $D$-dimensional Euclidean space with intrinsic dimension $d\ll D$. 
We show that in each iteration of NPMD, both the value function and the policy can be well approximated by CNNs. The approximation errors are controlled by the size of the networks, and the smoothness of the previous networks can be inherited. 
As a result, by properly choosing the network size and hyperparameters, NPMD can find an $\epsilon$-optimal policy with $\Tilde{O}(\epsilon^{-\frac{d}{\alpha}-2})$ samples in expectation, where $\alpha\in(0,1]$ indicates the smoothness of environment. 
Compared to previous work, our result exhibits that NPMD can leverage the low-dimensional structure of state space to escape from the curse of dimensionality, explaining the efficacy of deep policy gradient algorithms. 
\end{abstract}


\section{Introduction}\label{sec:introduction}

Deep Reinforcement Learning (DRL) is a popular approach for solving complex decision-making problems in various domains. 
DRL methods, especially policy-based ones including DDPG \citep{lillicrap2015continuous}, TRPO \citep{schulman2015trust}, and PPO \citep{schulman2017proximal}, are able to handle high-dimensional state space efficiently by leveraging function approximation with neural networks. For instance, in Atari games \citep{brockman2016openai}, the states are images of size $210\times 160$ with RGB color channels, resulting in a continuous state space of dimension $100,800$, to which tabular algorithms such as policy iteration \citep{puterman1994markov} and policy mirror descent (PMD, \citealt{lan2023policy}) are not applicable. Surprisingly, such a high dimension does not seem to significantly impact the efficacy of the aforementioned DRL algorithms. 

Despite the empirical success of these DRL methods in high dimensions, there currently exist no satisfactory results in theory that can explain the reason behind them. 
Most of the existing works about function approximation in RL focus on linear function class \citep{agarwal2021theory,alfano2022linear,yuan2023linear}. They assume the value function and the policy can be well approximated by linear functions of features representing states and actions \citep{jin2020provably}, which is restrictive and requires feature engineering. 
One way to relax such a linearity assumption is to consider the reproducing kernel Hilbert space (RKHS) which allows nonlinear function approximation \citep{agarwal2020pc,yang2020provably} through random features \citep{rahimi2007random}. However, the commonly used reproducing kernels such as the Gaussian radial basis function and randomized ReLU kernel suffer from the curse of dimensionality in sample complexity without additional smoothness assumptions \citep{bach2017breaking,yehudai2019power,hsu2021approximation}. 

Moreover, some researchers study neural network approximation in the neural tangent kernel (NTK) regime, which is equivalent to an RKHS \citep{jacot2018neural,liu2019neural,wang2019neural,cayci2022finite,alfano2023novel}. Consequently, these results inherit the curse of dimensionality from general RKHS.
Some other works investigate neural network approximation from a non-parametric perspective \citep{fan2020theoretical,nguyen-tang2022on}, but they consider value-based methods where only value function is approximated and also suffer from the curse of dimensionality. 
There are alternative lines of work on policy optimization with general function approximation, but they either assume the functions can be well approximated without any further verification \citep{lan2022policy,mondal2023improved}, or require some strong assumptions such as third-time differentiability \citep{yin2022offline}.

One possible explanation for the empirical effectiveness of DRL algorithms is the adaptivity of neural networks to the intrinsic low-dimensional structure of the state space. In the Atari games example, the images share common textures and are rendered using just a small number of internal states, such as the type, position, and angle of each object, thus the intrinsic dimension of the state space is in fact very small compared to the data dimension of 100,800. However, the extent of this adaptivity is yet to be explored in DRL literature.

To bridge this gap between theory and practice, we propose to investigate neural policy optimization within environments possessing low-dimensional state space structures.
Specifically, we consider the infinite-horizon discounted Markov decision process (MDP) with continuous state space $\cS$, finite action space $\cA$, and discount factor $\gamma$. We focus on the sample complexity of the neural policy mirror descent (NPMD) method. NPMD is based on the actor-critic framework \citep{konda1999actor,grondman2012survey} where both the policy (actor) and value function (critic) are approximated by neural networks. It is an implementation of the general PMD scheme \citep{lan2022policy} with neural network approximation. 
The NPMD-type methods including TRPO \citep{schulman2015trust}) and PPO \citep{schulman2017proximal} are widely used in applications like game AI \citep{berner2019dota} and fine-tuning large language models \citep{ziegler2019fine,ouyang2022training}. 
Moreover, instead of working on general Euclidean space, we assume the state space to be a $d$-dimensional manifold embedded in the $D$-dimensional Euclidean space where $d\ll D$. 

\begin{table}[tb]
    \centering
    \begin{tabular}{ | c | c | c | c | c | c | } 
        \hline
        {Algorithm} & {Approximation} & {Regularity} & {Complexity} & {Remark} \\ 
        \hline\hline
        \makecell{NPG (P) \\ \citep{yuan2023linear}} & Linear & Linear (D) & $\tilde{O}(\epsilon^{-2})$ & \makecell{strong \\ realizability}  \\ 
        \hline
        \makecell{NPPO (P) \\ \citep{liu2019neural}} & \makecell{NTK \\ (2-layer ReLU)} & RKHS (D) & $O(\epsilon^{-14})$ & \makecell{realizability} \\ 
        \hline
        \makecell{NNAC (P) \\ \citep{cayci2022finite}} & \makecell{NTK \\ (2-layer ReLU)} & RKHS (D) & $\tilde{O}(\epsilon^{-6})$ & \makecell{realizability} \\ 
        \hline
        \makecell{AMPO (P) \\ \citep{alfano2023novel}} & \makecell{NTK \\ (2-layer ReLU)} & RKHS (D) & $\tilde{O}(\epsilon^{-4})$ \tablefootnote{The results in \citet{liu2019neural,cayci2022finite,alfano2023novel} implicitly suffer from the curse of dimensionality due to hidden constants related to NTK, including the width of the network and the RKHS norm. These constants can have exponential dependence on $D$ for realizability \citep{yehudai2019power}.} & \makecell{realizability}\\ 
        \hline
        \makecell{FQI (V) \\ \citep{fan2020theoretical}} & \makecell{FNN \\ (Deep ReLU)} & H\"{o}lder (I) & $\tilde{O}(\epsilon^{-\frac{D}{\alpha}-2})$ & \makecell{curse of \\ dimension} \\ 
        \hline
        \makecell{FQI (V) \\ \citep{nguyen-tang2022on}} & \makecell{FNN \\ (Deep ReLU)} & Besov (I) & $\tilde{O}(\epsilon^{-\frac{2D}{\alpha}-2})$ & \makecell{curse of \\ dimension} \\ 
        \hline
        \makecell{NPMD (P) \\ (This paper)} & \makecell{CNN \\ (Deep ReLU)} & Lipschitz (I) \tablefootnote{We define Lipschitz continuity for all $\alpha\in(0,1]$, which reduces to the usually defined Lipschitz condition when $\alpha=1$ and reduces to the usually defined H\"{o}lder continuity when $\alpha\in(0, 1)$. } & $\tilde{O}(\epsilon^{-\frac{d}{\alpha}-2})$ & $d\ll D$ \\ 
        \hline
    \end{tabular}
    \caption{Comparison with existing results. The algorithms are classified as value-based (V) or policy-based (P, including actor-critic methods involving policy gradient). 
    The regularity assumptions are algorithm-dependent (D) or algorithm-independent (I).  
    We consider sample complexity for arbitrary $\epsilon>0$, for which some previous works require additional realizability assumptions to eliminate the error floors. 
    }
    \label{tab:comparison}
\end{table}

We summarize our main contributions as follows: 
\begin{enumerate}[label=(\arabic*)]
    \item We first investigate the universal approximation of the convolutional neural network (CNN), a popular architecture for image data, on the $d$-dimensional manifold. We show that under the Lipschitz MDP condition (\cref{asm:lip-MDP}), CNN with sufficient parameters can well approximate both the value function and the policy (\cref{thm:approx-error-Q}, \cref{cor:approx-error-pi}). Compared to previous work on policy-based methods, our analysis decouples the regularity conditions from algorithmic specifications. For example, in \cite{liu2019neural}, the value functions are assumed to be in a network width-dependent set that approximates the NTK-induced RKHS, while our regularity assumptions are not based on the network architecture in advance. In \cite{yuan2023linear}, the approximation error highly depends on the design of the feature map. 

    \item Based on CNN function approximation, we then derive $\Tilde{O}\bigl(|\cA|^{\frac{d}{2\alpha}+2}(1-\gamma)^{-\frac{4d}{\alpha}-10}\epsilon^{-\frac{d}{\alpha}-2}\bigr)$ sample complexity bound for NPMD with CNN approximation to find a policy whose value function is at most $\epsilon$ to the global optimal in expectation (\cref{cor:sample-complexity-NPMD}). 
    Here, $\alpha\in(0,1]$ is the exponent of the Lipschitz condition and $\Tilde{O}(\cdot)$ hides the logarithmic terms and some coefficients related to distribution mismatch and concentrability (see \cref{asm:full-support,asm:concen-coefficient}). Compared to the results in \cite{fan2020theoretical} and \cite{nguyen-tang2022on}, the curse of dimensionality (exponential dependence on $D$) is avoided by exploiting the intrinsic $d$-dimensional structure. 
    To the best of our knowledge, this is the first sample complexity result for policy gradient methods with deep neural network function approximation. 
\end{enumerate}

Some preliminary results for this work have been first presented in our conference paper \cite{ji2022sample}, which focuses on policy evaluation only. We extend the scope to policy optimization with a full characterization of both iteration complexity and sample complexity.

\subsection{Related Work}
Our work is based on previous studies on policy gradient methods with function approximation as well as deep supervised learning on manifolds.

\textbf{Policy gradient methods.} 
The policy gradient method \citep{williams1992simple,sutton1999policy} is first developed under the compatible function approximation framework. The natural policy gradient (NPG) method \citep{kakade2001natural} extends the policy gradient method by incorporating the geometry of the parameter space to improve convergence properties. Trust region policy optimization (TRPO, \citealt{schulman2015trust}) and proximal policy optimization (PPO, \citealt{schulman2017proximal}) are modern variants of policy gradient methods with neural network function approximation that use constraints or penalties to prevent aggressive updates, resulting in more stable and efficient learning. 
These modern methods are often used to handle high-dimensional state spaces and have been shown to achieve state-of-the-art results in a variety of RL domains. 
For example, the PPO algorithm and its variants are used in training some of the most advanced artificial intelligence, such as OpenAI Five \citep{berner2019dota} and GPT-4 \citep{openai2023gpt}. 
From a theoretical perspective, policy gradient methods such as NPG and PPO can be unified under the PMD framework \citep{geist2019theory,shani2020adaptive,lan2023policy}, whose fast linear rate of convergence has been established for the tabular case \citep{cen2022fast,xiao2022convergence,zhan2023policy}.

\textbf{Linear function approximation.} 
The majority of existing research on function approximation considers the linear function class \citep{agarwal2021theory,alfano2022linear,yuan2023linear}, which is the only known option for the compatible function approximation framework by far \citep{sutton1999policy}. However, these linear function approximation methods are restrictive. Only in simple environments, such as linear MDP \citep{jin2020provably}, can high approximation quality be guaranteed, which necessitates carefully designed features. Regrettably, the task of crafting such features is either infeasible or demands substantial effort from domain experts, and any misspecification of features could lead to an exponential gap \citep{du2020good}. 

\textbf{Reproducing kernel approach.}
The reproducing kernel Hilbert space (RKHS) has been adopted to relax the limitation of the linear function class and to enable more expressive nonlinear function approximation \citep{agarwal2020pc,yang2020provably}. To achieve efficient computation, random features are employed \citep{rahimi2007random}. Nevertheless, the RKHS suffers from the curse of dimensionality, which hinders its performance on high-dimensional problems.

\textbf{Neural tangent kernel.}
One approach to investigating the function approximation capabilities of neural networks is through the use of the neural tangent kernel (NTK, \citealt{jacot2018neural,liu2019neural,wang2019neural,cayci2022finite,alfano2023novel}). The NTK approach can be viewed as training a neural network with gradient descent under a specific regime, and as the width of the neural network approaches infinity, it converges to an RKHS. As a consequence, like other RKHS approaches, the NTK approach suffers from the curse of dimensionality, limiting its performance on high-dimensional problems. Additionally, some literature has pointed out that the NTK is susceptible to the kernel degeneracy problem \citep{chen2020deep,huang2020deep}, which may impact its overall learnability.

\textbf{Non-parametric neural network approximation.}
The non-parametric approach has been adopted to study the sample complexity of neural function approximation in RL under mild smoothness assumptions, such as \citet{fan2020theoretical} and \citet{nguyen-tang2022on}. These analyses are mainly focused on value-based methods and do not apply to policy gradient methods due to the lack of smoothness in neural policies.

\textbf{Deep supervised learning on manifolds.}
Parallel to DRL, existing work on deep supervised learning extensively studies the adaptivity of neural networks to the intrinsic low-dimensional data manifold embedded in high-dimensional ambient space, and how this adaptivity helps neural networks escape from the curse of dimensionality. 
In deep supervised learning, it has been shown that the sample complexity's exponential dependence on the ambient dimension $D$ can be replaced by the dependence on the manifold dimension $d$ \citep{chen2019efficient,schmidt2019deep,liu2021besov}. 
These analyses focus on fitting a single target function whose smoothness is predetermined by the nature of the learning task, while in our setting, the target functions include policies whose smoothness can get worse in each iteration.

\subsection{Notation}
For $n\in\NN$, $[n]\coloneqq\{i\mid 1\leq i\leq n\}$. For $a\in\RR$, $\lceil a\rceil$ denotes the smallest integer no less than $a$. For $a,b\in\RR$, $a\vee b\coloneqq\max(a,b)$ and $a\wedge b\coloneqq\min(a,b)$. For a vector, $\norm{\cdot}_p$ denotes the $p$-norm for $1\leq p\leq +\infty$. For a matrix, $\norm{\cdot}_\infty$ denotes the maximum magnitude of entries. For a finite set, $\abs{\cdot}$ denotes its cardinality. For a function $f\colon\cX\to\RR$, $\norm{f}_\infty$ denotes the maximal value of $\abs{f}$ over $\cX$. 
Given distribution $\rho$ on $\cX$, we use $f(\rho)\coloneqq\EE_{x\sim\rho}\bigl[f(x)\bigr]$ to denote the expected value of $f(x)$ where $x\sim\rho$. Given distributions $\mu$ and $\nu$ on $\cX$, the total variation distance is defined as $\dist_\mathrm{TV}(\mu,\nu)\coloneqq\sup_{A\in\Sigma}\abs{\mu(A)-\nu(A)}$, where $\Sigma$ contains all measurable sets on $\cX$. When $\mu$ is absolutely continuous with respect to $\nu$, the Pearson $\chi^2$-divergence is defined as $\chisqr{\mu}{\nu}\coloneqq\EE_{\nu}[(\dv{\mu}{\nu}-1)^2]$, where $\dv{\mu}{\nu}$ denotes the Radon--Nikodym derivative.  

Let $\cA$ be a finite set, we denote $P^{|\cA|}\coloneqq\{(p_a)_{a\in\cA}\mid p_a\in P\}$ as the Cartesian product of $P$'s indexed by $\cA$, $\one\coloneqq(1)_{a\in\cA}\in\RR^{|\cA|}$ as the vector with all entries being $1$, $\Delta_\cA\coloneqq\left\{p\in\RR^{|\cA|} \mid \sum_{a\in\cA}p_a=1,p_a\geq 0 \right\}$ as the probability simplex over $\cA$, and define the inner product $\inner{\cdot}{\cdot}\colon\RR^{|\cA|}\times\RR^{|\cA|}\to\RR$ as $\inner{p}{q}\coloneqq\sum_{a\in\cA}p_a q_a$. 
Let $\pi\colon\cS\to\Delta_\cA$ be a map, we use $h^\pi(s)\coloneqq\inner{\log\pi(s)}{\pi(s)}$ to denote the negative entropy of $\pi$ at $s\in\cS$ where $\log(\cdot)$ is performed entrywise, and denote the Kullback-Leibler (KL) divergence between two distributions $\pi^\prime(s)$ and $\pi(s)$ by $D_{\pi}^{\pi^\prime}(s)\coloneqq \inner{\log\pi^\prime(s)-\log\pi(s)}{\pi^\prime(s)}\geq 0$. 

\subsection{Roadmap}

The rest of this paper is organized as follows: Section 2 briefly introduces some preliminaries; Section 3 presents the neural policy mirror descent algorithm; Section 4 presents the theoretical analysis; Section 5 presents the experimental results to back up our theory; Section 6 discusses our results with the related work and draws a brief conclusion. 

\section{Background}\label{sec:background}
We introduce the problem setting and briefly review the Markov decision process, Riemannian manifold, and convolutional neural networks.

\subsection{Markov Decision Process}
We consider an infinite-horizon discounted Markov decision process (MDP) denoted as $\cM=(\cS,\cA,\cP,c,\gamma)$, where $\cS\subseteq\RR^D$ is a continuous state space in $\RR^D$, $\cA$ is a finite action space, $\cP$ is the transition kernel that describes the next state distribution $s^\prime\sim\cP(\cdot|s, a)$ at state $s\in\cS$ when action $a\in\cA$ is taken, $c\colon\cS\times\cA\to[0, C]$ is a cost function bounded by some constant $C>0$, and $\gamma\in (0,1)$ is a discount factor. 

A \emph{stochastic policy} $\pi\colon\cS\to\Delta_\cA$ describes the behavior of an agent. For any state $s\in\cS$, $\pi(\cdot|s)\in\Delta_\cA$ gives a conditional probability distribution over the action space $\cA$, where $\pi(a|s)$ is the probability of taking action $a$ at state $s$.

Given a policy $\pi$, the expected cost starting from state $s$ is given by the \emph{state value function} 
\begin{align*}
    V^{\pi}(s)=\EE_{\substack{a_t\sim\pi(\cdot | s_t),\\ s_{t+1}\sim\cP(\cdot | s_t,a_t)}}\biggl[\sum_{t=0}^\infty \gamma^t c(s_t,a_t) \biggm\vert s_0=s \biggr].
\end{align*}
The goal of policy optimization is to learn an optimal policy $\pi^\star$ by solving a stochastic optimization problem, where the objective function is the expected value function for a given initial state distribution $\rho$:\footnote{The optimal policy $\pi^\star$ does not depend on the choice of $\rho$.}
\begin{align}\label{eq:policy-optimization}
    V^\star(\rho)\coloneqq V^{\pi^\star}(\rho)=\minimize_{\pi} ~ \EE_{s\sim\rho}\bigl[V^{\pi}(s)\bigr].
\end{align}

A policy $\pi$ is called \emph{$\epsilon$-optimal}, if \[
    V^\pi(\rho)-V^\star(\rho)\leq \epsilon.
\]
In the reinforcement learning setting, the algorithm cannot directly access the transition kernel $\cP$ and the cost function $c$. Instead, the algorithm can only start from an initial state from $\rho$ and interact with the environment for the immediate cost $c_{s,a}=c(s,a)$ and the next state $s^\prime\sim\cP(\cdot|s,a)$. Each interaction is through a sample oracle. The (expected) number of sample oracle calls required to obtain an $\epsilon$-optimal policy is referred to as the \emph{sample complexity} of the algorithm. 
 
The state value function is closely related to the \emph{state-action value function}, which is the expected cost starting from state $s$ and taking action $a$: 
\begin{align*}
    &Q^{\pi}(s,a)=\EE_{\substack{s_{t+1}\sim\cP(\cdot | s_t,a_t),\\ a_{t+1}\sim\pi(\cdot | s_{t+1})}}\biggl[\sum_{t=0}^\infty \gamma^t c(s_t,a_t) \biggm\vert s_0=s,a_0=a \biggr].
\end{align*}
By definition, the value functions are bounded:
\begin{align}\label{eq:V-bound}
    0\leq V^\pi(s)\leq \frac{C}{1-\gamma},
    \quad 0\leq Q^\pi(s,a)\leq \frac{C}{1-\gamma}.
\end{align}
The value functions satisfy the following relations: 
\begin{align}
    &V^{\pi}(s)=\inner{Q^{\pi}(s,\cdot)}{\pi(\cdot|s)}=\EE_{a\sim\pi(\cdot|s)}[Q^\pi(s,a)],\label{eq:V-to-Q}\\
    &Q^{\pi}(s,a)=c(s,a) + \gamma \EE_{s^\prime\sim\cP(\cdot|s,a)}[V^\pi(s^\prime)].\label{eq:Q-to-V}
\end{align}

For the convenience of analysis, we define recursively 
\begin{equation}\label{eq:prob-transition}
    \cP_0^\pi=\rho,\quad \cP_{t+1}^\pi
    =\EE_{s\sim\cP_t^\pi,a\sim\pi(\cdot|s)}[\cP(\cdot|s,a)], 
\end{equation}
and define the \emph{state visitation distribution} and the \emph{state-action visitation distribution} respectively: 
\begin{align}
    &\nu_\rho^\pi=(1-\gamma)\sum_{t=0}^\infty \gamma^t\cP_t^\pi,\label{eq:visitation-distribution-s}\\
    &\overbar{\nu}_\rho^\pi(s,a)=\nu_\rho^\pi(s)\times\pi(a|s), ~\forall s\in\cS, a\in\cA.\label{eq:visitation-distribution-sa}
\end{align}
The prefactor $1-\gamma$ in \eqref{eq:visitation-distribution-s} makes $\nu_\rho^\pi$ be a distribution.
The visitation distributions reflect the frequency of visiting state $s$ or state-action pair $(s, a)$ along the trajectories starting from $s_0\sim\rho$ and taking actions according to policy $\pi$.
It follows immediately from the definition that the state visitation distribution is lower bounded by the initial distribution (in terms of Radon--Nikodym derivative) with factor $1-\gamma$:  
\begin{equation}\label{eq:visitation-bound}
    \dv{\nu_\rho^\pi}{\rho}
    =(1-\gamma)\sum_{t=0}^\infty \gamma^t\dv{\cP_t^\pi}{\rho}
    \geq(1-\gamma)\dv{\cP_0^\pi}{\rho}
    =1-\gamma.
\end{equation}
We can rewrite the value functions using the visitation distributions: 
\begin{align}
    &V^\pi(\rho)
    = \sum_{t=0}^\infty \gamma^t \int_\cS\sum_{a\in\cA} c(s,a)\pi(a|s)\ud \cP_t^\pi(s)
    = \frac{1}{1-\gamma}\EE_{(s,a)\sim\overbar{\nu}_\rho^\pi}[c(s,a)],\label{eq:V-visitation} \\ 
    &Q^\pi(s,a)
    = c(s,a) + \gamma V^\pi(\cP(\cdot|s,a))
    = c(s,a) + \frac{\gamma}{1-\gamma}\EE_{(s^\prime,a^\prime)\sim\overbar{\nu}_{\cP(\cdot|s,a)}^\pi}[c(s^\prime,a^\prime)], \label{eq:Q-visitation}
\end{align}
where \eqref{eq:Q-visitation} is from \eqref{eq:V-visitation} and \eqref{eq:Q-to-V}. 

\subsection{Riemannian Manifold}
We consider the state space $\cS$ to be a $d$-dimensional Riemannian manifold isometrically embedded in $\RR^D$. A \emph{chart} for $\cS$ is a pair $(U,\phi)$ such that $U\subset \cS$ is open and $\phi: U \rightarrow \RR^d$ is a homeomorphism, i.e., $\phi$ is a bijection; its inverse and itself are continuous. Two charts $(U,\phi)$ and $(V,\psi)$ are called \emph{$C^k$ compatible} if and only if \[
    \phi\circ\psi^{-1}\colon \psi(U\cap V)\rightarrow \phi(U\cap V) \quad \text{ and } \quad \psi\circ\phi^{-1}\colon \phi(U\cap V)\rightarrow \psi(U\cap V)
\]
are both $C^k$ functions ($k$ times continuously differentiable). A \emph{$C^k$ atlas} of $\cS$ is a collection of $C^k$ compatible charts $\{(U_i,\phi_i)\}_{i\in I}$ such that $\bigcup_{i\in I} U_i=\cS$. An atlas of $\cS$ contains an open cover of $\cS$ and mappings from each open cover to $\RR^d$.

\begin{definition}[Smooth manifold]
	A manifold $\cS$ is smooth if it has a $C^{\infty}$ atlas.
\end{definition}

We introduce the \emph{reach} \citep{federer1959curvature,niyogi2008finding} of a manifold to characterize the curvature of $\cS$.
\begin{definition}[Reach]
    The medial axis of $\cS$ is defined as $\overbar{\cT}(\cS)$, which is the closure of \[
        \cT(\cS) = \{x \in \RR^D \mid \exists x_1\neq x_2\in\cS, \norm{x- x_1}_2=\norm{x-x_2}_2 = \inf_{y \in \cS} \norm{x - y}_2\}.
    \] 
    The reach $\omega$ of $\cS$ is the minimum distance between $\cS$ and $\overbar{\cT}(\cS)$, that is, \[
        \omega = \inf_{x\in\overbar{\cT}(\cS),y\in\cS} \norm{x-y}_2.
    \] 
\end{definition}
Roughly speaking, reach measures how fast a manifold ``bends''. A manifold with a large reach ``bends'' relatively slowly. On the contrary, a small $\omega$ signifies more complicated local geometric structures, which are possibly hard to fully capture.

\subsection{Convolutional Neural Networks}

We consider one-sided stride-one convolutional neural networks (CNNs) with the rectified linear unit (ReLU) activation function $\ReLU(z)=\max(z,0)$. Specifically, a CNN we consider consists of a padding layer, several convolutional blocks, and finally a fully connected output layer.

Given an input vector $x \in\RR^{D}$, the network first applies a padding operator $P:\RR^{D}\rightarrow \RR^{D\times C}$ for some integer $C\geq 1$ such that
\[
    Z=P(x)=\begin{bmatrix}
    x & 0 &\cdots & 0 
    \end{bmatrix} \in \RR^{D \times C}.
\]
Then the matrix $Z$ is passed through $M$ convolutional blocks. We will denote the input matrix to the $m$-th block as $Z_m$ and its output as $Z_{m+1}$ (so that $Z_1=Z$).

We now define convolution as illustrated in \cref{fig:CNN}. Let $\cW=(\cW_{j,i,l})_{j,i,l}\in \RR^{C^\prime\times I\times C}$ be a filter where $C^\prime$ is the output channel size, $I$ is the filter size and $C$ is the input channel size. For $Z\in \RR^{D\times C}$, the convolution of $Z$ with $\cW$, denoted with $\cW * Z$, results in $Y\in \RR^{D\times C^\prime}$ with
\begin{align*}
	Y_{k,j}=\sum_{i=1}^I \sum_{l=1}^{C} \cW_{j,i,l} Z_{k+i-1,l},
\end{align*}
where we set $Z_{k+i-1,l}=0$ for $k+i-1>D$. 

\begin{figure}[htb!]
    \centering
    \includegraphics[width=0.45\linewidth]{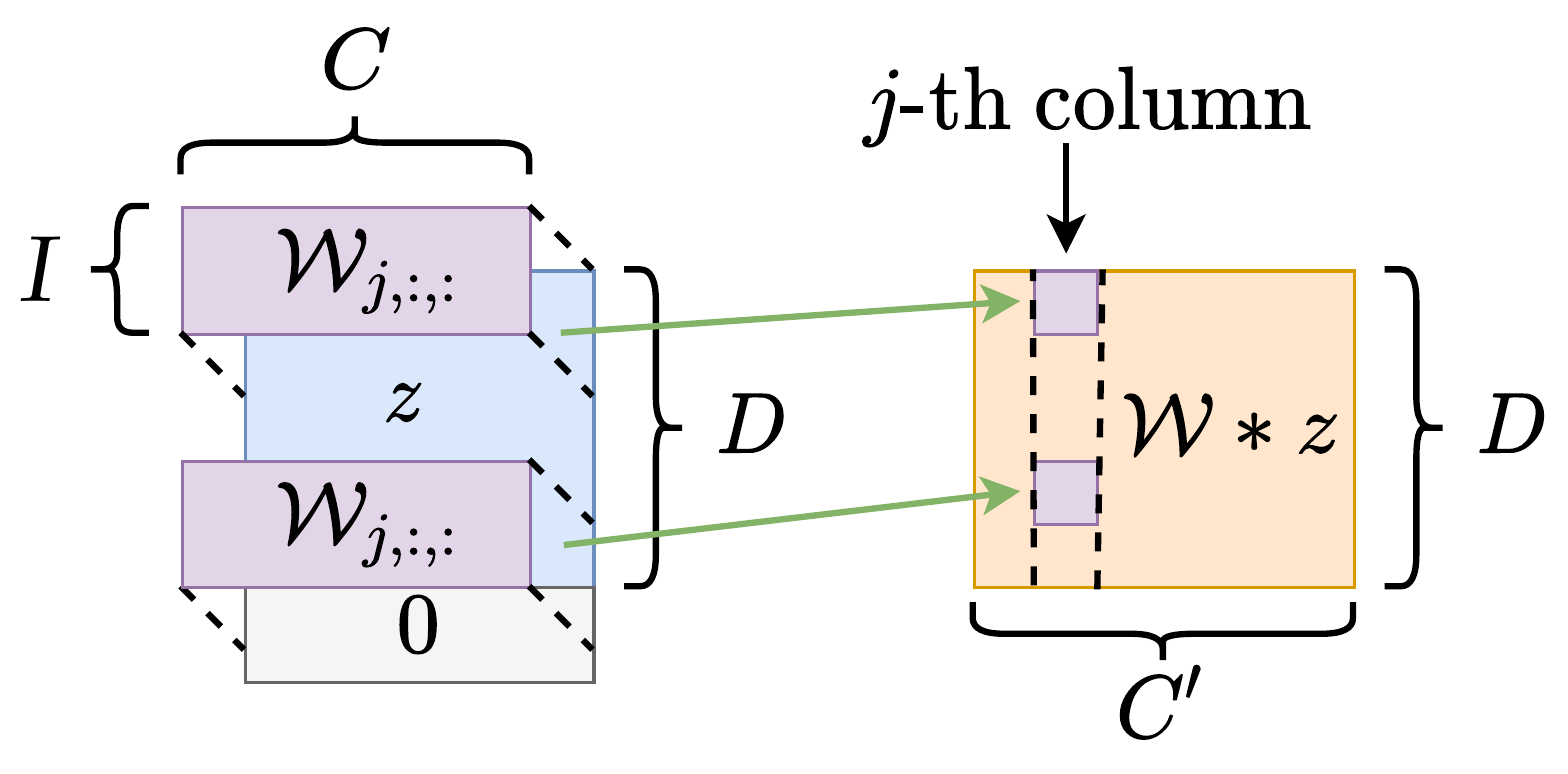}
    \caption{Convolution of $\cW * Z$. $\cW_{j,:,:}$ is a $I\times C$ matrix for the $j$-th output channel.}
    \label{fig:CNN}
\end{figure}

In the $m$-th convolutional block, let
$\cW_m=\{ \cW_m^{(1)},\dots,\cW_m^{(L_m)}\}$ be a collection of filters and $\cB_m=\{\cB_m^{(1)},\dots,\cB_m^{(L_m)}\}$ be a collection of biases of proper sizes.
The $m$-th block maps its input matrix $Z_m\in\RR^{D\times C}$ to $Z_{m+1}\in\RR^{D\times C}$ by 
\begin{align} 
	&Z_{m+1}=\ReLU\Big( \cW_m^{(L_m)}*\cdots *\ReLU\left(\cW_m^{(1)}*Z_m+\cB_m^{(1)}\right)\cdots +\cB_m^{(L_m)}\Big)
	\label{eq:conv}
\end{align}
with $\ReLU$ applied entrywise. For notational simplicity, we denote this series of operations in the $m$-th block with a single operator from $\RR^{D\times C}$ to $\RR^{D\times C}$ with $\Conv_{\cW_m,\cB_m}$, so \eqref{eq:conv} can be abbreviated as
\begin{align*}
	Z_{m+1}=\Conv_{\cW_m,\cB_m}(Z_m).
\end{align*}

Overall, we denote the mapping from input $x$ to the output of the $M$-th convolutional block as
\begin{align}
	G(x)=&\left(\Conv_{\cW_M,\cB_M}\right)\circ\cdots\circ \left(\Conv_{\cW_1,\cB_1}\right)\circ P(x).
	\label{eq:cnnBlock}
\end{align}

Given \eqref{eq:cnnBlock}, a CNN applies an additional fully connected layer to $G$ and outputs
\begin{align*}
	f(x)= W \otimes G(x)+b,
\end{align*}
where $W \in \RR^{D \times C}$ and $b \in \RR$ are a weight matrix and a bias, respectively, and $\otimes$ denotes the sum of entrywise product, that is, $W \otimes G(x) = \sum_{i, j} W_{i, j} [G(x)]_{i, j}$. Thus, we define a class of CNNs of the same architecture as 
\begin{align}\label{eq:nn_class_no_mag}
    &\cF(M,L,J,I,R_1,R_2)=\nonumber\\
    &\big\{ f \mid f(x)=W \otimes G(x)+ b, \|W\|_{\infty} \vee |b| \leq R_2, \text{ where $G(x)$ is in \eqref{eq:cnnBlock} with $M$ blocks.} \nonumber\\
    &\quad\quad \cW_m^{(l)}\in\RR^{C_m^{(l+1)}\times I_m^{(l)}\times C_m^{(l)}}, \cB_m^{(l)}\in\RR^{D\times C_m^{(l+1)}}, \text{ where } C_m^{(l)}\leq J, I_m^{(l)}\leq I, ~\forall l\in[L], m\in[M]; \nonumber\\ 
    &\quad\quad \max_{m,l}\|\cW_m^{(l)}\|_{\infty} \vee \|\cB_m^{(l)}\|_{\infty} \leq R_1\big\}.
\end{align}

\section{Neural Policy Mirror Descent}\label{sec:NPMD}

In this section, we present the neural policy mirror descent (NPMD) algorithm. It is an extension of the policy mirror descent (PMD) method with a \emph{critic network} $Q_w$ parameterized by $w$ to approximate the state-action value function, and an \emph{actor network} $f_\theta$ parameterized by $\theta$ to determine the policy. Both networks belong to a neural network function class $\cF$, which we will specify later in \cref{sec:results}. 

The NPMD algorithm starts from a uniform policy $\pi_0$.
At the $k$-th iteration (indexed from $0$), the policy $\pitheta{k}$ is determined by the actor network $f_{\theta_k}$ along with a hyperparameter $\lambda_k$. The NPMD algorithm first performs a critic update, training the critic network $Q_{w_k}$ to fit the state-action value function of the current policy. 
Then, the NPMD algorithm performs an actor update, indirectly obtaining an improved policy $\pitheta{k+1}$ by updating the actor network to $f_{\theta_{k+1}}$. 

\subsection{Critic Update}\label{sec:NPMD-critic}
For the critic update at the $k$-th iteration, the goal is to approximate the exact state-action value function $Q^{\pitheta{k}}$ with the critic network $Q_{w_k}$. The component of $Q_{w_k}$ corresponding to each action $a\in\cA$ is a neural network $Q_{w_k}(\cdot,a)\in\cF$ parameterized by $w_{k,a}$, which takes $s\in\cS$ as input and outputs a scalar. For simplicity, we let all $\abs{\cA}$ networks share the same architecture and denote $w_k\coloneqq(w_{k,a})_{a\in\cA}\in\cW_k$. We define the critic loss as 
\begin{equation}\label{eq:loss-Q}
    \lossQ(w_k;\pitheta{k})
    =\EE_{s\sim \nu_{\rho}^{\pitheta{k}}}\norm{Q_{w_k}(s,\cdot)-Q^{\pitheta{k}}(s,\cdot)}_2^2,
\end{equation} 
where $Q_{w_k}(s,\cdot)$ and $Q^{\pitheta{k}}(s,\cdot)$ are $|\cA|$-dimensional vectors and $\nu_\rho^{\pitheta{k}}$ is the state visitation distribution defined as \eqref{eq:visitation-distribution-s}.

Directly minimizing the critic loss \eqref{eq:loss-Q} is difficult since $Q^{\pitheta{k}}$ is unknown in advance. 
Instead, we sample $N$ states $\{s_{a,i}\}_{i=1}^N$ independently from the distribution $\nu_\rho^{\pitheta{k}}$ for every action $a\in\cA$ and use the empirical risk on these samples to approximate \eqref{eq:loss-Q}. For notation simplicity, we omit the iteration index $k$ of samples. The empirical risk $\lossQhat$ is defined as 
\begin{align}\label{eq:L2-empirical-Q}
    \lossQhat(w_k; \Xi_{k})=\frac{1}{N}\sum_{a\in\cA}\sum_{i=1}^{N}\abs{Q_{w_k}(s_{a,i},a) - c(s_{a,i},a) - \frac{\gamma}{1-\gamma}c(s^\prime_{a,i}, a^\prime_{a,i})}^2, 
\end{align}
where $N$ is the sample size, each pair $(s^\prime_{a,i}, a^\prime_{a,i})$ is sampled from distribution $\overbar{\nu}_{\cP(\cdot|s_{a,i},a)}^{\pitheta{k}}$,\footnote{We can acquire one sample from this distribution once the sampling algorithm terminates at $s_{a,i}$. Indeed, we can take action $a$ and restart sampling without resetting the environment, so the distribution would be $\overbar{\nu}_{\cP(\cdot|s_{a,i},a)}^{\pitheta{k}}$ as desired.} and $\Xi_{k}$ denotes the collection of samples. 
We let $w_k$ be the solution to the empirical risk minimization (ERM) problem, namely 
\begin{align}\label{eq:ERM-Q}
    w_{k}=\argmin_{w\in\cW_k}\lossQhat(w; \Xi_{k}).
\end{align}

\subsection{Actor Update}\label{sec:NPMD-actor}
For the actor update, the goal is to learn an improved policy. 
If no actor function approximation is considered, an ideal PMD update is given by (see \citealt{lan2023policy}): 
\begin{align}
    &\pi^\star_{k+1}(s)=\argmin_{\pi(\cdot|s)\in\Delta_{\cA}}\inner{Q_{w_k}(s,\cdot)}{\pi(\cdot|s)}+\frac{1}{\eta_k}D_{\pitheta{k}}^\pi(s), ~\forall s\in\cS,\label{eq:PMD-approximate}
\end{align}
where $D_{\pitheta{k}}^\pi(s)$ is the Kullback-Leibler (KL) divergence between $\pi$ and $\pi_k$ and $\eta_k$ is the step size. 
The PMD update \eqref{eq:PMD-approximate} coincides with the KL-penalty version of the PPO algorithm \citep{schulman2017proximal,liu2019neural}.

With neural function approximation, we train a neural policy $\pitheta{k+1}$ to approximate the ideal policy $\pi^\star_{k+1}$.
For any $k\geq 0$, the neural policy $\pi_k$ takes the form 
\begin{align}\label{eq:pitheta}
    \pitheta{k}(a|s)=\frac{\exp(\lambda_k^{-1}f_{\theta_k}(s, a))}{\sum_{a^\prime\in\cA}\exp(\lambda_k^{-1}f_{\theta_k}(s, a^\prime))},
\end{align}
where $\theta_k\coloneqq(\theta_{k,a})_{a\in\cA}$ is the collection of neural network parameters and $\lambda_k>0$ is a temperature parameter (will be discussed later). For any $a\in\cA$, $f_{\theta_k}(\cdot, a)\in\cF$ is a neural network parameterized by $\theta_{k,a}$, which takes $s\in\cS$ as input and outputs a scalar. Again, we let all $\abs{\cA}$ neural networks share the same parameter space and denote $\theta_k\coloneqq(\theta_{k,a})_{a\in\cA}\in\Theta_k$. With definition \eqref{eq:pitheta}, the ideal PMD update \eqref{eq:PMD-approximate} admits a closed-form solution.

\begin{lemma}\label{lem:opt-pi}
    The exact solution of \eqref{eq:PMD-approximate} with neural policy $\pi_k$ defined as \eqref{eq:pitheta} is given by 
    \begin{equation}\label{eq:opt-pi}
        \pi^\star_{k+1}(a|s)=\frac{\exp(g_{k+1}^\star(s,a))}{\sum_{a^\prime\in\cA}\exp(g_{k+1}^\star(s,a^\prime))},
    \end{equation}
    where $g_{k+1}^\star=\lambda_k^{-1}f_{\theta_k}-\eta_k Q_{w_k}$. 
\end{lemma}

The proof of \cref{lem:opt-pi} is given in \cref{proof:opt-pi}. In view of \cref{lem:opt-pi}, approximating $\pi^\star_{k+1}$ with $\pitheta{k+1}$ is equivalent to approximating $g_{k+1}^\star$ with the scaled actor network $\lambda_{k+1}^{-1}f_{\theta_{k+1}}$. 
We define the actor loss to be minimized as 
\begin{align}
    &\losspi{k}(\theta_{k+1};\theta_k,w_k)=\EE_{s\sim \nu_\rho^{\pitheta{k}}}\norm{\lambda_{k+1}^{-1}f_{\theta_{k+1}}(s,\cdot)-\lambda_k^{-1}f_{\theta_k}(s,\cdot)+\eta_k Q_{w_k}(s,\cdot)}_2^2,\label{eq:loss-pi}
\end{align}
where $\lambda_k$ is the current temperature, $\lambda_{k+1}$ is the next temperature, and $\eta_k$ is the step size. For notation simplicity, we omit the hyperparameters $\eta_k$, $\lambda_{k+1}$ and $\lambda_k$ in $\losspi{k}$. 
Similar to the critic update, instead of minimizing \eqref{eq:loss-pi} directly, we minimize the empirical risk:  
\begin{align}
    &\losspihat{k}(\theta_{k+1};\theta_k,w_k,\Xi_{k})\nonumber\\
    =&\frac{1}{N}\sum_{a\in\cA}\sum_{i=1}^{N}\abs{\lambda_{k+1}^{-1}f_{\theta_{k+1}}(s_{a,i},a) - \lambda_k^{-1}f_{\theta_k}(s_{a,i},a) + \eta_k Q_{w_k}(s_{a,i},a)}^2,\label{eq:L2-empirical-pi}
\end{align}
where $\Xi_{k}$ contains the same sampled states $\{s_{a,i}\}_{i=1}^N$ from $\nu_\rho^{\pitheta{k}}$ as used in the critic update.
The improved actor parameter $\theta_{k+1}$ is given by the solution to the ERM problem: 
\begin{align}\label{eq:ERM-pi}
    \theta_{k+1}=\argmin_{\theta\in\Theta_{k+1}}\losspihat{k}(\theta;\theta_k,w_k,\Xi_{k}).
\end{align}
When the sample size $N$ is sufficiently large, we have $\lambda_{k+1}^{-1}f_{\theta_{k+1}}\approx g_{k+1}^\star$ and hence $\pi_{k+1}\approx\pi^\star_{k+1}$.

\begin{algorithm2e}[tb]
    \caption{Neural Policy Mirror Descent}
    \label{alg:NPMD}
    \KwIn{Iteration number $K$, initial distribution $\rho$, sample size per iteration $N$, step size $\eta_k>0$, temperature parameter $\lambda_k>0$, discount factor $\gamma\in(0,1)$, neural network parameter space $\cW, \Theta$}
    Initialize $\theta_0=0$, $\Xi_k=\emptyset, ~\forall k\geq 0$\;
    \For{$k=0$ \KwTo $K-1$}{
        \For{$a\in\cA$}{
            Sample $\{s_{a,i}\}_{i=1}^{N}$ with $s_{a,i}\sim\nu_\rho^{\pitheta{k}}$\;
            Sample $\{(s_{a,i}^\prime,a_{a,i}^\prime)\}_{i=1}^{N}$ with $(s_{a,i}^\prime, a_{a,i}^\prime)\sim\overbar{\nu}_{\cP(\cdot|s_{a,i},a)}^{\pitheta{k}}$\;
            Update $\Xi_{k}\gets \Xi_{k}\cup\{(s_{a,i},s_{a,i}^\prime,a_{a,i}^\prime)\}_{i=1}^{N}$\;
        }
        Update $w_{k}\gets\argmin_{w\in\cW_k}\lossQhat(w; \pitheta{k}, \Xi_{k})$ as \eqref{eq:ERM-Q} \tcp*[r]{Critic update}
        Update $\theta_{k+1}\gets\argmin_{\theta\in\Theta_{k+1}}\losspihat{k}(\theta;\theta_k,w_k,\Xi_{k})$ as \eqref{eq:ERM-pi} \tcp*[r]{Actor update}
    }
    \KwOut{ $\theta_{K}$ as the policy parameter }
\end{algorithm2e}

\begin{remark}
The temperature parameter $\lambda_k$ is introduced mainly for technical reasons. For any infinite-horizon discounted MDP, there always exists a deterministic optimal policy $\pi^\star$ \citep{puterman1994markov}, while the neural policy $\pitheta{k}$ adopted to approximate $\pi^\star$ is fully stochastic in the sense that $\pitheta{k}(a|s)>0$ for any $(s,a)\in\cS\times\cA$. Using the temperature parameter $\lambda_k$ allows us to control the spikiness of $\pitheta{k}$. As $\lambda_k$ approaches zero, $\pitheta{k}$ is prone to the action with the maximal value of $f_{\theta_k}$. This makes stochastic policy $\pitheta{k}$ closer to the deterministic policy $\pi^\star$. 
\end{remark}

\begin{remark}
Our algorithm uses neural policy $\pitheta{k+1}$ to approximate the ideal policy $\pi^\star_{k+1}$. This allows us to keep the most up-to-date policy with only one actor network. If no actor network is used to approximate $\pi^\star_{k+1}$, ideally, we can obtain an implicitly defined policy by iteratively calling \eqref{eq:PMD-approximate}. However, this requires us to keep all the $k+1$ critic networks to compute $\pi^\star_{k+1}$, which is not scalable. 
On the other hand, the critic network, serving solely for training the improved policy $\pitheta{k+1}$, is dispensable. Consequently, we can remove the whole critic part and replace $Q_{w_k}$ in \eqref{eq:L2-empirical-pi} by target values in \eqref{eq:L2-empirical-Q}. This streamlined approach makes policy evaluation implicit and reduces computation overhead. Nevertheless, it cannot simplify function approximation for an improved sample complexity, as the (lack of) smoothness in $f_{\theta_k}$ remains a bottleneck (see \cref{sec:universal-approximation}), and the critic error does not vanish but resides in the form of sample noise. We keep the actor-critic framework for a more elucidated analysis. 
\end{remark}

We summarize NPMD in \cref{alg:NPMD}. Note that \cref{alg:NPMD} requires samples from the visitation distributions. We provide a sampling algorithm in \cref{sec:appendix-algorithm}.

\section{Main Results}\label{sec:results}
In this section, we present our main results on the sample complexity of \cref{alg:NPMD}. 
As mentioned in \cref{sec:introduction}, we focus on RL environments with low-dimensional structures, for which we make the following smooth manifold assumption on the state space.
\begin{assumption}[State space manifold]\label{asm:manifold}
    The state space $\cS$ is a $d$-dimensional compact Riemannian manifold isometrically embedded in $\RR^D$ where $d\ll D$. There exists $B>0$ such that $\norm{x}_\infty\leq B$ for any $x\in\cS$. The surface area of $\cS$ is $\Area(\cS)<\infty$, and the reach of $\cS$ is $\omega>0$.
\end{assumption}

We first derive the iteration complexity with well-approximated value functions and neural policies, then derive the number of samples to meet the requirement for approximation. Combining the results together, we establish the overall sample complexity for \cref{alg:NPMD}. 

\subsection{Iteration Complexity}\label{sec:iteration-complexity}
We make the following assumptions on the initial and visitation distributions for iteration complexity. 

\begin{assumption}[Full support]\label{asm:full-support}
    The initial distribution $\rho$ has full support on $\cS$, that is, for any measurable subset $S\subseteq\cS$, $\rho(S)>0$. 
\end{assumption}

\cref{asm:full-support} requires that every state can be visited when doing sampling. In view of \eqref{eq:visitation-bound}, as long as $\rho$ has full support, the visitation distribution also has full support even if the policy itself is deterministic. 
We measure the distribution mismatch between the optimal visitation distribution $\nu_\rho^{\pi^\star}$ and the initial distribution $\rho$ by a mismatch coefficient denoted as $\kappa$: 
\begin{equation}\label{def:kappa}
    \kappa\coloneqq\norm{\dv{\nu_\rho^{\pi^\star}}{\rho}}_\infty.
\end{equation}
Under \cref{asm:full-support}, the Radon--Nikodym derivative $\dv{\nu_\rho^{\pi^\star}}{\rho}$ is well-defined. We assume $\kappa<\infty$. 
Accordingly, we define the shifted discount factor as 
\begin{equation}\label{def:discount}
    \gamma_\rho\coloneqq 1-(1-\gamma)/\kappa.
\end{equation}

\begin{assumption}[Concentrability]\label{asm:concen-coefficient}
    There exists $C_\nu<\infty$ such that for all $k\geq 0$ iterations of \cref{alg:NPMD},  
    \begin{align*}
        &\chisqr{\nu_\rho^{\pi}}{\nu_\rho^{\pitheta{k}}}+1\leq C_\nu,
    \end{align*}
    where $\pi$ takes $\pitheta{k+1}$ or $\pi^\star$, $\chisqr{\nu_\rho^{\pi}}{\nu_\rho^{\pitheta{k}}}=\EE_{\nu_\rho^{\pitheta{k}}}[(\dv{\nu_\rho^{\pi}}{\nu_\rho^{\pitheta{k}}}-1)^2]$ is the $\chi^2$-divergence.
\end{assumption}

\cref{asm:concen-coefficient} requires the concentrability of the visitation distributions. The distance between visitation distributions is measured by the $\chi^2$-divergence, which is well-defined under \cref{asm:full-support} since the absolute continuity holds for fully supported distributions. This type of concentrability assumption is commonly adopted in the RL literature \citep{agarwal2021theory,yuan2023linear} and is tighter than the absolute density ratio $\norm{\dv{\nu_\rho^{\pi}}{\nu_\rho^{\pitheta{k}}}}_\infty$. \cref{asm:full-support,asm:concen-coefficient} together forms the \emph{optimism} in RL, that is, the initial distribution is not too far away from the optimal visitation distribution in terms of $\chi^2$-divergence, and the policy $\pi_k$ at each iteration is sufficiently exploratory to find out the optimal policy. 

With \cref{asm:full-support,asm:concen-coefficient}, we have the following one-step improvement lemma for \cref{alg:NPMD}. The proof is provided in \cref{proof:onestep-NPMD}.
\begin{lemma}\label{lem:onestep-NPMD}
    Suppose \cref{asm:full-support,asm:concen-coefficient} hold. Then \cref{alg:NPMD} yields 
    \begin{multline*}
        \bigl(V^{\pitheta{k+1}}(\rho) - V^\star(\rho)\bigr) + \frac{1}{\kappa\eta_k}\EE_{s\sim \nu_\rho^{\pi^\star}} \bigl[D_{\pitheta{k+1}}^{\pi^\star}(s)\bigr] \\
        \leq \gamma_\rho\bigl(V^{\pitheta{k}}(\rho) - V^\star(\rho) + \frac{1}{\kappa\gamma_\rho\eta_k}\EE_{s\sim \nu_\rho^{\pi^\star}} \bigl[D_{\pitheta{k}}^{\pi^\star}(s)\bigr] \bigr)\\
         + \frac{4\sqrt{C_\nu}}{\kappa(1-\gamma_\rho)}\Bigl(\sqrt{\lossQ(w_k;\pitheta{k})} + \frac{1}{\eta_k}\sqrt{\losspi{k}(\theta_{k+1};\theta_k,w_k)}\Bigr).
    \end{multline*}
\end{lemma}


\cref{lem:onestep-NPMD} demonstrates that the optimality gap in value function decreases at rate $\gamma_\rho$ up to approximation errors introduced by the critic and actor updates. 
When these errors are properly controlled, we can establish iteration complexity for \cref{alg:NPMD}.
\begin{theorem}\label{thm:iteration-complexity}
    Suppose \cref{asm:full-support,asm:concen-coefficient} hold. If $\eta_k=\frac{1-\gamma_\rho}{C\gamma_\rho^{k+1}}$ for all $k\geq 0$, the critic loss and the actor loss satisfy respectively 
    \[
        \EE\bigl[\lossQ(w_k;\pitheta{k})\bigr]\leq C^2\gamma_\rho^{2(k+1)},\quad 
        \EE\bigl[\losspi{k}(\theta_{k+1};\theta_k,w_k)\bigr]\leq (1-\gamma_\rho)^2,
    \]
    then after $k\geq 1$ iterations, the expected optimality gap of $\pitheta{k}$ given by \cref{alg:NPMD} is \[
        \EE \bigl[V^{\pitheta{k}}(\rho) - V^\star(\rho)\bigr]
        \leq \gamma_\rho^{k}\bigl(C_1 + C_2(k+1)\bigr)\cdot\frac{C}{1-\gamma},
    \] where $C_1=1 + \log|\cA|$, $C_2=8\sqrt{C_\nu}$.
    
    Moreover, for any $\epsilon>0$, the number of iterations required for $\EE \bigl[V^{\pitheta{k}}(\rho) - V^\star(\rho)\bigr]\leq\epsilon$ is \[
        \Tilde{O}\left(\log_\frac{1}{\gamma_\rho}\left(\frac{C\left(\sqrt{C_\nu}+\log|\cA|\right)}{\kappa(1-\gamma_\rho)^2\epsilon}\right)\right).
    \]
\end{theorem}

The proof of \cref{thm:iteration-complexity} is provided in \cref{proof:iteration-complexity}. 
We choose exponentially increasing step size $\eta_k$ to establish (almost) linear convergence rate. To achieve this fast rate, we require the critic loss to be exponentially decreasing and the actor loss to be small for any consecutive temperatures $\lambda_k$ and $\lambda_{k+1}$.
In fact, if we convert the loss into function approximation error (see \cref{sec:appendix-translation-Q,sec:appendix-translation-pi}) and choose $\lambda_k$ to be exponentially decreasing at rate $\gamma_\rho$, then the requirements on critic and actor networks become the same. 
We show in the following subsections that these requirements can be met by properly designing the CNN architecture and using sufficiently many training samples. As long as the conditions are satisfied, \cref{thm:iteration-complexity} guarantees to find an $\epsilon$-optimal policy (in expectation) after $\tilde{O}(\log\frac{1}{\epsilon})$ iterations. 

\begin{remark}
    \cref{thm:iteration-complexity} suggests that we can use smaller networks or fewer samples in the early stage and gradually increase the sizes along iterations. When the tolerance $\epsilon$ is given, the error bound for the last iteration is $\sqrt{\lossQ(w_K;\pitheta{K})} + \frac{1}{\eta_K}\sqrt{\losspi{K}(\theta_{K+1};\theta_K,w_K)}=\tilde{O}(\epsilon)$. 
    Notably, this coincides (up to logarithm factors) with the stipulated approximation error criterion in prior literature \citep{yuan2023linear,alfano2023novel}, with the distinction that their criteria are applied uniformly to all iterations. 
    If we use a constant step size, then the iteration complexity becomes $O(\epsilon^{-1})$ as in \citet{alfano2023novel}, while the error criteria for the last iterate remains $O(\epsilon)$. Consequently, the overall sample complexity becomes much worse as it requires more iterations to converge. 
\end{remark}

\subsection{Function Approximation on Lipschitz MDP with CNN}\label{sec:universal-approximation}
The iteration complexity in \cref{thm:iteration-complexity} is valid if the state-action value function $Q^{\pitheta{k}}$ and the policy $\pi^\star_{k+1}$ are well approximated by the critic and actor networks at each iteration. However, we have not yet specified the CNN architecture that can meet these requirements. In this section, we study the function approximation on \emph{Lipschitz MDP}, which possesses Lipschitz transition kernel $\cP$ and Lipschitz cost function $c$. Here, the Lipschitzness is defined with respect to the \emph{geodesic distance} on the state space $\cS$, which is a $d$-dimensional Riemannian manifold (\cref{asm:manifold}). 
Recall the definition of geodesic distance: 
\begin{definition}[Geodesic distance]\label{def:geodesic}
    The geodesic distance between two points $x,y\in\cS$ is defined as 
    \begin{align}
        \dist_\cS(x,y)\coloneqq\inf_{\Gamma\colon[0,1]\to\cS} ~ &\int_0^1\norm{\Gamma^\prime(t)}_2\ud t \label{eq:geodesic}\\ 
        \text{ s.t. } ~ &\Gamma(0)=x,~\Gamma(1)=y,~\text{$\Gamma$ is piecewise $C^1$.}\nonumber
    \end{align}
\end{definition}

One can show the existence of a solution to the minimization problem \eqref{eq:geodesic} under mild conditions and that $\dist_\cS(\cdot,\cdot)$ is indeed a distance. More references can be found in \cite{do1992riemannian}. With geodesic distance, we define Lipschitz functions on the Riemannian manifold $\cS$.
\begin{definition}[Lipschitz function]\label{def:lip-func}
    Let $L\geq 0$ and $\alpha\in(0,1]$ be constants. A function $f\colon\cS\to\RR$ is called $(L,\alpha)$-Lipschitz if for any $x,y\in\cS$, \[
        \abs{f(x)-f(y)}\leq L\cdot d_\cS^\alpha(x,y).
    \]
\end{definition}
For any fixed $\alpha$, the Lipschitz constant $L$ in \cref{def:lip-func} measures the smoothness of the function. A function is considered smooth if it possesses a small Lipschitz constant, whereas a non-smooth function will exhibit a large Lipschitz constant. Throughout the remainder of this paper, when we mention Lipschitzness, we specifically mean the property of being Lipschitz continuous with a moderate constant. 

\begin{remark}
    The geodesic distance $\dist_\cS$ in \cref{def:lip-func} is a global distance rather than a local one. This makes our definition of Lipschitz functions different from those based on local Euclidean distance and partition of unity as in \cite{chen2019efficient} and \cite{liu2021besov}. 
    The two ways of defining Lipschitzness have some technical differences, but they agree with each other in our setting up to constant factors.
    
    When the atlas $\{(U_i,\phi_i)\}_{i\in I}$ are local projections onto tangent spaces as in \cite{chen2019efficient}, the local Euclidean distance between two points in the same open set $U_i$ is no greater than their Euclidean distance in $\RR^D$, which is further less than their geodesic distance, hence the Lipschitzness defined with local distances implies \cref{def:lip-func}.
    
    On the other hand, when the curvature of manifold $\cS$ is not too large compared to the radius of the open set $U_i$, then \cref{def:lip-func} also implies Lipschitzness in the Euclidean sense on each local coordinate $\phi_i(U_i)\subset [0,1]^d$ (\cref{lem:lip-component}). Here, we adopt the global definition for simplicity. 
\end{remark}

We now formally define the Lipschitz MDP condition, which ensures the Lipschitzness of the state-action value function $Q^{\pi}$ for any policy $\pi$. 
\begin{assumption}[Lipschitz MDP]\label{asm:lip-MDP}
    There exist constants $L_\cP, L_c\geq 0$ and $\alpha\in(0,1]$ such that for any tuple $(s,s^\prime,a)\in\cS\times\cS\times\cA$, the cost function $c(\cdot,a)\colon\cS\to\RR$ is $(L_c,\alpha)$-Lipschitz and the transition kernel $\cP$ satisfies
    \begin{align*}
        &\dist_{\mathrm{TV}}(\cP(\cdot|s,a),\cP(\cdot|s^\prime,a))\leq L_\cP\cdot\dist_\cS^\alpha(s,s^\prime),
    \end{align*}
    where $\dist_{\mathrm{TV}}(\cdot,\cdot)$ is the total variation distance.
\end{assumption}

\cref{asm:lip-MDP} requires that when two states are close to each other, taking the same action would admit similar transition distributions and corresponding costs. 
This assumption holds for many spatially smooth environments, especially those driven by physical simulations such as MuJuCo \citep{todorov2012mujoco} and classic control environments \citep{brockman2016openai}. 
Under \cref{asm:lip-MDP}, we show in \cref{lem:lip-Q} that the state-action value function $Q^{\pi}$ is Lipschitz regardless of the evaluated policy $\pi$. The proof of \cref{lem:lip-Q} is provided in \cref{proof:lip-Q}.
\begin{lemma}\label{lem:lip-Q}
    If \cref{asm:lip-MDP} holds, then for any policy $\pi$ and any action $a\in\cA$, the state-action value function $Q^{\pi}(\cdot,a)\colon\cS\to\RR$ is $(L_Q,\alpha)$-Lipschitz with $L_Q=L_c+\frac{\gamma C}{1-\gamma} L_\cP$ being the Lipschitz constant. 
    That is, for any policy $\pi$ and any tuple $(s,s^\prime,a)\in\cS\times\cS\times\cA$, we have 
    \begin{equation*}
        |Q^{\pi}(s,a)-Q^{\pi}(s^\prime,a)|\leq L_Q\cdot\dist_\cS^\alpha(s,s^\prime).
    \end{equation*}
\end{lemma}

The Lipschitz constant $L_Q=L_c + \frac{\gamma C}{1-\gamma}L_\cP$ scales linearly with the magnitude of the cost function $c$ as $L_c$ does. In view of this property, we define a normalized Lipschitz constant which is invariant to the scaling of the cost: 
\begin{align}\label{def:normalized-lip-Q}
    \overbar{L}_Q\coloneqq(1-\gamma)L_c/C + \gamma L_\cP.
\end{align}
This normalized Lipschitz constant is a convex combination of $L_c/C$ and $L_\cP$, so it will not exceed the larger one for any $\gamma$. 

We make a few remarks on the Lipschitz condition. 
\begin{remark}\label{rmk:lip-MDP}
    \cref{asm:lip-MDP} is a sufficient condition for the Lipschitzness of $Q^\pitheta{k}$, regardless of the smoothness of $\pitheta{k}$. 
    The Lipschitzness of the target function is a minimal requirement for approximation theory and it is almost essential even for simple regression problems. 
    However, even if \cref{asm:lip-MDP} does not hold, $Q^\pitheta{k}$ being Lipschitz is still possible. 
    An extreme example is when state space $\cS=\SSS^1$ is a circle and the transition is a fixed rotation whichever action is taken. In this case, the transition kernel is not Lipschitz since the total variation distance between transitions is always $1$. Meanwhile, $Q^{\pitheta{k}}$ is Lipschitz for any $\pitheta{k}$ provided that $c$ is Lipschitz. 
    Similar arguments can be found in \citet{fan2020theoretical} and \citet{nguyen-tang2022on} for non-smooth MDP having smooth Bellman operator, but they implicitly involve the smoothness of neural policy $\pitheta{k}$. 
\end{remark}

\begin{remark}\label{rmk:smooth-MDP}
    In practice, many environments adhere to the Lipschitz condition, with only an extremely small portion of states being exceptions. 
    For example, in the Box2D Car Racing environment \citep{brockman2016openai}, the cost remains constant for each frame until a tile is reached, for which the agent will be given a huge reward. Even though this type of partially smooth environment does not fulfill the Lipschitz condition on a global scale, it is reasonable to expect the existence of an environment that is globally smooth and satisfies \cref{asm:lip-MDP}. Such a globally smooth environment could serve as a regularization of the original non-smooth environment, which inevitably introduces bias to the problem. When this bias is negligible compared to the extent of smoothness, we can study the smooth approximation under \cref{asm:lip-MDP} without loss of generality.
\end{remark}

With \cref{lem:lip-Q} established, we show that a CNN of the form \eqref{eq:nn_class_no_mag} can uniformly approximate $Q^{\pitheta{k}}(\cdot,a)$ for any $a\in\cA$. The approximation error depends on the specified CNN architecture.

\begin{theorem}[Critic approximation]\label{thm:approx-error-Q}
    Suppose \cref{asm:manifold,asm:lip-MDP} hold. For any integers $I \in [2, D]$ and $\Tilde{M}, \Tilde{J}>0$, we let 
    \begin{align*}
        &M=O(\Tilde{M}), ~ L=O(\log{(\Tilde{M}\Tilde{J})}+D+\log D), ~ J=O(D\Tilde{J}), \\ 
        &R_1=(8ID)^{-1}\Tilde{M}^{-\frac{1}{L}}=O(1), ~ \log R_2=O(\log^2 (\Tilde{M}\Tilde{J}) + D\log {(\Tilde{M}\Tilde{J})}),
    \end{align*}
    where $O(\cdot)$ hides a constant depending on $\log L_Q$, $\log \frac{C}{1-\gamma}$, $d$, $\alpha$, $\omega$, $B$, and the surface area $\Area(\cS)$. 
    Then for any policy $\pi$ and any action $a\in\cA$, there exists a CNN $Q_w(\cdot,a)\in\cF(M,L,J,I,R_1,R_2)$ such that 
    \begin{align*}
        \norm{Q_w(\cdot,a)-Q^\pi(\cdot,a)}_\infty \leq \frac{C}{1-\gamma}(\overbar{L}_Q+1)(\Tilde{M}\Tilde{J})^{-\frac{\alpha}{d}}.
    \end{align*}
    Here, $L_Q$ and $\overbar{L}_Q$ are defined as in \cref{lem:lip-Q} and \eqref{def:normalized-lip-Q}.
\end{theorem}

We provide a proof overview for \cref{thm:approx-error-Q} in \cref{sec:appendix-approx-overview} and the detailed proof in \cref{proof:approx-error-Q}. Compared to our preliminary work (Theorem 1 in \citealt{ji2022sample}) that deals with a larger class of Besov functions by using cardinal B-spline approximation as a crucial step, we simplify the proof for Lipschitz functions where first-order spline approximation is sufficient. 

To bound the approximation error by $\epsilon$, we require the number of parameters in $Q_w(\cdot,a)$ to be $O(MLJ^2I)=\tilde{O}(D^3I\epsilon^{-\frac{d}{\alpha}})$. 
Note that the exponent over $\epsilon$ is the intrinsic dimension $d$ rather than the data dimension $D$, and the hidden terms in $\tilde{O}(\cdot)$ also have no exponential dependence on $D$. 
This implies that CNN approximation does not suffer from the curse of dimensionality when the data support has a low-dimensional manifold structure.

Next, we consider function approximation for policy $\pi^\star_{k+1}$ in the actor update. 
As mentioned in \cref{sec:NPMD-actor}, our goal is to learn a deterministic optimal policy, which is equivalent to a discrete mapping from $\cS$ to $\cA$. 
Such a mapping is not continuous given the discrete nature of $\cA$, so it is difficult to be approximated directly. 
Instead, we iteratively update a temperature-controlled neural policy $\pitheta{k}$ in the form of \eqref{eq:pitheta} to approximate the deterministic optimal policy $\pi^\star$. 
Although $\pitheta{k}$ is a stochastic policy by construction, it can approximate the deterministic policy that chooses the action $a\in\cA$ with the maximal value of $f_{\theta_k}(s,a)$ by using a sufficiently small temperature $\lambda_k>0$. 
Therefore, as long as the actor network $f_{\theta_k}(s,\cdot)$ is learned to admit a maximizer $a\in\supp(\pi^\star(\cdot|s))$ for any state $s\in\cS$, it can serve as a good approximation of the optimal policy $\pi^\star$. 

To learn such an actor network, we iteratively train a new actor network $f_{\theta_{k+1}}$ based on the current critic and actor networks $Q_{w_k}$ and $f_{\theta_k}$. 
According to \cref{lem:opt-pi}, the target function for the next actor network $f_{\theta_{k+1}}$ is given by 
\begin{align*}
    \lambda_{k+1}g^\star_{k+1}=\lambda_k^{-1}\lambda_{k+1}f_{\theta_k} - \eta_k\lambda_{k+1} Q_{w_k}, 
\end{align*}
which is a weighted sum of the current critic and actor networks. 
This approximation target is not Lipschitz, since $Q_{w_k}$ is just an approximation of the Lipschitz function $Q^{\pitheta{k}}$, not a Lipschitz function in itself. 
Consequently, \cref{thm:approx-error-Q} cannot be directly transferred to actor approximation. 
To address the issue, we introduce the \emph{approximately Lipschitz} condition to describe the smoothness inherited from approximating a Lipschitz function. 

\begin{definition}[Approximate Lipschitzness]\label{def:approx-lip}
    Let $L, \epsilon\geq 0$ and $\alpha\in(0,1]$ be constants. A function $f\colon\cS\to\RR$ is called $(L,\alpha,\epsilon)$-approximately Lipschitz if for any $x,y\in\cS$, \[
        \abs{f(x)-f(y)}\leq L\cdot d_\cS^\alpha(x,y)+2\epsilon.
    \]
    Here, $L$ is called the Lipschitz constant, and $\epsilon$ is called the proximity constant.
    When $\epsilon=0$, $f$ is $(L,\alpha)$-Lipschitz as defined in \cref{def:lip-func}.
\end{definition}

\cref{def:approx-lip} relaxes the Lipschitz condition by allowing a proximity constant $\epsilon$. When $\epsilon=0$, the condition reduces to the Lipschitz continuity, and the Lipschitz constant is exactly the one in \cref{def:lip-func}. When $\epsilon\geq \norm{f}_\infty$, the condition is vacuously true for all $L\geq 0$. When $\epsilon$ is somewhere between $0$ and $\norm{f}_\infty$, a larger $\epsilon$ allows a potentially smaller Lipschitz constant $L$. 
As \cref{lem:approx-lip} shows, the approximate Lipschitzness is a shared property of all uniform approximators of a Lipschitz function. 

\begin{lemma}\label{lem:approx-lip}
    If $\overbar{f}_0\colon\cS\to\RR$ is $(L,\alpha)$-Lipschitz, $f\colon\cS\to\RR$ satisfies $\norm{f-\overbar{f}_0}_\infty\leq\epsilon$ for some $\epsilon>0$, then $f$ is $(L,\alpha,\epsilon)$-approximately Lipschitz.
\end{lemma}
\begin{proof}
    For any $x,y\in\cS$, 
    \begin{align*}
        \abs{f(x)-f(y)}
        &=\abs{\overbar{f}_0(x)-\overbar{f}_0(y)+f(x)-\overbar{f}_0(x)-f(y)+\overbar{f}_0(y)}\\
        &\leq\abs{\overbar{f}_0(x)-\overbar{f}_0(y)}+\abs{f(x)-\overbar{f}_0(x)}+\abs{-f(y)+\overbar{f}_0(y)}\\
        &\leq L\cdot d_\cS^\alpha(x,y)+2\epsilon.
    \end{align*}
    The first inequality comes from the triangle inequality. The second inequality is from $L$-Lipschitz continuity of $\overbar{f}_0$ and that $\norm{f-\overbar{f}_0}_\infty\leq\epsilon$.
\end{proof}

It follows immediately from \cref{thm:approx-error-Q} and \cref{lem:approx-lip} that there exists an $(L_Q,\epsilon,\epsilon_Q)$-approximately Lipschitz $Q_{w_k}$ that can uniformly approximate $Q^{\pitheta{k}}$ up to $\epsilon_Q$ error. Therefore, we can impose (approximately) Lipschitz restrictions on the CNN function class without damaging its approximation power for the state-action value function. 
To be more precise, for any CNN class $\cF=\cF(M,L,J,I,R_1,R_2)$, we define its Lipschitz-restricted version $\cF_\mathrm{Lip}(A,L_f,\alpha,\epsilon_f)$ as 
\begin{align}\label{eq:nn_class_lip}
    \cF_\mathrm{Lip}(A,L_f,\alpha,\epsilon_f)=
    \{f\in\cF \mid \norm{f}_\infty\leq A, ~ \text{$f$ is $(L_f,\alpha,\epsilon_f)$-approximately Lipschitz}\}.
\end{align}
Moreover, we denote the parameter space of the Lipschitz-restricted critic and actor network classes as $\cW_{\mathrm{Lip}}$ and $\Theta_\mathrm{Lip}$ respectively: 
\begin{align}
    &\cW_{\mathrm{Lip}}(A, L_Q, \alpha, \epsilon_Q)
    =\bigl\{w \mid Q_w(\cdot,a)\in\cF_{\mathrm{Lip}}(A, L_Q, \alpha, \epsilon_Q),~\forall a\in\cA\bigr\},\label{eq:param-lip-w}\\ 
    &\Theta_{\mathrm{Lip}}(A, L_f, \alpha, \epsilon_f)
    =\bigl\{\theta \mid f_\theta(\cdot,a)\in\cF_{\mathrm{Lip}}(A, L_f, \alpha, \epsilon_f),~\forall a\in\cA\bigr\}.\label{eq:param-lip-theta}
\end{align}
Then by setting $\cW_k=\cW_{\mathrm{Lip}}$ and $\Theta_k=\Theta_{\mathrm{Lip}}$ in \cref{alg:NPMD}, we ensure the $k$-th critic network $Q_{w_k}$ and actor network $f_{\theta_k}$ are both approximately Lipschitz, and such a restriction on $\cW_k$ will not affect the approximation power of $Q_{w_k}$ for $Q^{\pitheta{k}}$. 
In addition, the target function $\lambda_{k+1}g^\star_{k+1}$ for the next actor is also approximately Lipschitz since it is a weighted sum of two approximately Lipschitz functions. 
By carefully selecting the temperature parameters to match the configuration of $\eta_k$ in \cref{thm:iteration-complexity}, the approximate Lipschitzness of target functions for actor updates in all iterations can be uniformly controlled.
\begin{lemma}\label{lem:actor-lip}
    For $k\geq 0$, we let $\eta_k=\frac{1-\gamma_\rho}{C\gamma_\rho^{k+1}}$, $\lambda_k=\frac{C\gamma_\rho^k}{1-\gamma_\rho}$, $\cW_k=\cW_{\mathrm{Lip}}(\frac{C}{1-\gamma}, L_Q, \alpha, \epsilon_Q)$ and $\Theta_{k}=\Theta_\mathrm{Lip}(\frac{C}{(1-\gamma_\rho)(1-\gamma)}, \frac{L_Q}{1-\gamma_\rho}, \alpha, \frac{\epsilon_Q}{1-\gamma_\rho})$ with some $\epsilon_Q\geq 0$. Then the target actor $\lambda_{k+1}g^\star_{k+1}$ defined in \cref{lem:opt-pi} is $(\frac{L_Q}{1-\gamma_\rho}, \alpha, \frac{\epsilon_Q}{1-\gamma_\rho})$-approximately Lipschitz and is uniformly bounded by $\frac{C}{(1-\gamma_\rho)(1-\gamma)}$. 
\end{lemma}
\begin{proof}
    By \cref{lem:opt-pi} and our choice of $\eta_k$ and $\lambda_k$, the target function of the $k$-th critic update is 
    \begin{align*}
        \lambda_{k+1}g^\star_{k+1}
        &=\lambda_{k+1}\lambda_k^{-1}f_{\theta_k}-\eta_k\lambda_{k+1}Q_{w_k}\\
        &=\gamma_\rho f_{\theta_k}-Q_{w_k}.
    \end{align*}
    We initialize $\theta_0=0$ in \cref{alg:NPMD}, thus $\theta_0\in\Theta_\mathrm{Lip}(0, 0, \alpha, 0)\subseteq\Theta_0$. 
    It is easy to verify that if $f\colon\cS\to\RR$ is $(L_f, \alpha, \epsilon_f)$-approximately Lipschitz and $g\colon\cS\to\RR$ is $(L_g, \alpha, \epsilon_g)$-approximately Lipschitz, $c\in\RR$, then $f+g$ is $(L_f+L_g, \alpha, \epsilon_f + \epsilon_g)$-approximately Lipschitz, and $c\cdot f$ is $(\abs{c}L_f, \alpha, \abs{c}\epsilon_f)$-approximately Lipschitz.
    Combining this fact and our choice of $\cW_k$ and $\Theta_k$, we have that $\lambda_{1}g^\star_{1}$ is $(\frac{L_Q}{1-\gamma_\rho}, \alpha, \frac{\epsilon_Q}{1-\gamma_\rho})$-approximately Lipschitz and is uniformly bounded by $\frac{C}{(1-\gamma_\rho)(1-\gamma)}$.    
\end{proof}

\cref{lem:actor-lip} shows the target actor is approximately Lipschitz in each iteration with $\cW_k$ and $\Theta_k$ inserted. 
It remains to derive the approximation error for actor update with Lipschitz-restricted class $\Theta_{k+1}=\Theta_{\mathrm{Lip}}$. 
We first show in \cref{thm:approx-error-approxlip} that any bounded and approximately Lipschitz function on $\cS$ can be well approximated by a CNN with enough parameters, and this CNN is also bounded and approximately Lipschitz.
\begin{theorem}\label{thm:approx-error-approxlip}
    Suppose \cref{asm:manifold} holds, the target function $f_0\colon\cS\to\RR$ is bounded and $(L_f, \alpha, \epsilon_f)$-approximately Lipschitz. For any integers $I \in [2, D]$ and $\Tilde{M}, \Tilde{J}>0$, we let 
    \begin{align*}
        &M=O(\Tilde{M}), ~ L=O(\log{(\Tilde{M}\Tilde{J})}+D+\log D), ~ J=O(D\Tilde{J}), \\ 
        &R_1=(8ID)^{-1}\Tilde{M}^{-\frac{1}{L}}=O(1), ~ \log R_2=O(\log^2 (\Tilde{M}\Tilde{J}) + D\log {(\Tilde{M}\Tilde{J})}),
    \end{align*}
    where $O(\cdot)$ hides a constant depending on $\log L_f$, $\log \norm{f_0}_\infty$, $\alpha$, $\omega$, $B$, and the surface area $\Area(\cS)$. 
    Then there exists a CNN $f\in\cF(M,L,J,I,R_1,R_2)$ such that 
    \begin{align*}
        \norm{f-f_0}_\infty \leq (L_f+\norm{f_0}_\infty)(\Tilde{M}\Tilde{J})^{-\frac{\alpha}{d}} + 2\epsilon_f.
    \end{align*}
    Moreover, this $f$ can be $(L_f,\alpha,\hat{\epsilon}_f)$-approximately Lipschitz with $\hat{\epsilon}_f=(L_f+\norm{f_0}_\infty)(\Tilde{M}\Tilde{J})^{-\frac{\alpha}{d}}$ and uniformly bounded by $\norm{f_0}_\infty$.
\end{theorem}

The proof of \cref{thm:approx-error-approxlip} is provided in \cref{proof:approx-error-approxlip}. 
\cref{thm:approx-error-approxlip} shows the existence of an approximately Lipschitz CNN that is close under $L^\infty$ norm to any approximately Lipschitz target function on $\cS$. 
As a corollary, we obtain the approximation error for the actor update.

\begin{corollary}[Actor approximation]\label{cor:approx-error-pi}
    Suppose \cref{asm:manifold,asm:lip-MDP} hold. For any integers $I \in [2, D]$ and $\Tilde{M}, \Tilde{J}>0$, we let 
    \begin{align*}
        &M=O(\Tilde{M}), ~ L=O(\log{(\Tilde{M}\Tilde{J})}+D+\log D), ~ J=O(D\Tilde{J}), \\ 
        &R_1=(8ID)^{-1}\Tilde{M}^{-\frac{1}{L}}=O(1), ~ \log R_2=O(\log^2 (\Tilde{M}\Tilde{J}) + D\log {(\Tilde{M}\Tilde{J})}),
    \end{align*}
    where $O(\cdot)$ hides a constant depending on $\log L_Q$, $\log \frac{C}{1-\gamma}$, $d$, $\alpha$, $\omega$, $B$, and the surface area $\Area(\cS)$. 
    If $\eta_k$, $\lambda_k$, $\cW_k$ and $\Theta_k$ are as specified in \cref{lem:actor-lip} for all $k\geq 0$, $\epsilon_Q=\frac{C}{1-\gamma}(\overbar{L}_Q + 1)(\tilde{M}\tilde{J})^{-\frac{\alpha}{d}}$, then for any $w_k\in\cW_k$ and $\theta_k\in\Theta_k$, there exists $\theta\in\Theta_{k+1}$ such that 
    \begin{align*}
        \norm{f_{\theta}(\cdot,a) - \lambda_{k+1}g^\star_{k+1}(\cdot,a)}_\infty
        \leq \frac{3\epsilon_Q}{1-\gamma_\rho}.
    \end{align*}
    Here, $g^\star_{k+1}$, $L_Q$ and $\overbar{L}_Q$ are defined as in \cref{lem:opt-pi,lem:lip-Q} and \eqref{def:normalized-lip-Q}.
\end{corollary}

We note that the requirement for proximity constant $\epsilon_Q$ in \cref{cor:approx-error-pi} is the same as the approximation error in \cref{thm:approx-error-Q}. This alignment maintains the consistency of the Lipschitz constraints imposed on $\cW_k$, $\Theta_k$, and $\Theta_{k+1}$. In this case, the actor approximation error is comparable to the critic approximation error, and they both depend on the CNN architecture. Therefore, a large CNN class $\cF=\cF(M,L,J,I,R_1,R_2)$ guarantees the existence of good approximations to both the state-action value function and the policy.

\subsection{Sample Complexity}\label{sec:sample-complexity}
We have demonstrated that CNN approximation can be applied to both the state-action value function and the policy. To make sure that the solutions to the ERM subproblems \eqref{eq:ERM-Q} and \eqref{eq:ERM-pi} indeed provide good approximations for $Q^{\pitheta{k}}$ and $\pi^\star_{k+1}$, the number of samples must be sufficient. 
In this section, we derive the sample complexity for \cref{alg:NPMD}. To be more precise, we consider the expected number of oracle accesses to the transition kernel $\cP$ and the cost function $c$ for \cref{alg:NPMD} to find a policy $\pi_K$ that satisfies $\EE \bigl[V^{\pitheta{K}}(\rho) - V^\star(\rho)\bigr]\leq\epsilon$. 

We keep using the same notation for CNN class $\cF=\cF(M,L,J,I,R_1,R_2)$ and its Lipschitz-restricted version $\cF_\mathrm{Lip}$ as defined in \eqref{eq:nn_class_lip}, as well as the parameter spaces $\cW_{\mathrm{Lip}}$ and $\Theta_{\mathrm{Lip}}$ as denoted in \eqref{eq:param-lip-w} and \eqref{eq:param-lip-theta}.
We show the following lemma that characterizes the number of samples $N$ sufficient for accurate critic update at the $k$-th iteration. 
\begin{theorem}[Critic sample size]\label{thm:sample-complexity-Q}
    Suppose \cref{asm:lip-MDP,asm:manifold} hold. For $k\geq 0$, we let $\eta_k=\frac{1-\gamma_\rho}{C\gamma_\rho^{k+1}}$ and $\cW_k=\cW_{\mathrm{Lip}}(\frac{C}{1-\gamma}, L_Q, \alpha, \epsilon_Q)$ with 
    \begin{align*}
        &M=O(N^\frac{d}{d+2\alpha}), ~ L=O(\log N+D+\log D), ~ J=O(D), ~ I\in[2,D], ~ R_1=O(1), \\
        &\log R_2=O(\log^2 N + D\log N), ~ \epsilon_Q=(L_Q^2+C^2/(1-\gamma)^2)D^{\frac{3\alpha}{2\alpha + d}}N^{-\frac{\alpha}{2\alpha+d}}.
    \end{align*}
    If we take sample size $N=\Tilde{O}(\frac{\sqrt{|\cA|}}{1-\gamma}\gamma_\rho^{-(K+1)})^{\frac{d}{\alpha}+2}$, then $\EE\bigl[\lossQ(w_k;\pitheta{k})\bigr]\leq C^2\gamma_\rho^{2(k+1)}$ holds for all $k\leq K$ in \cref{alg:NPMD}.

    Moreover, if \cref{asm:concen-coefficient,asm:full-support} hold, then for any $\epsilon>0$, it suffices to let 
    \begin{align*}
        N=\Tilde{O}\left(\frac{\kappa C\left(\sqrt{C_\nu}+\log|\cA|\right)\sqrt{|\cA|}}{(1-\gamma)^3\epsilon}\right)^{\frac{d}{\alpha}+2}
    \end{align*}
    so that $\EE\bigl[\lossQ(w_k;\pitheta{k})\bigr]\leq C^2\gamma_\rho^{2(k+1)}$ for all $k\leq K$, where $K$ is the iteration number given in \cref{thm:iteration-complexity} that guarantees $\EE \bigl[V^{\pitheta{k}}(\rho) - V^\star(\rho)\bigr]\leq\epsilon$.     
    Here, $O(\cdot)$ and $\Tilde{O}(\cdot)$ hide the constant depending on $D^{\frac{6\alpha}{2\alpha+d}}$, $\overbar{L}_Q$, $\log L_Q$, $\log\frac{C}{1-\gamma}$, $d$, $\alpha$, $\omega$, $B$, and the surface area $\Area(\cS)$.

\end{theorem}

The proof of \cref{thm:sample-complexity-Q} is provided in \cref{proof:sample-complexity-Q}. As shown in \cref{thm:sample-complexity-Q}, when the iteration number $k$ increases, the required number of samples $N$ grows exponentially at rate $\tilde{O}(\gamma^{-\frac{d}{\alpha}-2})$. By \cref{thm:iteration-complexity}, the total number of iterations is $\tilde{O}(\log\frac{1}{\epsilon})$, so the growing procedure will not continue for too long. As a result, the number of samples for the last iteration is $\tilde{O}(\epsilon^{-\frac{d}{\alpha}-2})$, and the overall sample complexity for critic updates is in the same order up to logarithm terms. Moreover, the exponent over $\epsilon$ is again the intrinsic dimension $d$ instead of the data dimension $D$, which implies avoidance from the curse of dimensionality.

We now turn to derive a similar bound for actor updates. 
As shown in \cref{lem:actor-lip}, with adequately chosen temperature parameters, we can restrict the actor network class with constant parameters for all iterations, and the resulting target actor will have the same approximate Lipschitzness guarantee as the subsequent actor network. Hence we have the following \cref{thm:sample-complexity-pi} characterizing the sufficient sample size for accurate actor updates.
\begin{theorem}[Actor sample size]\label{thm:sample-complexity-pi}
    Suppose \cref{asm:lip-MDP,asm:manifold} hold. 
    For $k\geq 0$, we let $\eta_k=\frac{1-\gamma_\rho}{C\gamma_\rho^{k+1}}$, $\lambda_k=\frac{C\gamma_\rho^{k}}{1-\gamma_\rho}$, $\cW_k=\cW_{\mathrm{Lip}}(\frac{C}{1-\gamma}, L_Q, \alpha, \epsilon_Q)$ and $\Theta_{k}=\Theta_\mathrm{Lip}(\frac{C}{(1-\gamma_\rho)(1-\gamma)}, \frac{L_Q}{1-\gamma_\rho}, \alpha, \frac{\epsilon_Q}{1-\gamma_\rho})$ with 
    \begin{align*}
        &M=O(N^\frac{d}{d+2\alpha}), ~ L=O(\log N+D+\log D), ~ J=O(D), ~ I\in[2,D], ~ R_1=O(1), \\
        &\log R_2=O(\log^2 N + D\log N), ~ \epsilon_Q=(L_Q^2+C^2/(1-\gamma)^2)D^{\frac{3\alpha}{2\alpha + d}}N^{-\frac{\alpha}{2\alpha+d}},
    \end{align*}
    If we take sample size $N=\Tilde{O}\Bigl(\frac{\sqrt{|\cA|}\gamma_\rho^{-(K+1)}}{(1-\gamma_\rho)(1-\gamma)}\Bigr)^{\frac{d}{\alpha}+2}$, then $\EE\bigl[\losspi{k}(\theta_{k+1};\theta_k,w_k)\bigr]\leq (1-\gamma_\rho)^2$ holds for all $k\leq K$ in \cref{alg:NPMD}.

    Moreover, if \cref{asm:full-support,asm:concen-coefficient}, then for any $\epsilon>0$, it suffices to let 
    \begin{align*}
        N=\Tilde{O}\left(\frac{\kappa^2C\left(\sqrt{C_\nu}+\log|\cA|\right)\sqrt{|\cA|}}{(1-\gamma)^4\epsilon}\right)^{\frac{d}{\alpha}+2}
    \end{align*}
    so that $\EE\bigl[\losspi{k}(\theta_{k+1};\theta_k,w_k)\bigr]\leq (1-\gamma_\rho)^2$ for all $k\leq K$, where $K$ is the iteration number given in \cref{thm:iteration-complexity} that guarantees $\EE \bigl[V^{\pitheta{K}}(\rho) - V^\star(\rho)\bigr]\leq\epsilon$. 
    Here, $O(\cdot)$ and $\Tilde{O}(\cdot)$ hide the constant depending on $D^{\frac{6\alpha}{2\alpha+d}}$, $\overbar{L}_Q$, $\log L_Q$, $\log\frac{C}{1-\gamma}$, $d$, $\alpha$, $\omega$, $B$, and the surface area $\Area(\cS)$.
\end{theorem}

The proof of \cref{thm:sample-complexity-pi} is provided in \cref{proof:sample-complexity-pi}, which is similar to the proof of \cref{thm:sample-complexity-Q}. 
Compared to \cref{thm:sample-complexity-Q}, \cref{thm:sample-complexity-pi} requires a $\tilde{O}((1-\gamma_\rho)^{-\frac{d}{\alpha}-2})$ times larger sample size because the target actor in each iteration has a worse approximate Lipschitzness than the target actor. Nevertheless, we can align the sample size to the larger one for actor updates so that both the actor and the critic will be accurate.
Combining the results together, we establish the overall sample complexity for \cref{alg:NPMD}.
\begin{theorem}[Total sample complexity]\label{cor:sample-complexity-NPMD}
    Suppose \cref{asm:lip-MDP,asm:manifold,asm:full-support,asm:concen-coefficient} hold. 
    If for $k\geq 0$, $\eta_k=\frac{1-\gamma_\rho}{C\gamma_\rho^{k+1}}$, $\lambda_k=\frac{C\gamma_\rho^{k}}{1-\gamma_\rho}$, $\cW_k=\cW_{\mathrm{Lip}}(\frac{C}{1-\gamma}, L_Q, \alpha, \epsilon_Q)$ and $\Theta_{k}=\Theta_\mathrm{Lip}(\frac{C}{(1-\gamma_\rho)(1-\gamma)}, \frac{L_Q}{1-\gamma_\rho}, \alpha, \frac{\epsilon_Q}{1-\gamma_\rho})$ with 
    \begin{align*}
        &M=O(N^\frac{d}{d+2\alpha}), ~ L=O(\log N+D+\log D), ~ J=O(D), ~ I\in[2,D], ~ R_1=O(1), \\
        &\log R_2=O(\log^2 N + D\log N), ~ \epsilon_Q^{(k)}=(L_Q^2+C^2/(1-\gamma)^2)D^{\frac{3\alpha}{2\alpha + d}}N^{-\frac{\alpha}{2\alpha+d}}.
    \end{align*}
    Then for $\epsilon>0$, it suffices to set $N=\Tilde{O}\Bigl(\frac{\kappa^2C\left(\sqrt{C_\nu}+\log|\cA|\right)\sqrt{|\cA|}}{(1-\gamma)^4\epsilon}\Bigr)^{\frac{d}{\alpha}+2}$, and the expected number of sample oracle calls for \cref{alg:NPMD} to find a $\pi_K$ satisfying $\EE \bigl[V^{\pitheta{K}}(\rho) - V^\star(\rho)\bigr]\leq\epsilon$ is
    \begin{align*}
        \tilde{O}\Bigl(\kappa^{\frac{2d}{\alpha}+5}C_\nu^{\frac{d}{2\alpha}+1}|\cA|^{\frac{d}{2\alpha}+2} (1-\gamma)^{-\frac{4d}{\alpha}-10}C^{\frac{d}{\alpha}+2}\epsilon^{-\frac{d}{\alpha}-2}\Bigr).
    \end{align*}
    Here, $O(\cdot)$ and $\Tilde{O}(\cdot)$ hide the constant depending on $D^{\frac{6\alpha}{2\alpha+d}}$, $\overbar{L}_Q$, $\log L_Q$, $\log \frac{C}{1-\gamma}$, $d$, $\alpha$, $\omega$, $B$, and the surface area $\Area(\cS)$. 
\end{theorem}
\begin{proof}
    Let $K$ be the iteration number in \cref{thm:iteration-complexity}. By \cref{thm:sample-complexity-Q,thm:sample-complexity-pi}, our specification of $N$ ensures that $\EE\bigl[\lossQ(w_k;\pitheta{k})\bigr]\leq C^2\gamma_\rho^{2(k+1)}$ and $\EE\bigl[\losspi{k}(\theta_{k+1};\theta_k,w_k)\bigr]\leq (1-\gamma_\rho)^2$ for all $k\leq K$. Note that we have $|\cA|$ actions in total, and by \cref{lem:onestep-sample}, each sample requires $O(\frac{1}{1-\gamma})$ oracle calls. As a result, the overall sample complexity is 
    \begin{align*}
        O\Bigl(\frac{KN|\cA|}{1-\gamma}\Bigr)
        &=\tilde{O}\Bigl(\frac{1}{\log\frac{1}{\gamma_\rho}}\kappa^{\frac{2d}{\alpha}+4}C_\nu^{\frac{d}{2\alpha}+1}|\cA|^{\frac{d}{2\alpha}+2}C^{\frac{d}{\alpha}+2} (1-\gamma)^{-\frac{4d}{\alpha}-9}\epsilon^{-\frac{d}{\alpha}-2}\Bigr)\\
        &\leq\tilde{O}\Bigl(\kappa^{\frac{2d}{\alpha}+5}C_\nu^{\frac{d}{2\alpha}+1}|\cA|^{\frac{d}{2\alpha}+2}C^{\frac{d}{\alpha}+2} (1-\gamma)^{-\frac{4d}{\alpha}-10}\epsilon^{-\frac{d}{\alpha}-2}\Bigr),
    \end{align*}
    where the inequality uses $\frac{1}{\log\frac{1}{\gamma_\rho}}\leq\frac{1}{1-\gamma_\rho}=\frac{\kappa}{1-\gamma}$.
\end{proof}

\cref{cor:sample-complexity-NPMD} characterizes the expected number of oracle access to the environment for finding an $\epsilon$-optimal (in the sense of expected value function) policy $\pi_K$. The resulting sample complexity $\tilde{O}\Bigl(\kappa^{\frac{2d}{\alpha}+5}C_\nu^{\frac{d}{2\alpha}+1}|\cA|^{\frac{d}{2\alpha}+2} (1-\gamma)^{-\frac{4d}{\alpha}-10}C^{\frac{d}{\alpha}+2}\epsilon^{-\frac{d}{\alpha}-2}\Bigr)$ has no exponential dependence on the data dimension $D$. In our assumption, $d\ll D$, thus the sample complexity does not suffer from the curse of dimensionality.

\section{Numerical Experiments}\label{sec:experiment}

In this section, we present numerical experiments for NPMD to illustrate that its sample complexity does not necessarily grow exponentially with the ambient dimension $D$. We perform experiments on the CartPole environment \citep{barto1983neuronlike} with visual display. The action space contains $2$ discrete actions. The states are images of the cart and pole rendered from $4$ internal factors indicating the status of the objects. Therefore, the intrinsic dimension $d=4$, while the ambient dimension $D$ scales with the image resolution. We consider three resolutions: low ($3\times 20\times 75$), high ($3\times 40\times 150$), and super high ($3\times 60\times 225$). 

We set the discount factor $\gamma=0.98$. We use $2048$, $4096$, and $8192$ samples for low and high resolutions in each NPMD iteration. For further comparison, we use $8192$ samples per iteration in the super high-resolution setting. 
For the CNN architecture, we set the pairs of kernel size and stride in the first layer as $(7,3)$, $(5,2)$, and $(3,1)$ for low, high, and super high resolutions respectively. The other CNN layers all share kernel size $3$ and stride $1$. 
We run $200$ epochs of SGD to solve the ERM subproblems with batch size $256$ and learning rate $0.001$. 
To evaluate the obtained policy, we generate $32$ independent trajectories using the policy and compute the average of their total rewards until termination or truncation after hitting the maximal reward limit of $200$. We repeat the experiments 5 times with different random seeds, and the results are shown in \cref{fig:result,tab:result-comparison}.

\begin{figure}[htb!]
     \centering
     \begin{subfigure}[b]{0.49\textwidth}
         \centering
         \includegraphics[width=\textwidth]{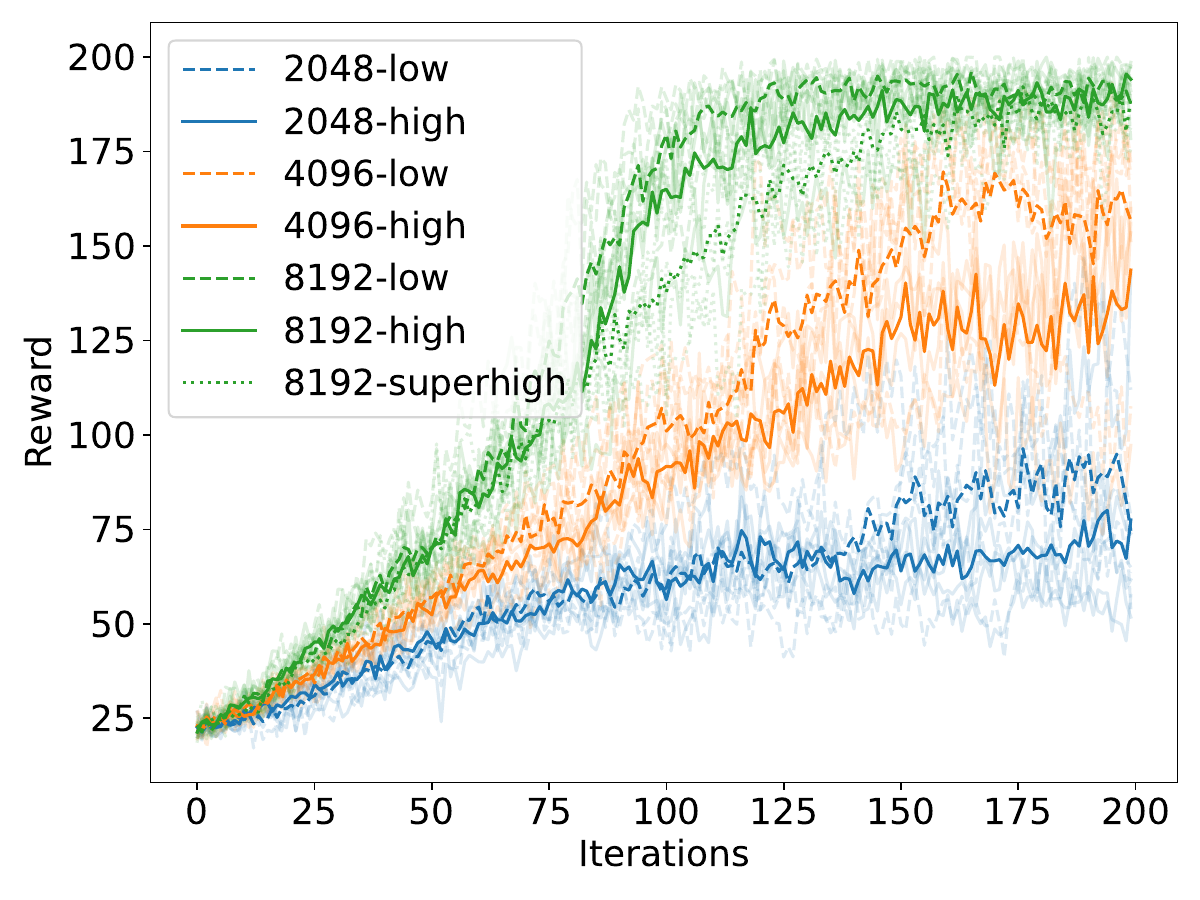}
         \caption{}
         \label{fig:result-iter}
     \end{subfigure}
     \hfill
     \begin{subfigure}[b]{0.49\textwidth}
         \centering
         \includegraphics[width=\textwidth]{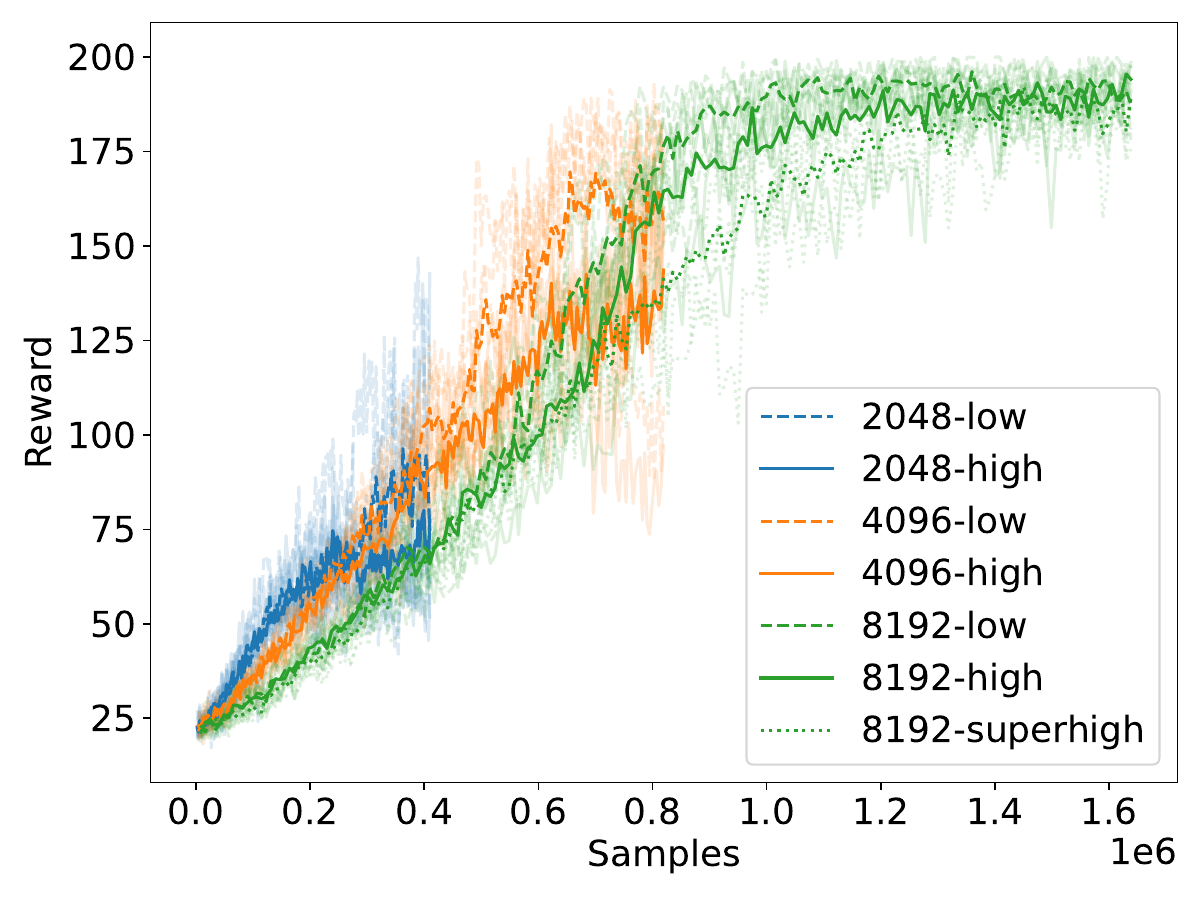}
         \caption{}
         \label{fig:result-sample}
     \end{subfigure}
    \caption{Evaluated rewards in different setups. The opaque lines are the average of 5 runs. }
    \label{fig:result}
\end{figure}

\begin{table}[htb!]
\begin{tabular}{|c|c|c|ccccc|c|}
\hline
\makecell{Sample \\ size $N\cdot|\cA|$}   & Metric               & Resolution & \multicolumn{5}{c|}{Runs (low to high)}                                                                                   & Range      \\ \hline
\multirow{9}{*}{8192} & \multirow{3}{*}{Max}  & Low        & \multicolumn{1}{c|}{199.8} & \multicolumn{1}{c|}{200.0} & \multicolumn{1}{c|}{200.0} & \multicolumn{1}{c|}{200.0} & 200.0 & $199.9\pm 0.1$  \\ \cline{3-9} 
                      &                       & High       & \multicolumn{1}{c|}{198.2} & \multicolumn{1}{c|}{198.3} & \multicolumn{1}{c|}{199.2} & \multicolumn{1}{c|}{199.8} & 200.0 & $199.1\pm 0.9$  \\ \cline{3-9} 
                      &                       & Super high & \multicolumn{1}{c|}{195.2} & \multicolumn{1}{c|}{195.9} & \multicolumn{1}{c|}{196.6} & \multicolumn{1}{c|}{197.2} & 198.2 & $196.7\pm 1.5$  \\ \cline{2-9} 
                      & \multirow{3}{*}{EMA}  & Low        & \multicolumn{1}{c|}{186.1} & \multicolumn{1}{c|}{190.0} & \multicolumn{1}{c|}{190.4} & \multicolumn{1}{c|}{193.1} & 193.4 & $189.8\pm 3.7$  \\ \cline{3-9} 
                      &                       & High       & \multicolumn{1}{c|}{185.2} & \multicolumn{1}{c|}{186.2} & \multicolumn{1}{c|}{190.6} & \multicolumn{1}{c|}{193.7} & 194.2 & $189.7\pm 4.5$  \\ \cline{3-9} 
                      &                       & Super high & \multicolumn{1}{c|}{181.1} & \multicolumn{1}{c|}{181.7} & \multicolumn{1}{c|}{186.4} & \multicolumn{1}{c|}{186.8} & 190.4 & $185.8\pm 4.7$  \\ \cline{2-9} 
                      & \multirow{3}{*}{Last} & Low        & \multicolumn{1}{c|}{178.0} & \multicolumn{1}{c|}{181.3} & \multicolumn{1}{c|}{187.7} & \multicolumn{1}{c|}{193.0} & 198.1 & $188.1\pm 10.1$ \\ \cline{3-9} 
                      &                       & High       & \multicolumn{1}{c|}{188.7} & \multicolumn{1}{c|}{191.1} & \multicolumn{1}{c|}{194.4} & \multicolumn{1}{c|}{197.6} & 198.6 & $193.7\pm 5.0$  \\ \cline{3-9} 
                      &                       & Super high & \multicolumn{1}{c|}{175.7} & \multicolumn{1}{c|}{187.7} & \multicolumn{1}{c|}{190.5} & \multicolumn{1}{c|}{194.0} & 197.5 & $186.6\pm 10.9$ \\ \hline
\end{tabular}
    \caption{Comparison of 5 runs in different resolutions with per iteration sample size $8192$. The evaluation metrics include the maximal reward along the trajectory (Max), the exponential moving average of history rewards (EMA), and the last iterate reward (Last). The EMA is computed through $EMA_{k}=0.9* EMA_{k-1}+0.1* R_k$, where $R_k$ is the reward at the $k$-th iteration. }
    \label{tab:result-comparison}
\end{table}

The numerical results show that NPMD exhibits comparable performance across varying image resolutions, where the intrinsic structure of the CartPole environment remains the same. While the ambient dimension $D$ still affects the performance, its impact is not as much as an exponential function would do: Suppose the sample complexity depends exponentially on $D$, then moving from low to high or from high to super high resolution will increase $D$ by factor $4$, resulting in a quartic growth in sample complexity. However, the numerical results suggest that much fewer samples are required to get the same near-optimal performance (see \cref{fig:result-iter,tab:result-comparison}). Therefore, the sample complexity does not exhibit exponential dependence on $D$. 
We also observe that training becomes more stable as the per iteration sample size increases, and the average performance becomes better as the total sample size increases (see \cref{fig:result-sample}). These empirical results confirm our upper bound in \cref{cor:sample-complexity-NPMD}, which does not depend exponentially on $D$. 




\section{Conclusion and Discussion}\label{sec:conclusion}
We have derived the overall sample complexity for NPMD (\cref{alg:NPMD}) on Lipschitz MDP with intrinsically low-dimensional state space $\cS$. 
Our result gives a concrete characterization of the expected number of samples required for an $\epsilon$-optimal policy under mild regularity assumptions and shows no curse of dimensionality. We make a few remarks about this result.

\textbf{Tightness in $\epsilon$.}
The sample complexity $\tilde{O}(\epsilon^{-\frac{d}{\alpha}-2})$ can be interpreted as two parts. The first part $\tilde{O}(\epsilon^{-2})$ comes from iterations of the NPMD algorithm and is optimal up to logarithm terms. It matches the complexity of PMD in tabular case \citep{lan2023policy} and NPG on linear MDP \citep{yuan2023linear}. The second part $\tilde{O}(\epsilon^{-\frac{d}{\alpha}})$ comes from function approximation on $d$-dimensional state space manifold. Intuitively, it matches the number of states $|\cS_{\mathrm{dis}}|$ appeared in the complexity of tabular PMD if we discretize the continuous state space into a finite set of points $\cS_{\mathrm{dis}}\subset\cS$.\footnote{We note that directly performing discretization is difficult when $\cS$ has a complicated geometric structure. } This part scales with the intrinsic dimension $d$ of $\cS$ and can be interpreted as the result of neural networks adapting to the low-dimensional state space geometry. It is yet to be examined whether the overall complexity is tight in $\epsilon$.

\textbf{Dependence on the cost function scale.}
The sample complexity reasonably depends on $C^{-1}\epsilon$, which can be viewed as the relative error. Therefore, as long as $\epsilon$ scales with the cost function, the term $C^{\frac{d}{\alpha}+2}\epsilon^{-\frac{d}{\alpha}-2}$ in the complexity bound remains the same. Although the scaling of the cost function can still affect the hidden constant depending on $\log L_Q=\log{(L_c+\frac{\gamma C}{1-\gamma}L_\cP)}$ and $\log\frac{C}{1-\gamma}$, it will not have a major impact since the dominating term is the normalized Lipschitz constant $\overbar{L}_Q$, which is invariant to the scaling of the cost function. 

\textbf{Distribution mismatch and concentrability.} 
The distribution mismatch coefficient $\kappa$ in \eqref{def:kappa} and concentrability coefficient $C_\nu$ in \cref{asm:concen-coefficient} have been widely used in the analysis of PMD-type methods in both the tabular setting \citep{xiao2022convergence} and linear function approximation \citep{agarwal2021theory,alfano2022linear,yuan2023linear}. The mismatch coefficient $\kappa$ occurs when the initial distribution $\rho$ is different from the optimal visitation distribution $\nu_\rho^{\pi^\star}$. When \cref{asm:full-support} does not hold, $\kappa$ can possibly be infinity, leading to a vacuous complexity bound. This mismatch coefficient seems unavoidable for analysis involving the performance difference lemma (\cref{lem:performance-difference}) even in the tabular setting \citep{xiao2022convergence} and can affect the iteration complexity through $\gamma_\rho$ as shown in \cref{thm:iteration-complexity}. Recently, \citet{johnson2023optimal} have established an alternative analysis that avoids the use of the performance difference lemma. Their result removes the mismatch coefficient for tabular PMD by using an adaptive step size. However, their sophisticated step size rule does not apply to our analysis, as it will make the smoothness of target actors uncontrollable. 

The concentrability coefficient $C_\nu$ comes from the change of error measures. 
In the $k$-th iteration of NPMD, we cannot directly sample states from $\nu_\rho^{\pitheta{k+1}}$ or $\nu_\rho^{\pi^\star}$ because the corresponding policies $\pitheta{k+1}$ and $\pi^\star$ are not available yet. 
Instead, we sample states from $\nu_\rho^{\pitheta{k}}$ and solve the ERM subproblems with these samples. As a result, the actor and critic errors are naturally measured in $\nu_\rho^{\pitheta{k}}$, and $C_\nu$ comes in when measuring the errors in $\nu_\rho^{\pitheta{k+1}}$ and $\nu_\rho^{\pi^\star}$, which are related to the performance difference between consecutive policies. 
The concentrability coefficient is unavoidable when we have function approximation errors measured in the $L^2$ norm. To remove $C_\nu$, one has to consider the exact PMD case where there is no function approximation at all, or to derive $L^\infty$ error bounds for critic and actor updates, both are intractable for continuous state space. 

As suggested in \citet{yuan2023linear} for finite state space, both coefficients can be potentially improved if we decouple the sampling distribution in \cref{alg:NPMD} and the evaluation distribution in the policy optimization problem \eqref{eq:policy-optimization}. Indeed, one can replace $\rho$ in \eqref{eq:policy-optimization} by another distribution $\rho^\prime$, and replace the initial distribution in \cref{alg:sampler} by $\varrho$. In this case, we can choose $\rho^\prime$ with full support to satisfy \cref{asm:full-support} and $\kappa$ becomes $\kappa^\prime=\Bigl\|{\dv{\nu_{\rho^\prime}^{\pi^\star}}{\rho^\prime}}\Bigr\|_\infty$. The optimality gap can be translated as 
\begin{align*}
    V^{\pitheta{k}}(\rho)-V^\star(\rho)\leq \norm{\dv{\rho}{\rho^\prime}}_\infty(V^{\pitheta{k}}(\rho^\prime)-V^\star(\rho^\prime)).
\end{align*}
Similarly, \cref{asm:concen-coefficient} can be replaced by 
\begin{align*}
    \chisqr{\nu_{\rho^\prime}^{\pi}}{\nu_\varrho^{\pitheta{k}}}+1\leq C_\nu^\prime
\end{align*}
for some $C_\nu^\prime$. However, we note that the pathological behavior of distribution mismatch and concentrability in the continuous state space is not always removable as it requires the Radon--Nikodym derivatives to exist, and it is difficult to get $\kappa^\prime<\kappa$ and $C_\nu^\prime<C_\nu$ for a better complexity. To the best of our knowledge, it remains an open problem to study function approximation policy gradient methods in continuous state space without \cref{asm:full-support,asm:concen-coefficient}.

\textbf{Dependence on the action space.}
The sample complexity depends on ${|\cA|}^{\frac{d}{2\alpha}+2}$, which comes from two parts. One is from the aggregation of $|\cA|$ networks, and the other part ${|\cA|}^{\frac{d}{2\alpha}+1}$ is from function approximation, as $|\cA|$ appears in the approximation error after critic/actor loss translation (see \cref{sec:appendix-translation-Q,sec:appendix-translation-pi}). 
We notice that some concurrent works do not have such dependence on $|\cA|$, e.g. \citet{yuan2023linear} and \citet{alfano2023novel}. 
Their results are based on the \emph{state-action} concentrability condition (see e.g. Assumption 6 in \citealt{yuan2023linear}), which requires bounded $\chi^2$-divergence between state-action visitation distributions, rather than just state visitation distributions. 
Consequently, it is generally stronger than our \cref{asm:concen-coefficient}. 

Furthermore, these works directly make realizability assumptions on target functions and do not consider the approximation bias, which is the main focus of our work and is the cause of the polynomial dependence on the size of the action space. To avoid the $|\cA|$ factor in error translation, it is necessary to consider fitting a single target function (either the value function or the target actor) defined over the product space $\cS\times\cA$, rather than fitting $|\cA|$ functions on $\cS$ separately. This requires the target function to be smooth in both states and actions. However, environments with discrete action space often exhibit non-smoothness in actions, since many of these actions have opposite functionalities (e.g. move left/right). Therefore, it is difficult to derive guarantees for function approximation even if we assume state-action concentrability. 
Nevertheless, we conjecture that the second part, ${|\cA|}^{\frac{d}{2\alpha}+1}$, might be removable, and we leave it for future investigation. 

\textbf{Computational concerns.} In each iteration of \cref{alg:NPMD}, there are two ERM subproblems with approximately Lipschitz constraints, which are assumed to be solvable. However, solving such problems is non-trivial due to the non-convex objectives and the approximately Lipschitz constraints. In practice, one can use gradient descent (GD) and its variants to minimize the objectives, but little is known about their theoretical convergence behavior. Recently there has been some work studying the GD dynamics in the NTK or mean-field regime \citep{jacot2018neural,song2018mean}, but the gap remains as they cannot fully explain the convergence behavior of GD-like methods on deep models, and their results cannot adapt to the low-dimensional manifold structure. 

Meanwhile, there is no existing result on optimization with approximate Lipschitzness constraints. One can apply Lipschitz regularization methods, such as spectral regularization \citep{yoshida2017spectral,gogianu2021spectral}, gradient regularization \citep{gulrajani2017improved}, projected gradient descent for Lipschitz constant constraint \citep{gouk2021regularisation}, and adversarial training \citep{miyato2018virtual}. Most of these techniques are heuristic, so they cannot exactly control the Lipschitzness of networks. Nevertheless, they could result in approximately Lipschitz networks as it is a relaxed condition. In addition, recent studies have discovered that GD itself has some algorithmic regularization effect, which implicitly controls the Lipschitzness of the learned networks \citep{mulayoff2021implicit}. We leave the study of approximately Lipschitz-constrained optimization of neural networks for future work.

\textbf{Overparameterization.} In modern deep learning practice, there exists a propensity to employ overparameterized models that have more parameters than the number of data. 
Our current analysis is based on the classical bias-variance trade-off argument and cannot handle the overparameterization case. 
Recently, \citet{zhang2022deep} have established deep non-parametric regression results that apply to overparameterized models, but their work does not exploit the low-dimensional structure. It is an interesting future direction to examine if their work can be extended to the manifold setting and fit into our analysis. We expect it to close the theory-practice gap in DRL further. 

\textbf{Comparison with value-based methods.} The sample complexity result for NPMD matches the bound for the value-based FQI method in \citet{fan2020theoretical} when $d=D$, while our result is significantly better when the intrinsic dimension $d\ll D$. 
This shows that policy-based methods can achieve as good performance as value-based methods in theory. 
From a technical perspective, value-based methods only approximate smooth value functions (under the Bellman closedness assumption in \citealt{fan2020theoretical,nguyen-tang2022on}). On the other hand, policy-based methods require repetitively approximating new policies, whose Lipschitz constant will accumulate. We address the issue by introducing the notion of approximate Lipschitzness, imposing approximately Lipschcitz constraints on the neural networks, and establishing approximation theory for them. Our analysis framework can be applied to more general scenarios where there are iterative refittings of neural networks. 


\textbf{Beyond Lipschitz MDP.} In this paper, we work on Lipschitz MDP (\cref{asm:lip-MDP}). In practice, the MDP can be either smoother or not as smooth as Lipschitz MDP. For the former case, one can consider H\"{o}lder smooth MDP with higher exponent, namely $\alpha>1$, and expect there is a better sample complexity. 
If we only consider the policy evaluation and value-based algorithms, where the target value function is smooth, then this is possible as suggested by the results from deep supervised learning \citep{chen2022nonparametric}. However, for policy-based methods, it is unclear whether neural networks that uniformly approximate H\"{o}lder functions can have smoothness beyond approximate Lipschitzness, given that networks with ReLU activation are not differentiable. It is a future direction to examine the sample complexity of policy-based methods in smoother MDPs. 
For the latter case, one can consider extending the Lipschitz condition to the more general Sobolev or Besov conditions to deal with spatial inhomogeneity in smoothness. Also, as mentioned in \cref{rmk:lip-MDP,rmk:smooth-MDP}, one can use a smooth approximation of the non-smooth MDP as a surrogate in this case.  


\acks{The authors thank Yan Li for the discussion in the early stage of this work. }




\appendix


\section{Algorithms}\label{sec:appendix-algorithm}
This section presents the missing sampling algorithm (\cref{alg:sampler}) for \cref{alg:NPMD}, along with some auxiliary results related to \cref{alg:NPMD,alg:sampler}. 

\begin{algorithm2e}[htb]
    \caption{Sample $s\sim \nu_\rho^\pi$ or $(s,a)\sim\overbar{\nu}_\rho^{\pi}$}
    \label{alg:sampler}
    \KwIn{Distribution $\rho$, policy $\pi$, factor $\gamma\in(0,1)$}
    Initialize $\mathtt{flag} = \mathtt{True}$, $t=0$, $s_0\sim\rho$, $a_0\sim\pi(\cdot|s_0)$\;
    \While{$\mathtt{flag}$ is $\mathtt{True}$}{
        Sample $p\sim\mathrm{Unif}([0,1])$\;
        \uIf{$p\leq\gamma$}{
            Sample $s_{t+1}\sim\cP(\cdot|s_t,a_t)$\;
            Sample $a_{t+1}\sim\pi(\cdot|s_{t+1})$\;
            $t\gets t+1$\;
        }
        \Else{
            $\mathtt{flag} = \mathtt{False}$\;
        }
    }
    \KwOut{$s_t$ as $s$ or $(s_t,a_t)$ as $(s,a)$}
\end{algorithm2e}


\subsection{Sample Complexity of Algorithm \ref{alg:sampler}}
We compute the expected number of sample oracle calls of \cref{alg:sampler} to get $s\sim \nu_\rho^\pi$ or $(s,a)\sim\overbar{\nu}_\rho^{\pi}$.
\begin{lemma}\label{lem:onestep-sample}
    \cref{alg:sampler} returns $(s,a)\sim\overbar{\nu}_\rho^{\pi}$, and the expected number of sample oracle calls for each pair of $(s,a)$ is $\frac{1}{1-\gamma}$.
\end{lemma}
\begin{proof}
    Let $T$ be the terminating time (trajectory length) of \cref{alg:sampler}, which has probability 
    \begin{align*}
        \prob(T=t)=\gamma^t(1-\gamma).
    \end{align*}
    The probability distribution of the output $s_t$ is then given by 
    \begin{align*}
        \cP_\mathrm{out}
        =\sum_{t=0}^\infty \cP_t^\pi\cdot\prob(T=t)
        =\sum_{t=0}^\infty \cP_t^\pi\cdot\gamma^t(1-\gamma), 
    \end{align*}
    which is exactly $\nu_\rho^\pi$ according to \eqref{eq:visitation-distribution-s}. Since $\prob(a_t=a|s_t)=\pi(a|s_t)$, we have $(s,a)\sim\overbar{\nu}_\rho^{\pi}$.
    The expected number of sample oracle calls is $\EE[T+1]$, which is
    \begin{align*}
        \EE[T+1]
        =\sum_{t=0}^\infty (t+1) \gamma^t(1-\gamma)
        =\frac{1}{1-\gamma}.
    \end{align*}
    The expected trajectory length $\frac{1}{1-\gamma}$ is also called the \emph{effective horizon}.
\end{proof}

\subsection{Critic Loss Translation}\label{sec:appendix-translation-Q}
The critic loss can be translated to the mean squared error (MSE) of a regression problem.  
To see this, recall the state-action value function of the form \eqref{eq:Q-visitation}: 
\[
    Q^{\pi}(s,a)
    =c(s,a) + \frac{\gamma}{1-\gamma}\EE_{(s^\prime,a^\prime)\sim\overbar{\nu}_{\cP(\cdot|s,a)}^\pi}\bigl[c(s^\prime,a^\prime)\bigr].
\]
For any fixed $a\in\cA$ and $k\geq 0$, let $X\sim \nu_\rho^{\pitheta{k}}$ be a random variable on $\cS$ and $Z=(S,A)$ be another random variable with conditional distribution $p_{Z|X}=\overbar{\nu}_{\cP(\cdot|X,a)}^{\pitheta{k}}$. 
With this notation, reformulating \eqref{eq:Q-visitation} gives 
\begin{align}\label{eq:expectation}
    c(X, a) + \frac{\gamma}{1-\gamma}c(S, A)=Q^\pitheta{k}(X,a)+\zeta,
\end{align}
where 
\begin{equation}\label{eq:noise}
    \zeta=\frac{\gamma}{1-\gamma}\Bigl(c(S, A) - \EE_{(s^\prime,a^\prime)\sim\overbar{\nu}_{\cP(\cdot|X,a)}^{\pitheta{k}}}\bigl[c(s^\prime,a^\prime)\bigr]\Bigr)
\end{equation}
is a random variable. We can verify several properties of the noise term $\zeta$.

\begin{lemma}\label{lem:noise}
    The noise term $\zeta$ defined in \eqref{eq:noise} is a zero-mean sub-Gaussian random variable with variance proxy $\sigma^2 =\frac{\gamma^2C^2}{4(1-\gamma)^2}$ and is uncorrelated with $X$. 
\end{lemma}
\begin{proof}
    By our construction, the randomness of $\zeta$ comes from the randomness of $X$ and $Z=(S,A)$. 
    When $X=x$ is fixed, $\overbar{\nu}_{\cP(\cdot|x,a)}^{\pitheta{k}}$ is a fixed distribution, so we have 
    \begin{align*}
        \EE[\zeta\mid X=x]
        &=\frac{\gamma}{1-\gamma}\EE_{Z\sim p_{Z|X}}\Bigl[c(S, A) - \EE_{(s^\prime,a^\prime)\sim\overbar{\nu}_{\cP(\cdot|x,a)}^{\pitheta{k}}}[c(s^\prime,a^\prime)]\mid X=x\Bigr]\\
        &=\frac{\gamma}{1-\gamma}\Bigl(\EE_{(S,A)\sim\overbar{\nu}_{\cP(\cdot|x,a)}^{\pitheta{k}}}[c(S, A)] - \EE_{(s^\prime,a^\prime)\sim\overbar{\nu}_{\cP(\cdot|x,a)}^{\pitheta{k}}}[c(s^\prime,a^\prime)]\Bigr)\\
        &=0.
    \end{align*} 
    Therefore, $\EE[\zeta\mid X]=0$, 
    \begin{align*}
        \EE[\zeta]=\EE_{X}[\EE[\zeta\mid X]]=\EE_X[0]=0, 
    \end{align*}
    $\zeta$ is uncorrelated with $X$. 
    
    By \eqref{eq:V-bound} we know that given any realization of $X$, \[
        \frac{-\gamma\mu}{1-\gamma}
        \leq \zeta 
        \leq \frac{\gamma(C-\mu)}{1-\gamma},
    \] where $\mu=\EE_{Z|X}\bigl[c(S,A)\bigr]$, thus $\zeta$ is sub-Gaussian with variance proxy $\sigma^2=\frac{\gamma^2C^2}{4(1-\gamma)^2}$.
\end{proof}

Note that the target values $\{c(s_{a,i},a)+\frac{\gamma}{1-\gamma}c(s^\prime_{a,i}, a^\prime_{a,i})\}_{i=1}^N$ in the empirical risk \eqref{eq:L2-empirical-Q} are i.i.d. copies of the left-hand side of \eqref{eq:expectation}, and the right-hand side of \eqref{eq:expectation} is a function plus a noise term. Therefore, the ERM subproblem \eqref{eq:ERM-Q} can be viewed as $|\cA|$ independent regression problems, each corresponding to an action $a\in\cA$. 
It follows immediately from the definition of the critic loss \eqref{eq:loss-Q} that 
\begin{align}\label{eq:loss-Q-translation}
    \EE\sbr{\lossQ(w_k;\pitheta{k})}
    &=\EE\sbr{\EE_{s\sim \nu_\rho^{\pitheta{k}}}\norm{Q_{w_k}(s,a)-Q^{\pitheta{k}}(s,a)}_2^2}\nonumber\\
    &=\EE\biggl[{\EE_{s\sim \nu_\rho^{\pitheta{k}}}\sum_{a\in\cA}\abs{Q_{w_k}(s,a)-Q^{\pitheta{k}}(s,a)}^2}\biggr]\nonumber\\
    &\leq|\cA|\max_{a\in\cA}\EE\sbr{\EE_X\left[\abs{Q_{w_k}(X,a)-Q^\pitheta{k}(X,a)}^2\right]},
\end{align}
where the outer expectation is taken with respect to the estimated $w_k$, which depends on the samples used for ERM. 
As a result, the critic loss of $Q_{w_k}$ can be upper bounded as long as we can bound the MSE of the regression problem for every $a\in\cA$. We provide statistical results for the regression problem in \cref{sec:appendix-nonparametric}. 

\subsection{Actor Loss Translation}\label{sec:appendix-translation-pi}
The actor loss can also be translated. 
We already know the exact solution to \eqref{eq:PMD-approximate} is given by $g^\star_{k+1}$ defined in \cref{lem:opt-pi}. The actor loss is thus translated into the error between the estimated function $f_{\theta_{k+1}}$ and the ground truth function $\lambda_{k+1}g^\star_{k+1}$:
\begin{align}
    &\EE\sbr{\losspi{k}(\theta_{k+1};\theta_k,w_k)}\nonumber\\
    =&\EE\sbr{\EE_{s\sim \nu_\rho^{\pitheta{k}}}\norm{\lambda_{k+1}^{-1}f_{\theta_{k+1}}(s,\cdot)-\lambda_k^{-1}f_\theta(s,\cdot)+\eta_k Q_{w_k}(s,\cdot)}_2^2}\nonumber\\
    =&\EE\biggl[{\EE_{s\sim \nu_\rho^{\pitheta{k}}}\sum_{a\in\cA}\abs{\lambda_{k+1}^{-1}f_{\theta_{k+1}}(s,a)-\lambda_k^{-1}f_\theta(s,a)+\eta_k Q_{w_k}(s,a)}^2}\biggr]\nonumber\\
    \leq &\frac{|\cA|}{\lambda_{k+1}^2}\max_{a\in\cA}\EE\sbr{\EE_{s\sim \nu_\rho^{\pitheta{k}}}\abs{f_{\theta_{k+1}}(s,a)-\frac{\lambda_{k+1}}{\lambda_k}f_\theta(s,a)+\lambda_{k+1}\eta_k Q_{w_k}(s,a)}^2},\label{eq:loss-pi-translation}
\end{align}
where the outer expectation is taken with respect to $\theta_{k+1}$, which depends on the samples used for empirical risk minimization.
Once we derive an upper bound for the MSE of $f_{\theta_{k+1}}$, an upper bound for the actor update loss follows immediately.

\section{Proofs in Section \ref{sec:NPMD}}
\subsection{Proof of Lemma \ref{lem:opt-pi}}\label{proof:opt-pi}
\begin{proof}
    For any $s\in\cS$, \eqref{eq:PMD-approximate} is a convex problem. Its Karush-Kuhn-Tucker condition yields 
    \[
        Q_{w_k}(s,\cdot) -\frac{1}{\eta_k}\log\pitheta{k}(s,\cdot) + \frac{1}{\eta_k}\log\pi^\star_{k+1}(s,\cdot) + \mu_s^\star\one =0
    \]
    for some $\mu_s^\star\in\RR$, where we denote $\pi^\star_{k+1}$ as the solution. This means for any $s\in\cS$, 
    \[
        \pi^\star_{k+1}(a|s)\propto \exp(\log\pitheta{k}(a|s)-\eta_k Q_{w_k}(s,a)).
    \]
    By definition \eqref{eq:pitheta}, we have \[
        \log\pitheta{k}(a|s)
        =\lambda_k^{-1}f_{\theta_k}(s,a)-\log\sum_{a^\prime\in\cA}\exp(\lambda_k^{-1}f_{\theta_k}(s,a^\prime)).
    \]
    Since $\pi^\star_{k+1}(\cdot|s)$ is shift-invariant with $g_{k+1}^\star(s,\cdot)$, that is, \[
        \pi^\star_{k+1}(a|s)
        =\frac{\exp(g_{k+1}^\star(s,a))\exp(p(s))}{\sum_{a^\prime\in\cA}\exp(g_{k+1}^\star(s,a^\prime))\exp(p(s))}
        =\frac{\exp(g_{k+1}^\star(s,a) + p(s))}{\sum_{a^\prime\in\cA}\exp(g_{k+1}^\star(s,a^\prime) + p(s))}
    \]
    for any $p(s)$ independent of $a$, choosing $g_{k+1}^\star=\lambda_k^{-1}f_{\theta_k}-\eta_k Q_{w_k}$ suffices.
\end{proof}

\section{Proofs for Iteration Complexity}
In this section, we present the missing proofs in \cref{sec:iteration-complexity} and auxiliary lemmas for them.

\subsection{Supporting Lemmas for Iteration Complexity}
We first recall the performance difference lemma for the value function, which measures the difference between value functions under different policies.
\begin{lemma}[Performance difference lemma]\label{lem:performance-difference}
    For any pair of policies $\pi$ and $\pi^\prime$ and any state $s$, we have
    \begin{align*}
        V^{\pi^\prime}(\rho)-V^\pi(\rho)=\frac{1}{1-\gamma}\EE_{s\sim \nu_\rho^{\pi^\prime}}\left[\inner{Q^\pi(s,\cdot)}{\pi^\prime(\cdot|s)-\pi(\cdot|s)}\right].
    \end{align*}
\end{lemma}
\begin{proof}
    The performance difference between $\pi^\prime$ and $\pi$ is given by 
    \begin{align*}
        &V^{\pi^\prime}(\rho)-V^\pi(\rho)\\
        =& V^{\pi^\prime}(\cP_0^{\pi^\prime})-V^\pi(\cP_0^{\pi^\prime})\\
        =&\EE_{s\sim\cP_0^{\pi^\prime}}\sum_{a\in\cA}\left(c(s,a)\left(\pi^\prime(a|s)-\pi(a|s)\right)+\gamma\int_\cS\left(V^{\pi^\prime}(s^\prime)\pi^\prime(a|s) - V^\pi(s^\prime)\pi(a|s)\right)\ud\cP(s^\prime|s,a)\right)\\
        =&\EE_{s\sim\cP_0^{\pi^\prime}}\sum_{a\in\cA}\left(c(s,a) + \gamma\int_\cS V^\pi(s^\prime)\ud\cP(s^\prime|s,a)\right)\left(\pi^\prime(a|s)-\pi(a|s)\right) \\ 
        &\quad + \gamma\EE_{s\sim\cP_0^{\pi^\prime}}\sum_{a\in\cA}\pi^\prime(a|s)\int_\cS\left(V^{\pi^\prime}(s^\prime) - V^\pi(s^\prime)\right)\ud\cP(s^\prime|s,a)\\
        =&\EE_{s\sim\cP_0^{\pi^\prime}}\inner{Q^\pi(s,\cdot)}{\pi^\prime(\cdot|s)-\pi(\cdot|s)} + \gamma\EE_{s\sim\cP_0^\pi,a\sim\pi^\prime(\cdot|s)}\int_\cS\left(V^{\pi^\prime}(s^\prime) - V^\pi(s^\prime)\right)\ud\cP(s^\prime|s,a)\\
        =&\EE_{s\sim\cP_0^{\pi^\prime}}\inner{Q^\pi(s,\cdot)}{\pi^\prime(\cdot|s)-\pi(\cdot|s)} + \gamma\left(V^{\pi^\prime}(\cP_1^{\pi^\prime}) - V^\pi(\cP_1^{\pi^\prime})\right)\\
        =&\sum_{t=0}^\infty \gamma^t\EE_{s\sim\cP_t^{\pi^\prime}}\inner{Q^\pi(s,\cdot)}{\pi^\prime(\cdot|s)-\pi(\cdot|s)}.
    \end{align*}
    The first equality uses $\cP_0^{\pi^\prime}\equiv\rho$. The second equality is from \eqref{eq:V-to-Q} and \eqref{eq:Q-to-V}. The fourth equality is from the opposite direction of \eqref{eq:Q-to-V}. The last two lines are from the recursion \eqref{eq:prob-transition} and the sequence converges because the value functions are bounded by \eqref{eq:V-bound}. Plugging in the definition of visitation distribution completes the proof. 
\end{proof}

Our NPMD algorithm uses KL divergence and softmax neural policies. In this case, we have the following lemma that can replace the approximated state-action value function $Q_{w_k}$ that appeared in the inner product by the difference of log-policies. 
\begin{lemma}\label{lem:diff-pistar}
    For any pair of policies $\pi$ and $\pi^\prime$ and any state $s$, we have 
    \begin{align*}
        \inner{Q_{w_k}(s,\cdot)}{\pi(\cdot|s)-\pi^\prime(\cdot|s)}
        &=-\frac{1}{\eta_k}\inner{\log\pi^\star_{k+1}(s,\cdot)-\log\pitheta{k}(s,\cdot)}{\pi(\cdot|s)-\pi^\prime(\cdot|s)}\\
        &=-\frac{1}{\eta_k}\inner{\nabla h^{\pi^\star_{k+1}}(s,\cdot)-\nabla h^{\pitheta{k}}(s,\cdot)}{\pi(\cdot|s)-\pi^\prime(\cdot|s)},
    \end{align*}
    where $h^\pi(s)\coloneqq\inner{\log\pi(s,\cdot)}{\pi(\cdot|s)}$ is the negative entropy of $\pi$ at $s$ and $\nabla h^\pi(s,\cdot)\in\RR^{|\cA|}$ is the gradient taken with respect to the actions, $\pi^\star_{k+1}$ is the exact solution of \eqref{eq:PMD-approximate}. 
\end{lemma}
\begin{proof}
    For any policies $\pi$ and $\pi^\prime$ and any state $s$ we have $\inner{\one}{\pi(\cdot|s)-\pi^\prime(\cdot|s)}=0$. By using \cref{lem:opt-pi} and the definition of the KL divergence, we obtain the result.
\end{proof}

The next lemma bounds the expected error (measured in $\nu_\rho^\pitheta{k+1}$ or $\nu_\rho^{\pi^\star}$) of replacing $Q^\pitheta{k}$ with $Q_{w_k}$ by the critic loss. 
\begin{lemma}\label{lem:change-measure-err-Q}
    Suppose \cref{asm:full-support,asm:concen-coefficient} hold. If $\pi$ is $\pitheta{k+1}$ or $\pi^\star$, then for any pair of policies $\pi^\prime$ and $\pi^{\prime\prime}$, we have 
    \begin{align*}
        \abs{\EE_{s\sim\nu_\rho^\pi}\left[\inner{Q^\pitheta{k}(s,\cdot)-Q_{w_k}(s,\cdot)}{\pi^\prime(\cdot|s)-\pi^{\prime\prime}(\cdot|s)}\right]}\leq 2\sqrt{C_\nu\lossQ(w_k;\pitheta{k})}.
    \end{align*}
\end{lemma}
\begin{proof}
    We have 
    \begin{align*}
        &\abs{\EE_{s\sim\nu_\rho^\pi}\left[\inner{Q^\pitheta{k}(s,\cdot)-Q_{w_k}(s,\cdot)}{\pi^\prime(\cdot|s)-\pi^{\prime\prime}(\cdot|s)}\right]}\\
        \leq &\EE_{s\sim\nu_\rho^\pi}\abs{\inner{Q^\pitheta{k}(s,\cdot)-Q_{w_k}(s,\cdot)}{\pi^\prime(\cdot|s)-\pi^{\prime\prime}(\cdot|s)}}\\
        \leq &\EE_{s\sim\nu_\rho^\pi}\norm{Q^\pitheta{k}(s,\cdot)-Q_{w_k}(s,\cdot)}_\infty\norm{\pi^\prime(\cdot|s)-\pi^{\prime\prime}(\cdot|s)}_1\\
        \leq &2\EE_{s\sim\nu_\rho^\pi}\norm{Q^\pitheta{k}(s,\cdot)-Q_{w_k}(s,\cdot)}_\infty, 
    \end{align*}
    where the third line is by H\"{o}lder's inequality. 
    By \cref{asm:full-support} and \eqref{eq:visitation-bound}, the visitation distributions have full support, hence $\dv{\nu_\rho^{\pi}(s)}{\nu_\rho^{\pitheta{k}}(s)}$ exists. By using Cauchy--Schwarz inequality and \cref{asm:concen-coefficient}, we can replace the error measure $\nu_\rho^\pi$ by $\nu_\rho^{\pitheta{k}}$: 
    \begin{align*}
        &\EE_{s\sim \nu_\rho^{\pi}}\norm{Q^\pitheta{k}(s,\cdot)-Q_{w_k}(s,\cdot)}_\infty\\
        \leq& \sqrt{\EE_{s\sim \nu_\rho^{\pi}}\left[\dv{\nu_\rho^{\pi}(s)}{\nu_\rho^{\pitheta{k}}(s)}\right]\EE_{s\sim \nu_\rho^{\pi}}\left[\dv{\nu_\rho^{\pitheta{k}}(s)}{\nu_\rho^{\pi}(s)}\norm{Q^\pitheta{k}(s,\cdot)-Q_{w_k}(s,\cdot)}_\infty^2\right]}\\
        =& \sqrt{\left(\chisqr{\nu_\rho^{\pi}}{\nu_\rho^{\pitheta{k}}}+1\right)\bigl(\EE_{s\sim \nu_\rho^{\pitheta{k}}}\norm{Q^\pitheta{k}(s,\cdot)-Q_{w_k}(s,\cdot)}_\infty^2\bigr)}\\
        \leq& \sqrt{C_\nu \lossQ(w_k;\pitheta{k})}.
    \end{align*}
    Plugging this back, we obtain the result.
\end{proof}

Analogously, the next lemma bounds the expected error (measured in $\nu_\rho^\pitheta{k+1}$ or $\nu_\rho^{\pi^\star}$) of replacing $\pi^\star_{k+1}$ with $\pitheta{k+1}$ by the actor loss. 
\begin{lemma}\label{lem:change-measure-err-pi}
    Suppose \cref{asm:full-support,asm:concen-coefficient} hold. If $\pi$ is $\pitheta{k+1}$ or $\pi^\star$, then for any pair of policies $\pi^\prime$ and $\pi^{\prime\prime}$, we have 
    \begin{align*}
        \abs{\EE_{s\sim\nu_\rho^\pi}\left[\bigl(D^{\pi^\prime}_{\pi^\star_{k+1}}(s)-D^{\pi^{\prime\prime}}_{\pi^\star_{k+1}}(s)\bigr)-\bigl(D^{\pi^\prime}_{\pitheta{k+1}}(s)-D^{\pi^{\prime\prime}}_{\pitheta{k+1}}(s)\bigr)\right]}\leq 2\sqrt{C_\nu\losspi{k}(\theta_{k+1};\theta_k,w_k)},
    \end{align*}
    where $D^{\pi}_{\pi^\prime}(s)\coloneqq\inner{\log\pi(s,\cdot)-\log\pi^\prime(s,\cdot)}{\pi(\cdot|s)}$ is the KL divergence between $\pi$ and $\pi^\prime$ at $s$.
\end{lemma}
\begin{proof}
    By \cref{lem:opt-pi} and the definition of the KL divergence, we have 
    \begin{align*}
        &\bigl(D^{\pi^\prime}_{\pi^\star_{k+1}}(s)-D^{\pi^{\prime\prime}}_{\pi^\star_{k+1}}(s)\bigr)-\bigl(D^{\pi^\prime}_{\pitheta{k+1}}(s)-D^{\pi^{\prime\prime}}_{\pitheta{k+1}}(s)\bigr)\\
        =&\inner{\log\pitheta{k+1}(s,\cdot)-\log\pi^\star_{k+1}(s,\cdot)}{\pi^\prime(\cdot|s)-\pi^{\prime\prime}(\cdot|s)}\\
        =&\inner{\lambda_{k+1}^{-1}f_{\theta_{k+1}}(s,\cdot)-\lambda_{k}^{-1}f_{\theta_{k}}(s,\cdot) + \eta_k Q_{w_k}(s,\cdot)}{\pi^\prime(\cdot|s)-\pi^{\prime\prime}(\cdot|s)}.
    \end{align*}
    Similar to the proof of \cref{lem:change-measure-err-Q}, by H\"{o}lder's inequality, we have 
    \begin{align*}
        &\abs{\EE_{s\sim\nu_\rho^\pi}\left[\bigl(D^{\pi^\prime}_{\pi^\star_{k+1}}(s)-D^{\pi^{\prime\prime}}_{\pi^\star_{k+1}}(s)\bigr)-\bigl(D^{\pi^\prime}_{\pitheta{k+1}}(s)-D^{\pi^{\prime\prime}}_{\pitheta{k+1}}(s)\bigr)\right]}\\
        \leq &\EE_{s\sim\nu_\rho^\pi}\abs{\inner{\lambda_{k+1}^{-1}f_{\theta_{k+1}}(s,\cdot)-\lambda_{k}^{-1}f_{\theta_{k}}(s,\cdot) + \eta_k Q_{w_k}(s,\cdot)}{\pi^\prime(\cdot|s)-\pi^{\prime\prime}(\cdot|s)}}\\
        \leq &\EE_{s\sim\nu_\rho^\pi}\norm{\lambda_{k+1}^{-1}f_{\theta_{k+1}}(s,\cdot)-\lambda_{k}^{-1}f_{\theta_{k}}(s,\cdot) + \eta_k Q_{w_k}(s,\cdot)}_\infty\norm{\pi^\prime(\cdot|s)-\pi^{\prime\prime}(\cdot|s)}_1\\
        \leq &2\EE_{s\sim\nu_\rho^\pi}\norm{\lambda_{k+1}^{-1}f_{\theta_{k+1}}(s,\cdot)-\lambda_{k}^{-1}f_{\theta_{k}}(s,\cdot) + \eta_k Q_{w_k}(s,\cdot)}_\infty.
    \end{align*}
    By \cref{asm:full-support,asm:concen-coefficient} and the Cauchy--Schwarz inequality we have 
    \begin{align*}
        &\EE_{s\sim \nu_\rho^{\pi}}\norm{\lambda_{k+1}^{-1}f_{\theta_{k+1}}(s,\cdot)-\lambda_{k}^{-1}f_{\theta_{k}}(s,\cdot) + \eta_k Q_{w_k}(s,\cdot)}_\infty\\
        \leq& \sqrt{\EE_{s\sim \nu_\rho^{\pi}}\left[\dv{\nu_\rho^{\pi}(s)}{\nu_\rho^{\pitheta{k}}(s)}\right]\EE_{s\sim \nu_\rho^{\pi}}\left[\dv{\nu_\rho^{\pitheta{k}}(s)}{\nu_\rho^{\pi}(s)}\norm{\lambda_{k+1}^{-1}f_{\theta_{k+1}}(s,\cdot)-\lambda_{k}^{-1}f_{\theta_{k}}(s,\cdot) + \eta_k Q_{w_k}(s,\cdot)}_\infty^2\right]}\\
        =& \sqrt{\left(\chisqr{\nu_\rho^{\pi}}{\nu_\rho^{\pitheta{k}}}+1\right)\bigl(\EE_{s\sim \nu_\rho^{\pitheta{k}}}\norm{\lambda_{k+1}^{-1}f_{\theta_{k+1}}(s,\cdot)-\lambda_{k}^{-1}f_{\theta_{k}}(s,\cdot) + \eta_k Q_{w_k}(s,\cdot)}_\infty^2\bigr)}\\
        \leq& \sqrt{C_\nu \losspi{k}(\theta_{k+1};\theta_k,w_k)},
    \end{align*}
    and the result follows immediately.
\end{proof}

Using the above auxiliary lemmas, we can now prove \cref{lem:onestep-NPMD}

\subsection{Proof of Lemma \ref{lem:onestep-NPMD}}\label{proof:onestep-NPMD}
\begin{proof}
    By \cref{lem:performance-difference} we have the following relations: 
    \begin{align}
        &V^\pitheta{k}(\rho)-V^\star(\rho)=\frac{1}{1-\gamma}\EE_{s\sim\nu_\rho^{\pi^\star}}\left[\inner{Q^\pitheta{k}(s,\cdot)}{\pitheta{k}(\cdot|s)-\pi^\star(\cdot|s)}\right],\label{eq:perf-diff-1}\\
        &V^\pitheta{k+1}(\rho)-V^\pitheta{k}(\rho)=\frac{1}{1-\gamma}\EE_{s\sim\nu_\rho^{\pitheta{k+1}}}\left[\inner{Q^\pitheta{k}(s,\cdot)}{\pitheta{k+1}(\cdot|s)-\pitheta{k}(\cdot|s)}\right].\label{eq:perf-diff-2}
    \end{align}
    Applying \cref{lem:change-measure-err-Q} to \eqref{eq:perf-diff-1} and \eqref{eq:perf-diff-2} gives 
    \begin{align}
        \begin{split}
            V^\pitheta{k}(\rho)-V^\star(\rho)
            \leq\frac{1}{1-\gamma}\EE_{s\sim\nu_\rho^{\pi^\star}}\left[\inner{Q_{w_k}(s,\cdot)}{\pitheta{k}(\cdot|s)-\pi^\star(\cdot|s)}\right] \\
            \quad + \frac{2}{1-\gamma}\sqrt{C_\nu \lossQ(w_k;\pitheta{k})},\label{eq:perf-bound-1}\\
        \end{split}\\
        \begin{split}
            V^\pitheta{k+1}(\rho)-V^\pitheta{k}(\rho)
            \leq\frac{1}{1-\gamma}\EE_{s\sim\nu_\rho^{\pitheta{k+1}}}\left[\inner{Q_{w_k}(s,\cdot)}{\pitheta{k+1}(\cdot|s)-\pitheta{k}(\cdot|s)}\right] \\ 
            \quad + \frac{2}{1-\gamma}\sqrt{C_\nu \lossQ(w_k;\pitheta{k})}.\label{eq:perf-bound-2}
        \end{split}
    \end{align}

    Applying \cref{lem:diff-pistar,lem:change-measure-err-pi} to \eqref{eq:perf-bound-1}, we have the following: 
    \begin{align}
        &\EE_{s\sim\nu_\rho^{\pi^\star}}\left[\inner{Q_{w_k}(s,\cdot)}{\pitheta{k}(\cdot|s)-\pi^\star(\cdot|s)}\right]\nonumber\\
        =&-\frac{1}{\eta_k}\EE_{s\sim\nu_\rho^{\pi^\star}}\left[\inner{\nabla h^{\pi^\star_{k+1}}(s,\cdot)-\nabla h^{\pitheta{k}}(s,\cdot)}{\pitheta{k}(\cdot|s)-\pi^\star(\cdot|s)}\right]\nonumber\\
        =&-\frac{1}{\eta_k}\EE_{s\sim\nu_\rho^{\pi^\star}}\left[D^{\pi^\star}_{\pi^\star_{k+1}}(s) - D^{\pi^\star}_{\pitheta{k}}(s) - D^{\pitheta{k}}_{\pi^\star_{k+1}}(s)\right]\nonumber\\
        \leq &-\frac{1}{\eta_k}\EE_{s\sim\nu_\rho^{\pi^\star}}\left[D^{\pi^\star}_{\pitheta{k+1}}(s) - D^{\pi^\star}_{\pitheta{k}}(s) - D^{\pitheta{k}}_{\pitheta{k+1}}(s)\right] + \frac{2}{\eta_k}\sqrt{C_\nu \losspi{k}(\theta_{k+1};\theta_k,w_k)},\label{eq:perf-bound-3}
    \end{align}
    where the second equality uses the three-point identity of KL-divergence. 
    Similarly, applying \cref{lem:diff-pistar,lem:change-measure-err-pi} to \eqref{eq:perf-bound-2}, we get 
    \begin{align*}
        &\EE_{s\sim\nu_\rho^{\pitheta{k+1}}}\left[\inner{Q_{w_k}(s,\cdot)}{\pitheta{k+1}(\cdot|s)-\pitheta{k}(\cdot|s)}\right]\\
        \leq&\frac{1}{\eta_k}\EE_{s\sim\nu_\rho^{\pitheta{k+1}}}\left[ - D^{\pitheta{k+1}}_{\pitheta{k}}(s) - D^{\pitheta{k}}_{\pitheta{k+1}}(s)\right] + \frac{2}{\eta_k}\sqrt{C_\nu \losspi{k}(\theta_{k+1};\theta_k,w_k)}.
    \end{align*}
    Note that $- D^{\pitheta{k+1}}_{\pitheta{k}}(s) - D^{\pitheta{k}}_{\pitheta{k+1}}(s)\leq 0$. Hence by \eqref{eq:visitation-bound} and \eqref{def:kappa} we get 
    \begin{align}
        &\EE_{s\sim\nu_\rho^{\pitheta{k+1}}}\left[\inner{Q_{w_k}(s,\cdot)}{\pitheta{k+1}(\cdot|s)-\pitheta{k}(\cdot|s)}\right]\nonumber\\
        \leq& - \frac{1}{\eta_k}\EE_{s\sim\rho}\left[\dv{\nu_\rho^{\pitheta{k+1}}(s)}{\rho(s)} \left( D^{\pitheta{k+1}}_{\pitheta{k}}(s) + D^{\pitheta{k}}_{\pitheta{k+1}}(s)\right)\right] + \frac{2}{\eta_k}\sqrt{C_\nu \losspi{k}(\theta_{k+1};\theta_k,w_k)}\nonumber\\
        \leq& - \frac{1-\gamma}{\eta_k}\EE_{s\sim\nu_\rho^{\pi^\star}}\left[\dv{\rho(s)}{\nu_\rho^{\pi^\star}(s)} \left( D^{\pitheta{k+1}}_{\pitheta{k}}(s) + D^{\pitheta{k}}_{\pitheta{k+1}}(s)\right)\right] + \frac{2}{\eta_k}\sqrt{C_\nu \losspi{k}(\theta_{k+1};\theta_k,w_k)}\nonumber\\
        \leq& - \frac{1-\gamma}{\kappa\eta_k}\EE_{s\sim\nu_\rho^{\pi^\star}}\left[D^{\pitheta{k+1}}_{\pitheta{k}}(s) + D^{\pitheta{k}}_{\pitheta{k+1}}(s)\right] + \frac{2}{\eta_k}\sqrt{C_\nu \losspi{k}(\theta_{k+1};\theta_k,w_k)}.\label{eq:perf-bound-4}
    \end{align}
    
    We first multiply \eqref{eq:perf-bound-1} by $1-\gamma_\rho=\frac{1-\gamma}{\kappa}$ and sum it up with \eqref{eq:perf-bound-2}, and then apply \eqref{eq:perf-bound-3} and \eqref{eq:perf-bound-4}. Rearranging terms gives us 
    \begin{align*}
        &V^\pitheta{k+1}(\rho)-V^\star(\rho) + \frac{1}{\kappa\eta_k}\EE_{s\sim\nu_\rho^{\pi^\star}}\left[D^{\pi^\star}_{\pitheta{k+1}}(s)\right] \\
        \leq& \gamma_\rho \left(V^\pitheta{k}(\rho)-V^\star(\rho)\right) + \frac{1}{\kappa\eta_k}\EE_{s\sim\nu_\rho^{\pi^\star}}\left[ D^{\pi^\star}_{\pitheta{k}}(s) - D^{\pitheta{k+1}}_{\pitheta{k}}(s)\right] \\
        &\quad + \frac{2(2-\gamma_\rho)}{(1-\gamma_\rho)\kappa}\sqrt{C_\nu \lossQ(w_k;\pitheta{k})} + \frac{2(2-\gamma_\rho)}{(1-\gamma_\rho)\kappa\eta_k}\sqrt{C_\nu \losspi{k}(\theta_{k+1};\theta_k,w_k)} \\
        \leq& \gamma_\rho \left(V^\pitheta{k}(\rho)-V^\star(\rho) + \frac{1}{\kappa\gamma_\rho\eta_k}\EE_{s\sim\nu_\rho^{\pi^\star}}\left[ D^{\pi^\star}_{\pitheta{k}}(s)\right]\right) \\
        &\quad + \frac{4\sqrt{C_\nu}}{(1-\gamma_\rho)\kappa}\left(\sqrt{\lossQ(w_k;\pitheta{k})} + \frac{1}{\eta_k}\sqrt{\losspi{k}(\theta_{k+1};\theta_k,w_k)}\right),
    \end{align*}
    where the last inequality is obtained by dropping some non-positive terms. 
\end{proof}

\subsection{Proof of Theorem \ref{thm:iteration-complexity}}\label{proof:iteration-complexity}
\begin{proof}
    For $k=0$, we have $V^{\pitheta{0}}(\rho)-V^\star(\rho)\leq\frac{C}{1-\gamma}$ from \eqref{eq:V-bound}. Since $\theta_0=0$, we have $f_{\theta_0}(s,a)$ be constant for any $s$ and $a$, thus $D_{\pitheta{0}}^{\pi^\star}(s)=\log|\cA|$. 
    
    By \cref{lem:onestep-NPMD} and concavity of the square root function, taking the expectation with respect to the samples $\Xi_k^Q$ and $\Xi_k^\Pi$ conditioned on previous samples yields 
    \begin{align*}
        &\EE \bigl[V^{\pitheta{k+1}}(\rho) - V^\star(\rho)\bigr]\\
        \leq &\EE \Bigl(V^{\pitheta{k+1}}(\rho) - V^\star(\rho) + \smallfrac{1}{\kappa\gamma_\rho\eta_{k+1}}\EE_{s\sim \nu_\rho^{\pi^\star}} \bigl[D_{\pitheta{k+1}}^{\pi^\star}(s)\bigr]\Bigr) \\
        = &\EE \Bigl(V^{\pitheta{k+1}}(\rho) - V^\star(\rho) + \smallfrac{1}{\kappa\eta_k}\EE_{s\sim \nu_\rho^{\pi^\star}} \bigl[D_{\pitheta{k+1}}^{\pi^\star}(s)\bigr]\Bigr) \\
        \leq &\gamma_\rho\Bigl(V^{\pitheta{k}}(\rho) - V^\star(\rho) + \smallfrac{1}{\kappa\gamma_\rho\eta_k}\EE_{s\sim \nu_\rho^{\pi^\star}} \bigl[D_{\pitheta{k}}^{\pi^\star}(s)\bigr]\Bigr) \\ 
        &\quad +\smallfrac{4\sqrt{C_\nu}}{\kappa(1-\gamma_\rho)}\Bigl(\sqrt{\EE\bigl[\lossQ(w_k;\pitheta{k})\bigr]} + \smallfrac{1}{\eta_k}\sqrt{\EE\bigl[\losspi{k}(\theta_{k+1};\theta_k,w_k)\bigr]}\Bigr) \\
        \leq &\gamma_\rho\Bigl(V^{\pitheta{k}}(\rho) - V^\star(\rho) + \smallfrac{1}{\kappa\gamma_\rho\eta_k}\EE_{s\sim \nu_\rho^{\pi^\star}} \bigl[D_{\pitheta{k}}^{\pi^\star}(s)\bigr]\Bigr)+ \smallfrac{8\sqrt{C_\nu}C}{\kappa(1-\gamma_\rho)}\gamma_\rho^{k+1}.
    \end{align*}
    Dividing both sides by $\gamma_\rho^{k+1}$, telescoping from $0$ to $k$ and rearranging terms, we obtain 
    \begin{align*}
        &\EE \bigl[V^{\pitheta{k+1}}(\rho) - V^\star(\rho)\bigr]\\
        \leq &\frac{C}{1-\gamma}\cdot \gamma_\rho^{k+1}\Bigl(\bigl(1+\log|\cA|\bigr) + 8\sqrt{C_\nu}(k+1)\Bigr),
    \end{align*}
    where the expectation is taken with respect to all samples from zeroth to $k$-th iteration. 

    Let $C_1=1 + \log|\cA|$, $C_2=8\sqrt{C_\nu}$. For any $\epsilon>0$, it suffices to choose $k=\left\lceil a\left(\log\frac{b}{\epsilon}+\log\log\frac{b}{\epsilon}\right)\right\rceil$ with $a=\frac{1}{\log\frac{1}{\gamma_\rho}}$ and $b=\frac{2C(C_1+C_2)}{(1-\gamma)\log\frac{1}{\gamma_\rho}}$, since 
    \begin{align*}
        \EE \bigl[V^{\pitheta{k}}(\rho) - V^\star(\rho)\bigr]
        &\leq (C_1+C_2)k\gamma_\rho^{k} \cdot\frac{C}{1-\gamma}\\
        &\leq \frac{(C_1+C_2)a\left(\log\frac{b}{\epsilon}+\log\log\frac{b}{\epsilon}\right)}{\frac{b}{\epsilon}\log\frac{b}{\epsilon}}\cdot\frac{C}{1-\gamma}\\
        &\leq \frac{2(C_1+C_2)a\log\frac{b}{\epsilon}}{b\log\frac{b}{\epsilon}}\cdot\frac{C\epsilon}{1-\gamma} = \epsilon.
    \end{align*}
    In view of $\log x\geq 1-\frac{1}{x}$, we obtain 
    \begin{align*}
        k
        &\leq \log_\frac{1}{\gamma_\rho}\left(\frac{2C(C_1+C_2)}{(1-\gamma)\epsilon\log\frac{1}{\gamma_\rho}} \log(\frac{2C(C_1+C_2)}{(1-\gamma)\epsilon\log\frac{1}{\gamma_\rho}})\right)+1\\ 
        &\leq \log_\frac{1}{\gamma_\rho}\left(\frac{2C(C_1+C_2)}{(1-\gamma_\rho)(1-\gamma)\epsilon} \log(\frac{2C(C_1+C_2)}{(1-\gamma_\rho)(1-\gamma)\epsilon})\right)+1\\ 
        &=\Tilde{O}\left(\log_\frac{1}{\gamma_\rho}\left(\frac{C\left(\sqrt{C_\nu}+\log|\cA|\right)}{\kappa(1-\gamma_\rho)^2\epsilon}\right)\right), 
    \end{align*}
    where $\Tilde{O}(\cdot)$ hides the logarithm terms.
\end{proof}

\section{Proofs for Sample Complexity}
In this section, we provide missing proofs for \cref{thm:sample-complexity-Q,thm:sample-complexity-pi}. 
The proofs are based on the statistical recovery result (\cref{lem:estimation-error-approxlip}) for CNN in \cref{sec:appendix-nonparametric}. 

\subsection{Proof of Theorem {\ref{thm:sample-complexity-Q}}}\label{proof:sample-complexity-Q}
\begin{proof}
    From \eqref{eq:V-bound} and \cref{lem:lip-Q} we know $Q^{\pitheta{k}}(\cdot,a)$ is $\frac{C}{1-\gamma}$-bounded and $(L_Q, \alpha)$-Lipschitz on $\cS$ for all $a\in\cA$. From \eqref{eq:loss-Q-translation} in \cref{sec:appendix-translation-Q} we have  
    \begin{align*}
        \EE\sbr{\lossQ(w_k;\pitheta{k})}
        &\leq|\cA|\max_{a\in\cA}\EE\sbr{\EE_{s\sim \nu_\rho^{\pitheta{k}}}\abs{Q_{w_k}(s,a)-Q^{\pitheta{k}}(s,a)}^2},
    \end{align*}
    Thus to ensure $\EE\bigl[\lossQ(w_k;\pitheta{k})\bigr]\leq C^2\gamma_\rho^{2(k+1)}$, it suffices to use a sample size $N$ such that for all $a\in\cA$, 
    \begin{equation}\label{eq:estimation-errorbound-Q}
        \EE_{\Xi_k}\EE_{s\sim \nu_\rho^{\pitheta{k}}}\abs{Q_{w_k}(s,a)-Q^{\pitheta{k}}(s,a)}^2\leq\frac{C^2\gamma_\rho^{2(k+1)}}{|\cA|}.
    \end{equation}
    
    We set the following for $\cW_k=\cW_{\mathrm{Lip}}(\frac{C}{1-\gamma}, L_Q, \alpha, \epsilon_Q)$: 
    \begin{align*}
        &M=O(N^\frac{d}{d+2\alpha}), ~ L=O(\log N+D+\log D), ~ J=O(D), ~ I\in[2,D], ~ R_1=O(1), \\
        &\log R_2=O(\log^2 N + D\log N), ~ \epsilon_Q=(L_Q^2+\smallfrac{C^2}{(1-\gamma)^2})D^{\frac{3\alpha}{2\alpha + d}}N^{-\frac{\alpha}{2\alpha+d}},
    \end{align*}
    where $O(\cdot)$ hides some constant depending on $\log L_Q$, $\log\frac{C}{1-\gamma}$, $d$, $\alpha$, $\omega$, $B$, and the surface area $\Area(\cS)$. Then by \cref{lem:noise} in \cref{sec:appendix-translation-Q} and \cref{lem:estimation-error-approxlip} in \cref{sec:appendix-nonparametric}, the following bound on the regression error holds: 
    \begin{align}
        &\EE_{\Xi_k}\EE_{s\sim \nu_\rho^{\pitheta{k}}}\abs{Q_{w_k}(s,a)-Q^{\pitheta{k}}(s,a)}^2\nonumber\\
        \leq &C^\prime\left((\overbar{L}_Q + 1)^2\frac{C^2}{(1-\gamma)^2} + \sigma^2\right) N^{-\frac{2\alpha}{2\alpha+d}}\log^6 N,\label{eq:estimation-error-Q}
    \end{align}
    where $\overbar{L}_Q$ is defined as in \eqref{def:normalized-lip-Q}, $\sigma^2=\frac{\gamma^2 C^2}{4(1-\gamma)^2}\leq\frac{C^2}{4(1-\gamma)^2}$ is the variance proxy derived in \cref{lem:noise} and $C^\prime$ is a constant depending on $D^{\frac{6\alpha}{2\alpha+d}}$, $\log L_Q$, $\log\frac{C}{1-\gamma}$, $d$, $\alpha$, $\omega$, $B$, and the surface area $\Area(\cS)$.

    By choosing $N=\left(\frac{1}{\delta}\log^{3}\frac{1}{\delta}\right)^{\frac{d}{\alpha}+2}$ where $\delta=\sqrt{\frac{1}{4096 ((\overbar{L}_Q + 1)^2 + \frac{1}{4})C^\prime|\cA|(\frac{d}{\alpha}+2)^6}}(1-\gamma)\gamma_\rho^{k+1}$, we have 
    \begin{align*}
        N^{-\frac{2\alpha}{2\alpha+d}}\log^6 N
        &=\frac{\left((\frac{d}{\alpha}+2)\log\left(\frac{1}{\delta}\log^{3}\frac{1}{\delta}\right)\right)^6}{\frac{1}{\delta^2}\log^6\frac{1}{\delta}}\\
        &\leq\frac{(\frac{d}{\alpha}+2)^6\left(4\log\frac{1}{\delta}\right)^6}{\frac{1}{\delta^2}\log^6\frac{1}{\delta}}\\
        &=\frac{(1-\gamma)^2\gamma_\rho^{2(k+1)}}{((\overbar{L}_Q + 1)^2 + \frac{1}{4})C^\prime|\cA|}.
    \end{align*}
    Plugging the result into \eqref{eq:estimation-error-Q}, we obtain \eqref{eq:estimation-errorbound-Q}.
    
    Denoting $C_3=\sqrt{4096 ((\overbar{L}_Q + 1)^2 + \frac{1}{4})C^\prime|\cA|(\frac{d}{\alpha}+2)^6}=O(\sqrt{|\cA|})$, we have the sample size $N=\Tilde{O}\left(\frac{C_3}{1-\gamma}\gamma_\rho^{-(k+1)}\right)^{\frac{d}{\alpha}+2}=\Tilde{O}\left(\frac{\sqrt{|\cA|}}{1-\gamma}\gamma_\rho^{-(k+1)}\right)^{\frac{d}{\alpha}+2}$.
    When \cref{asm:concen-coefficient,asm:full-support} hold, we know from \cref{thm:iteration-complexity} that the total iteration number $K$ satisfies
    \begin{align*}
        \gamma_\rho^{-K}=\Tilde{O}\left(\frac{C\left(\sqrt{C_\nu}+\log|\cA|\right)}{\kappa(1-\gamma_\rho)^2\epsilon}\right).
    \end{align*}
    Plugging $K$ into our choice of $N$ yields the result.
\end{proof}

\subsection{Proof of Theorem {\ref{thm:sample-complexity-pi}}}\label{proof:sample-complexity-pi}
\begin{proof}
    By \eqref{eq:loss-pi-translation}, it suffices to specify the architecture $\cF$ and restrictions so that for all $a\in\cA$ and all $k\leq K$, 
    \begin{align}
        &\EE_{\Xi_{k}}\EE_{s\sim\nu_\rho^{\pitheta{k}}}\abs{f_{\theta_{k+1}}(s,a)-\gamma_\rho f_{\theta_k}(s,a)+Q_{w_k}(s,a)}^2
        \leq \frac{C^2\gamma_\rho^{2(k+1)}}{|\cA|}\label{eq:estimation-errorbound-pi-1}.
    \end{align}
    By \cref{lem:actor-lip}, our choice of $\eta_k$, $\lambda_k$, $\cW_k$ and $\Theta_k$ ensures that $\lambda_{k+1}g^\star_{k+1}(\cdot,a)=\gamma_\rho f_{\theta_k}(\cdot,a)-Q_{w_k}(\cdot,a)$ is $(\frac{L_Q}{1-\gamma_\rho }, \alpha, \frac{\epsilon_Q}{1-\gamma_\rho })$-approximately Lipschitz and is uniformly bounded by $\frac{C}{(1-\gamma_\rho )(1-\gamma)}$. 
    
    We set the following for the underlying CNN architecture $\cF$ and the restricted parameter spaces $\cW_k$ and $\Theta_k$: 
    \begin{align*}
        &M=O(N^\frac{d}{d+2\alpha}), ~ L=O(\log N+D+\log D), ~ J=O(D), ~ I\in[2,D], ~ R_1=O(1), \\
        &\log R_2=O(\log^2 N + D\log N), ~ \epsilon_Q=(L_Q^2+\smallfrac{C^2}{(1-\gamma)^2})D^{\frac{3\alpha}{2\alpha + d}}N^{-\frac{\alpha}{2\alpha+d}},
    \end{align*}
    where $O(\cdot)$ hides some constant depending on $\log L_Q$, $\log\frac{C}{1-\gamma}$, $d$, $\alpha$, $\omega$, $B$, and the surface area $\Area(\cS)$. 
    Then by \cref{lem:noise} in \cref{sec:appendix-translation-Q} and \cref{lem:estimation-error-approxlip} in \cref{sec:appendix-nonparametric}, the following bound on the regression error holds: 
    \begin{align}
        &\EE_{\Xi_k}\EE_{s\sim \nu_\rho^{\pitheta{k}}}\abs{f_{\theta_{k+1}}(s,a)-\gamma_\rho f_{\theta_k}(s,a)+Q_{w_k}(s,a)}^2\nonumber\\
        \leq &C^\prime\frac{C^2(\overbar{L}_Q + 1)^2}{(1-\gamma_\rho)^2(1-\gamma)^2} N^{-\frac{2\alpha}{2\alpha+d}}\log^6 N,\label{eq:estimation-error-pi}
    \end{align}
    where $\overbar{L}_Q$ is defined as in \eqref{def:normalized-lip-Q} and $C^\prime$ is a constant depending on $D^{\frac{6\alpha}{2\alpha+d}}$, $\log L_Q$, $\log\frac{C}{1-\gamma}$, $d$, $\alpha$, $\omega$, $B$, and the surface area $\Area(\cS)$.
    
    We let $N=\left(\frac{1}{\delta}\log^{3}\frac{1}{\delta}\right)^{\frac{d}{\alpha}+2}$ with $\delta=\sqrt{\frac{1}{C^{\prime\prime}|\cA|(\frac{d}{\alpha}+2)^6}}(1-\gamma)(1-\gamma_\rho )\gamma_\rho^{K+1}$, where $C^{\prime\prime}$ is some constant depending on $C^\prime$ and $\overbar{L}_Q$ so that the right-hand side of \eqref{eq:estimation-error-pi} becomes $\frac{C^2\gamma_\rho^{2(K+1)}}{|\cA|}$, then we have \eqref{eq:estimation-errorbound-pi-1} satisfied.
    Denoting $C_4=\sqrt{C^{\prime\prime}|\cA|(\frac{d}{\alpha}+2)^6}=O(\sqrt{|\cA|})$, we can write $N=\Tilde{O}\left(\frac{C_4}{(1-\gamma_\rho )(1-\gamma)}\gamma_\rho^{-(K+1)}\right)^{\frac{d}{\alpha}+2}$. 
    
    For $\epsilon>0$, the total iteration number $K$ satisfies
    \begin{align*}
        \gamma_\rho^{-K}=\Tilde{O}\left(\frac{C\left(\sqrt{C_\nu}+\log|\cA|\right)}{\kappa(1-\gamma_\rho)^2\epsilon}\right).
    \end{align*}
    Plugging $K$ into our choice of $N$ yields the result.
\end{proof}

\section{Approximation Theory for CNN}\label{sec:appendix-approx}
In this section, we introduce the approximation theory for CNN. We first consider the case when target function $f_0=Q^\pi(\cdot,a)$ is $(L_Q,\alpha)$-Lipschitz for any $a\in\cA$ (\cref{thm:approx-error-Q}), and then proceed to the case when $f_0$ is a general $(L_f,\alpha,\epsilon_f)$-approximately Lipschitz function (\cref{thm:approx-error-approxlip}). \cref{cor:approx-error-pi} follows immediately from \cref{thm:approx-error-approxlip}.

Let us define a class of single-block CNNs in the form of 
\begin{equation}\label{eq:SCNN}
    f(x)=W\cdot \Conv_{\cW,\cB}(P(x))
\end{equation}
as 
\begin{align}\label{eq:nn_class_single_no_mag}
    \cF^{\mathrm{SCNN}}(L,J,I,R_1,R_2)=& \bigl\{f \mid f(x) \text{ in the form \eqref{eq:SCNN} with $L$ layers; the filter size } \nonumber\\
    &\quad\quad \text{ is bounded by $I$; the number of channels is bounded by $J$; }\nonumber\\ 
    &\quad\quad \max_{l}\|\cW^{(l)}\|_{\infty} \vee \|\cB^{(l)}\|_{\infty} \leq R_1, ~ \norm{W}_\infty\leq R_2 \bigr\}.
\end{align}
We will use this class of single-block CNNs as the building blocks of our final CNN approximation for the ground truth Lipschitz function.

\subsection{Proof of Theorem {\ref{thm:approx-error-Q}} Overview}\label{sec:appendix-approx-overview}
\cref{thm:approx-error-Q} establishes the relation between network architecture and approximation error for $f_0=Q^\pi$. We prove \cref{thm:approx-error-Q} for $(L_f,\alpha)$-Lipschitz $f_0$ in the following steps:

\paragraph{Step 1: Decompose $f_0$ as a sum of locally supported functions over the manifold.}

Since manifold $\cS$ is assumed compact (\cref{asm:manifold}), we can cover it with a finite set of $D$-dimensional open Euclidean balls $\{B_\beta(\cbb_i)\}_{i=1}^{C_\cS}$, where $\cbb_i$ denotes the center of the $i$-th ball and $\beta$ is the radius. We choose $\beta<\frac{\omega}{4}$ and define $U_i=B_\beta(\cbb_i)\cap\cS$. Note that each $U_i$ is diffeomorphic to an open subset of $\RR^d$ \citep[Lemma 5.4]{niyogi2008finding}. Moreover, the set $\{U_i\}_{i=1}^{C_\cS}$ forms an open cover for $\cS$. There exists a carefully designed open cover with cardinality $C_\cS\leq\left\lceil\frac{\Area(\cS)}{\beta^d}T_d\right\rceil$, where $\Area(\cS)$ denotes the surface area of $\cS$ and $T_d$ denotes the thickness of $U_i$'s, that is, the average number of $U_i$'s that contain a given point on $\cS$. It has been shown that $T_d=O(d\log d)$ \citep{conway1988sphere}.

Moreover, for each $U_i$, we can define a linear transformation 
\begin{equation}\label{eq:linear-map}
    \phi_i(x)=a_i V_i^\top (x-\cbb_i) + b_i,
\end{equation}
where $a_i\in (0,1]$ is a scaling factor and $b_i\in\RR^d$ is the translation vector, both of which are chosen to ensure $\phi(U_i)\subset [0,1]^d$, and the columns of $V_i\in\RR^{D\times d}$ form an orthonormal basis for the tangent space $T_{\cbb_i}(\cS)$ at $\cbb_i$. Overall, the atlas $\{(\phi_i, U_i)\}_{i=1}^{C_\cS}$ transforms each local neighborhood on the manifold to a $d$-dimensional cube.

Thus, we can decompose $f_0$ using this atlas as 
\begin{equation}
    f_0 = \sum_{i=1}^{C_\cS} f_i, \quad \text{ with } \quad f_i=f_0\times \rho_i,
\end{equation}
because there exists such a $C^\infty$ partition of unity $\{\rho_i\}_{i=1}^{C_\cS}$ with $\supp(\rho_i)\subset U_i$ \citep[Proposition 1]{liu2021besov}. Since each $f_i$ is only supported on a subset of $U_i$, we can further write 
\begin{equation}\label{eq:f0-decompose}
    f_0 = \sum_{i=1}^{C_\cS} (f_i\circ \phi_i^{-1})\circ \phi_i\times \ind_{U_i},
\end{equation}
where $\ind_{U_i}$ is the indicator function of $U_i$.

Lastly, we extend $f_i\circ \phi_i^{-1}$ to the entire cube $[0,1]^d$ with $0$:
\begin{equation}\label{eq:def-component}
    \overbar{f}_i(x)=\begin{cases}
        f_i\circ\phi_i^{-1}(x), & x\in\supp(f_i\circ\phi_i^{-1}),\\
        0, & x\in [0,1]^d\setminus\supp(f_i\circ\phi_i^{-1}).
    \end{cases}
\end{equation}
By \cref{lem:lip-component} in \cref{sec:appendix-approx-lemmas}, $\overbar{f}_i$ is a Lipschitz function with Lipschitz constant at most $L_i\coloneqq C_i(L_f+\norm{f_0}_\infty)$, where $C_i$ is a constant depending on $\alpha,\omega, \phi_i$ and $\rho_i$. This extended function will be approximated with first-order B-splines in the next step. 

\paragraph{Step 2: Approximate each local function with first-order B-splines.}
Since each local function $\overbar{f}_i$ is Lipschitz on $d$-dimensional unit cube, a weighted sum of first-order B-splines can approximate it. The number of splines depends exponentially on the intrinsic dimension $d$, rather than the ambient dimension $D$.
To be more precise, we partition the unit cube into $N=2^{pd}$ small cubes with side lengths $2^{-p}$, where $p\in\NN$ is positive. We denote $J(p)=\{0, 1, \dots, 2^p-1\}^d$ as a vector index set. The first-order B-spline $M_{p,j}$ with shift vector $j\in J(p)$ is defined as 
\begin{equation}
    M_{p,j}(x)=\prod_{k=1}^d\psi(2^p x_k-j_k),
\end{equation}
where $\psi\colon [0,2]\to\RR$ is a sawtooth function:
\begin{align*}
    \psi(x)=\begin{cases}
        x, & 0\leq x\leq 1,\\
        2-x, & 1< x\leq 2,\\
        0, & \text{otherwise.}
    \end{cases}
\end{align*}
Each B-spline $M_{p,j}$ is supported on the small cube $B_j=\{x\in\RR^d \mid j_k\leq x_k\leq j_k+2, ~\forall k\in[d]\}$. Then by \cref{lem:spline-local}, there exists a function $\tilde{f}_i$ in the form 
\begin{align*}
    \tilde{f}_i=\sum_{j\in J(p)}c_{i,j}M_{p,j},
\end{align*}
such that 
\begin{equation}\label{eq:approx-error-spline}
    \norm{\tilde{f}_i - \overbar{f}_i}_\infty\leq 2L_i dN^{-\alpha/d}.
\end{equation}

By \eqref{eq:f0-decompose} and \eqref{eq:approx-error-spline}, we now have a sum of first-order B-splines 
\begin{equation}\label{eq:f0-decompose-spline}
    \tilde{f}
    \coloneqq \sum_{i=1}^{C_\cS}\tilde{f}_i\circ\phi_i\times\ind_{U_i}
    = \sum_{i=1}^{C_\cS}\sum_{j\in J(p)}c_{i,j}M_{p,j}\circ\phi_i\times\ind_{U_i},
\end{equation}
which can approximate the target Lipschitz function $f_0$ with error 
\begin{equation}
    \norm{\tilde{f} - f_0}_\infty\leq 2C_\cS d \max_{i=1,\dots,C_\cS}C_i(L_f+\norm{f_0}_\infty) N^{-\alpha/d}.
\end{equation}

\paragraph{Step 3: Approximate each first-order B-spline with a composition of CNNs.}

We now turn to approximate $\tilde{f}$ defined in \eqref{eq:f0-decompose-spline} with a composition of CNNs. We first approximate some building blocks with single-block CNNs defined as in \eqref{eq:nn_class_single_no_mag} and then ensemble them together. The building blocks include the multiplication operator $\times$, chart mappings $\{\phi_i\}_{i=1}^{C_\cS}$, indicator functions $\{\ind_{U_i}\}_{i=1}^{C_\cS}$, and first-order B-splines $\{M_{p,j}\}_{j\in J(p)}$.

The multiplication operator $\times$ can be approximated by a single-block CNN $\hat{\times}$ with at most $\eta$ error in the $L^\infty$ sense (\cref{prop:multiplication}), which needs $O(\log\frac{1}{\eta})$ layers and $6$ channels. All weight parameters are bounded by $(c_0^2\vee 1)$, where $c_0$ is the uniform upper bound of the input functions to be multiplied.

The chart mapping $\phi_i$, according to \eqref{eq:linear-map}, is a linear transformation. Thus, it can be expressed with a single-layer perceptron $\hat{\phi}_i$, which can be equivalently expressed by a CNN. 

The indicator function $\ind_{U_i}$ is equal to $1$ if $d_i^2(x)=\norm{x-\cbb_i}_2^2\leq\beta^2$ and equal to $0$ otherwise. By this definition, we can write $\ind_{U_i}$ as a composition of a univariate indicator $\ind_{[0,\beta^2]}$ and the distance function $d_i^2$:
\begin{equation}
    {\ind_{U_i}(x)} = {\ind_{[0,\beta^2]}}\circ {d_i^2(x)}.
\end{equation}
Given $\theta\in(0,1)$ and $\Delta\geq 8DB^2\theta$, it turns out that $\ind_{[0,\beta^2]}$ and $d_i^2$ can be approximated with two single-block CNNs $\hat{\ind}_\Delta$ and $\hat{d}_i^2$ respectively (\cref{prop:SCNN-chart}) such that 
\begin{equation}
    \norm{\hat{d}_i^2 - d_i^2}_\infty\leq 4B^2D\theta
\end{equation}
and 
\begin{equation*}
    {\hat{\ind}_\Delta}\circ{\hat{d}_i^2}(x)=\begin{cases}
        1, & \text{if $x \in U_i, d_i^2(x)\leq\beta^2-\Delta$,}\\
        0, & \text{if $x \notin U_i$,}\\
        \text{between $0$ and $1$}, & \text{otherwise.}
    \end{cases}
\end{equation*}
The architecture and size of $\hat{\ind}_\Delta$ and $\hat{d}_i^2$ are characterized in \cref{prop:SCNN-chart} as functions of $\theta$ and $\Delta$.

The first-order B-spline $M_{p,j}$ can be approximated by a single-block CNN $\hat{M}_{p,j}$ up to arbitrarily chosen $\epsilon_1$ error (\cref{prop:SCNN-spline}). We can find a proper $\epsilon_1$ a set of single-block CNNs $\{\hat{f}_{i,j}\}_{j\in J(p)}$ such that the error matches \eqref{eq:approx-error-spline}: 
\begin{equation}
    \norm{\sum_{j\in J(p)}\hat{f}^{\mathrm{SCNN}}_{i,j} - \tilde{f}_i}_\infty\leq 2L_idN^{-\alpha/d}.
\end{equation}
The architecture and size of $\hat{f}^{\mathrm{SCNN}}_{i,j}$ are characterized in \cref{prop:SCNN-local} as functions of $N$.

Putting the above results together, we can develop a composition of single-block CNNs, which can be further expressed by a single-block CNN (\cref{lem:SCNN-comp}):
\begin{equation}\label{eq:SCNN-local}
    \hat{g}^{\mathrm{SCNN}}_{i,j}=\hat{\times}\left(\hat{f}^{\mathrm{SCNN}}_{i,j}\circ\hat{\phi}_i, {\hat{\ind}_\Delta}\circ{\hat{d}_i^2}\right).
\end{equation}
Details are provided in \cref{proof:approx-error-Q}.

\paragraph{Step 4: Express the sum of CNN compositions with a CNN.}
Finally, we can assemble everything into $\hat{f}$ as 
\begin{equation}
    \hat{f}=\sum_{i=1}^{C_\cS}\sum_{j\in J(p)}\hat{g}^{\mathrm{SCNN}}_{i,j},
\end{equation}
which serves as an approximation of $f_0$. By choosing the appropriate network size in \cref{lem:approx-error-tradeoff}, we can ensure that 
\[
    \norm{\hat{f}-f_0}_\infty\leq c_0(L_f+\norm{f_0}_\infty)N^{-\alpha/d}
\]
for some constant $c_0$ depending on $d$, $\alpha$, $\omega$, $B$, and the surface area $\Area(\cS)$.

By \cref{lem:SCNN-sum-mod}, for $\tilde{M},\tilde{J}>0$, we can write this sum of $C_\cS\cdot\abs{J(p)}= C_\cS\cdot N$ single-block CNNs as a sum of $\tilde{M}$ single-block CNNs with the same architecture, whose channel number upper bound $J$ depends on $\tilde{J}$. This allows \cref{thm:approx-error-Q} to be more flexible with network architecture. By \cref{lem:SCNN-sum-CNN}, this sum of $\tilde{M}$ single-block CNNs can be further expressed as one CNN in the CNN class \eqref{eq:nn_class_no_mag}. Finally, $N$ (or equivalently, $p$) will be chosen appropriately as a function of network architecture parameters, and the approximation theory of CNN is proven by plugging in $f_0=Q^\pi$, $L_f=L_Q$ in \cref{lem:lip-Q}.

In the following, we provide the proof details for \cref{thm:approx-error-Q}. 

\subsection{Proof of Theorem {\ref{thm:approx-error-Q}}}\label{proof:approx-error-Q}
\begin{proof}
We start from the decomposition of the approximation error of $\hat{f}$, which is based on the decomposition of the approximation error of $\hat{g}^{\mathrm{SCNN}}_{i,j}$ in \eqref{eq:SCNN-local}.

\begin{lemma}\label{lem:approx-error-tradeoff}
    Let $\eta>0$ be the approximation error of the multiplication operator $\hat{\times}(\cdot,\cdot)$ as defined in Step 3 of \cref{sec:appendix-approx-overview} and \cref{prop:multiplication}, $\Delta$ and $\theta$ be defined as in Step 3 of \cref{sec:appendix-approx-overview} and \cref{prop:SCNN-chart}. Assume $N=2^{pd}$ is chosen according to \cref{prop:SCNN-local}. For any $i=1,\dots,C_\cS$, we have $\norm{\hat{f}-f_0}_\infty\leq\sum_{i=1}^{C_\cS}(A_{i,1}+A_{i,2}+A_{i,3})$ with 
    \begin{align*}
        &A_{i,1}=\sum_{j\in J(p)}\norm{\hat{\times}\left(\hat{f}^{\mathrm{SCNN}}_{i,j}\circ\hat{\phi}_i, {\hat{\ind}_\Delta}\circ{\hat{d}_i^2}\right) - \hat{f}^{\mathrm{SCNN}}_{i,j}\circ\hat{\phi}_i\times ({\hat{\ind}_\Delta}\circ{\hat{d}_i^2})}_\infty\leq N\eta,\\
        &A_{i,2}=\norm{\left(\sum_{j\in J(p)}\left(\hat{f}^{\mathrm{SCNN}}_{i,j}\circ\hat{\phi}_i\right)\right)\times ({\hat{\ind}_\Delta}\circ{\hat{d}_i^2}) - f_i\times ({\hat{\ind}_\Delta}\circ{\hat{d}_i^2})}_\infty\leq 4L_idN^{-\alpha/d},\\
        &A_{i,3}=\norm{f_i\times({\hat{\ind}_\Delta}\circ{\hat{d}_i^2}) - f_i\times{\ind_{U_i}}}_\infty\leq\frac{C^{\prime}(\pi+1)}{\beta(1-\beta/\omega)}\Delta
    \end{align*}
    for some constant $C^{\prime}$ depending on $\rho_i$ and $\phi_i$. Furthermore, for any $\epsilon\in(0,1)$, setting 
    \begin{equation}\label{eq:parameter-tradeoff}
        \eta=L_i d N^{-1-\alpha/d}, ~ \Delta=\frac{L_id\beta(1-\beta/\omega)N^{-\alpha/d}}{C^{\prime}(\pi+1)}, ~ \theta=\frac{\Delta}{16B^2D}
    \end{equation}
    yields 
    \begin{align*}
        \norm{\hat{f}-f_0}_\infty\leq C^{\prime\prime}(L_f+\norm{f_0}_\infty)N^{-\frac{\alpha}{d}},
    \end{align*}
    where $C^{\prime\prime}$ is a constant depending on $d$, $\alpha$, $\omega$, $\rho_i$ and $\phi_i$. 
    The choice in \eqref{eq:parameter-tradeoff} satisfies the condition $\Delta>8B^2D\theta$ in \cref{prop:SCNN-chart}.
\end{lemma}
\begin{proof}{\bf{of \cref{lem:approx-error-tradeoff} }}
    As in \cref{prop:multiplication}, $A_{i,1}$ measures the approximation error from $\hat{\times}$: 
    \[
        A_{i,1}=\sum_{j\in J(p)}\norm{\hat{\times}\left(\hat{f}^{\mathrm{SCNN}}_{i,j}\circ\hat{\phi}_i, {\hat{\ind}_\Delta}\circ{\hat{d}_i^2}\right) - \hat{f}^{\mathrm{SCNN}}_{i,j}\circ\hat{\phi}_i\times ({\hat{\ind}_\Delta}\circ{\hat{d}_i^2})}_\infty\leq N\eta.
    \]

    The term $A_{i,2}$ measures the error from CNN approximation of local Lipschitz functions. As in \cref{prop:SCNN-local}, $A_{i,2}\leq 4L_idN^{-\alpha/d}$.

    The term $A_{i,3}$ measures the error from the CNN approximation of the chart determination function. The bound of $A_{i,3}$ can be derived using \cref{prop:SCNN-chart} and the proof of Lemma 4 in \citep{chen2022nonparametric}, since $\overbar{f}_i$ is a Lipschitz function on $[0,1]^d$.

    Finally, by \cref{lem:lip-component}, we have $L_i=O(L_f+\norm{f}_\infty)$, and the proof is complete.
\end{proof}

In order to attain the error desired in \cref{lem:approx-error-tradeoff}, we need each network in \eqref{eq:SCNN-local} with appropriate size. The network size of the components can be analyzed as follows:

\begin{itemize}
    \item ${\hat{\ind}_i}$: The chart determination network ${\hat{\ind}_i}\coloneqq \hat{d}_i^2\circ{\hat{\ind}_\Delta}$ is the composition of $\hat{d}_i^2$ and ${\hat{\ind}_\Delta}$. By \cref{prop:SCNN-chart}, $\hat{d}_i^2$ is a single-block CNN with $O(\log\frac{1}{\theta}+D)=O(\frac{\alpha}{d}\log N + D + \log D)$ layers and $6D$ channels; ${\ind}_\Delta$ is a single-block CNN with $O(\log{(\beta^2/\Delta)})=O(\frac{\alpha}{d}\log N)$ layers and $2$ channels. In both subnetworks, all parameters are bounded by $O(1)$. By \cref{lem:SCNN-comp}, the chart determination network ${\hat{\ind}_i}$ is a single-block CNN with $O(\frac{\alpha}{d}\log N + D + \log D)$ layers, $6D+2$ channels and all weight parameters bounded by $O(1)$.

    \item $\hat{\times}$: By \cref{prop:multiplication}, the multiplication network is a single-block CNN with $O(\log\frac{1}{\eta})=O((1+\frac{\alpha}{d})\log N)$ layers and $O(1)$ channels. By construction of $\hat{f}_{i,j}^{\mathrm{SCNN}}$ and $\hat{\ind}_\Delta$, all weight parameters are bounded by $(\norm{f}_\infty^2\vee 1)$.

    \item $\hat{\phi}_i$: The projection $\phi_i$ is a linear mapping, so it can be expressed with a single-layer perceptron. By Lemma 8 in \citet{liu2021besov}, this single-layer perceptron can be expressed with a single-block CNN with $D+2$ layers and width $d$. All parameters are of order $O(1)$.

    \item $\hat{f}_{i,j}^{\mathrm{SCNN}}$: By \cref{prop:SCNN-local}, each $\hat{f}_{i,j}^{\mathrm{SCNN}}$ is a single-block CNN with $O(\log N)$ layers and $80d$ channels. All weight parameters are in the order of $O(N^{\frac{1}{d}})$.
    
\end{itemize}

Next, we show that the composition $\hat{\times}\left(\hat{f}^{\mathrm{SCNN}}_{i,j}\circ\hat{\phi}_i, {\hat{\ind}_\Delta}\circ{\hat{d}_i^2}\right)$ can be simply expressed as a single-block CNN.
By \cref{lem:SCNN-comp}, there exists a single-block CNN $g_{i,j}$ with $O(\log N+D)$ layers and $80d$ channels realizing $\hat{f}^{\mathrm{SCNN}}_{i,j}\circ{\hat{\phi}_i}$. All parameters in $g_{i,j}$ are in the order of $O(N^{\frac{1}{d}})$. Moreover,  recall that the chart determination network $\hat{\ind}_i$ is a single-block CNN with $O(\log N + D + \log D)$ layers and $6D+2$ channels, whose parameters are of $O(1)$. By Lemma 14 in \citet{liu2021besov}, one can construct a convolutional block, denoted by $\overbar{g}_{i,j}$, such that 
\begin{align}
    \overbar{g}_{i,j}(x)=\begin{bmatrix}
        (g_{i,j}(x))_+ & (g_{i,j}(x))_- & (\widehat{\ind}_i(x))_+ & (\widehat{\ind}_i(x))_-\\
        \star & \star & \star &\star
    \end{bmatrix}.
    \label{eq:g-stack}
\end{align}
Here $\overbar{g}_{i,j}$ has $80d+6D+2$ channels.

Since the input of $\hat{\times}$ is $\begin{bmatrix}
    g_{i,j}\\
    \widehat{\ind}_i
\end{bmatrix}$, by Lemma 15 in \citet{liu2021besov}, there exists a CNN $\mathring{g}_{i,j}$ which takes \eqref{eq:g-stack} as the input and outputs $\hat{\times}(g_{i,j},\widehat{\ind}_i)$. 

Since $\overbar{g}_{i,j}$ only contains convolutional layers, the composition $\mathring{g}_{i,j}\circ\overbar{g}_{i,j}$, denoted by $\widehat{g}^{\rm SCNN}_{i,j}$, is a single-block CNN and for any $x\in\cS$, $\widehat{g}^{\rm SCNN}_{i,j}(x)=\hat{\times}\left(\hat{f}^{\mathrm{SCNN}}_{i,j}\circ{\hat{\phi}_i}(x), {\hat{\ind}_\Delta}\circ{\hat{d}_i^2}(x)\right)$. We have $\widehat{g}^{\rm SCNN}_{i,j}\in \cF^{\rm SCNN}(L,J,I,R,R)$ with
\begin{align*}
    L=O(\log N+D+\log D), ~ J=160d+12D+O(1), ~ R=O(N^{\frac{1}{d}}),
\end{align*}
and $I$ can be any integer in $[2,D]$.
Therefore, we have shown that $\widehat{g}^{\rm SCNN}_{i,j}$ is a single-block CNN that expresses the composition \eqref{eq:SCNN-local}, as we desired.

Furthermore, recall that $\hat{f}$ can be written as a sum of $C_\cS N$ such single-block CNNs. By \cref{lem:SCNN-sum-mod}, for any $\tilde{M}$ and $\tilde{J}$ satisfying $\tilde{M}\tilde{J}=O(C_\cS N)$ and $\tilde{J}\geq 160d+12D+O(1)$, there exists a CNN architecture $\cF^{\mathrm{SCNN}}(L,J,I,R,R)$ that gives rise to a set of single-block CNNs $\{\hat{g}_i\}_{i=1}^{\tilde{M}}\subset\cF^{\mathrm{SCNN}}(L,J,I,R,R)$ with 
\begin{align}\label{eq:sum-SCNN}
    \hat{f}=\sum_{i=1}^{\tilde{M}}\hat{g}_i
\end{align}
and 
\begin{align*}
    L=O(\log N + D + \log D), ~ J=O(D), ~ R=O(N^{\frac{1}{d}}).
\end{align*}

By \cref{lem:SCNN-norm-mod}, we slightly adjust the CNN architecture by re-balancing the weight parameters of the convolutional blocks and that of the final fully connected layer. In particular, we rescale all parameters in convolutional layers of $\hat{g}_i$ to be no larger than $1$. This procedure preserves the approximation power of the CNN while reducing the covering number (see \cref{sec:appendix-covering}) of the CNN class. We set $\lambda=c^{\prime}N^{\frac{1}{d}}(8ID)\tilde{M}^{\frac{1}{L}}$, where $c^\prime$ is a constant such that $R\leq c^{\prime}N^{\frac{1}{d}}$. With this $\lambda$, we have $\hat{f}_i\in\cF^{\mathrm{SCNN}}(L,J,I,R_1,R_2)$ with 
\begin{align}
    & L=O(\log N + D + \log D), ~ J=O(D), ~ R_1=O((8ID)^{-1}\tilde{M}^{-\frac{1}{L}})=O(1),\nonumber\\ 
    & \log R_2=O(\log\tilde{M}+\log^2N+D\log N)\label{eq:network-size}
\end{align}
such that $\hat{g}_i\equiv \hat{f}_i$.

Finally, we prove that it suffices to use one CNN to realize the sum of single-block CNNs in \eqref{eq:sum-SCNN}. By \cref{lem:SCNN-sum-CNN}, there exists a CNN that can express the sum of $\tilde{M}$ single-block CNNs with architecture $\cF(M,L,J,I,R_1,R_2)$, where 
\begin{align*}
    & M=O(\tilde{M}), ~ L=O(\log N + D + \log D), ~ J=O(D\tilde{J}), \\
    & R_1=O((8ID)^{-1}\tilde{M}^{-\frac{1}{L}})=O(1), ~ \log R_2=O(\log\tilde{M}+\log^2N+D\log N).
\end{align*}
Here, $\tilde{M},\tilde{J}$ satisfy 
\[
    \tilde{M}\tilde{J}=O(N),
\]
which is a requirement inherited from \cref{lem:SCNN-sum-mod}. This CNN is our final approximation of $f_0$.

Applying this relation $N=O(\tilde{M}\tilde{J})$ to \eqref{eq:network-size} gives 
\[
    \norm{\hat{f} - f_0}\leq (L_f+\norm{f_0}_\infty)(\tilde{M}\tilde{J})^{-\frac{\alpha}{d}}
\]
and the network size 
\begin{align*}
    & M=O(\tilde{M}), ~ L=O(\log {(\tilde{M}\tilde{J})} + D + \log D), ~ J=O(D\tilde{J}), \\
    & R_1=O((8ID)^{-1}\tilde{M}^{-\frac{1}{L}})=O(1), ~ \log R_2=O(\log\tilde{M}+\log^2{(\tilde{M}\tilde{J})}+D\log {(\tilde{M}\tilde{J})}).
\end{align*}

By \eqref{eq:V-bound} and \cref{lem:lip-Q}, we have $L_f=L_Q=\frac{C}{1-\gamma}\overbar{L}_Q$ and $\norm{f_0}_\infty=\frac{C}{1-\gamma}$, which completes the proof of \cref{thm:approx-error-Q}.
\end{proof}

\subsection{Proof of Lemma {\ref{lem:lip-Q}}}\label{proof:lip-Q}
\begin{proof}
    Let us recall a useful characterization for total variation distance between probability measures $\mu$, $\nu$ on any measurable space $\cX$: 
    \begin{align}\label{eq:dtv}
        \dist_\mathrm{TV}(\mu,\nu)
        &=\frac{1}{2}\sup_{f\colon\cX\to[-1,1]}\abs{\int_\cX f\ud \mu - \int_\cX f\ud \nu}.
    \end{align}
    Define \[
        f(s)=\frac{2(1-\gamma)}{C}\cdot V^\pi(s)-1.
    \]
    By \eqref{eq:V-bound} we have $-1\leq f(s)\leq 1$ for any $s\in\cS$. 
    In view of \eqref{eq:Q-to-V}, \eqref{eq:dtv} and that $\cM$ is $(L_\cP,L_c)$-Lipschitz, we have
    \begin{align*}
        \abs{Q^\pi(s,a)-Q^\pi(s^\prime,a)}
        &\leq \abs{c(s,a)-c(s^\prime,a)} + \gamma\abs{\int_{\cS}V^\pi(s^{\prime\prime}) \ud \bigl(\cP(s^{\prime\prime}|s,a)-\cP(s^{\prime\prime}|s^\prime,a)\bigr)}\\
        &= \abs{c(s,a)-c(s^\prime,a)} + \frac{\gamma C}{2(1-\gamma)}\abs{\int_{\cS}f(s^{\prime\prime}) \ud \bigl(\cP(s^{\prime\prime}|s,a)-\cP(s^{\prime\prime}|s^\prime,a)\bigr)}\\
        &\leq \abs{c(s,a)-c(s^\prime,a)} + \frac{\gamma C}{(1-\gamma)}\dist_\mathrm{TV}(\cP(\cdot|s,a), \cP(\cdot|s^\prime,a))\\
        &\leq L_c\cdot\dist_\cS^\alpha(s,s^\prime) + \frac{\gamma C}{1-\gamma} L_\cP\cdot\dist_\cS^\alpha(s,s^\prime)\\
        &= L_Q\cdot\dist_\cS^\alpha(s,s^\prime).
    \end{align*}
    The first line is from \eqref{eq:Q-to-V} and the triangle inequality. The second line is from the definition of $f$ and that $\cP(\cdot|s, a)$ is a distribution. The third line is from \eqref{eq:dtv}, and the fourth line is from the Lipschitz assumption.
\end{proof}

\subsection{Proof of Theorem {\ref{thm:approx-error-approxlip}}}\label{proof:approx-error-approxlip}
\begin{proof}
We first show in \cref{lem:lip-reference} that for any approximately Lipschitz function $f_0$, there exists a Lipschitz ``reference function'' that is not far from $f_0$ in the $L^\infty$ sense and has the same Lipschitz constant. 

\begin{lemma}\label{lem:lip-reference}
    Suppose \cref{asm:manifold} holds. If a function $f_0\colon\cS\to\RR$ is $(L,\alpha,\epsilon)$-approximately Lipschitz, then there exists an $(L,\alpha)$-Lipschitz function $\overbar{f}_0$ such that $\norm{\overbar{f}_0}_\infty\leq\norm{f_0}_\infty$ and $\norm{f_0-\overbar{f}_0}_\infty\leq 2\epsilon$.
\end{lemma}
\begin{proof}{\bf{of \cref{lem:lip-reference} }}
    Define an envelope function $f_L(x)\coloneqq\inf_{y\in\cS}\left\{f_0(y)+L\cdot d_\cS^\alpha(y,x)\right\}$. It follows immediately from the $(L,\alpha,\epsilon)$-approximate Lipschitzness of $f$ that for any $x\in\cS$, 
    \begin{align*}
        &f_L(x)
        \leq f_0(x)+L\cdot d_\cS^\alpha(x,x)=f_0(x),\\
        &f_L(x)
        \geq\inf_{y\in\cS}\left\{f_0(x)-L\cdot d_\cS^\alpha(y,x)-2\epsilon+L\cdot d_\cS^\alpha(y,x)\right\}
        = f_0(x)-2\epsilon,
    \end{align*}
    hence $\norm{f_L-f_0}_\infty\leq 2\epsilon$. Furthermore, for any $x,y\in\cS$, 
    \begin{align*}
        f_L(x)-f_L(y)
        &=\inf_{z\in\cS}\left\{f_0(z)+L\cdot d_\cS^\alpha(z,x)\right\}-\inf_{z\in\cS}\left\{f_0(z)+L\cdot d_\cS^\alpha(z,y)\right\}\\
        &=\inf_{z\in\cS}\left\{f_0(z)+L\cdot d_\cS^\alpha(z,x)\right\}+\sup_{z\in\cS}\left\{-f_0(z)-L\cdot d_\cS^\alpha(z,y)\right\}\\
        &=\sup_{z\in\cS}\left\{\inf_{z^\prime\in\cS}\left\{f_0(z^\prime)+L\cdot d_\cS^\alpha(z^\prime,x)\right\}-f_0(z)-L\cdot d_\cS^\alpha(z,y)\right\}\\
        &\leq\sup_{z\in\cS}\left\{f_0(z)+L\cdot d_\cS^\alpha(z,x)-f_0(z)-L\cdot d_\cS^\alpha(z,y)\right\}\\
        &\leq\sup_{z\in\cS}\left\{L\cdot (d_\cS(z,y) + d_\cS(y,x))^\alpha-L\cdot d_\cS^\alpha(z,y)\right\}\\
        &\leq\sup_{z\in\cS}\left\{L\cdot (d_\cS^\alpha(z,y) + d_\cS^\alpha(y,x))-L\cdot d_\cS^\alpha(z,y)\right\}\\
        &\leq L\cdot d_\cS^\alpha(x,y).
    \end{align*}
    The first inequality is from that $\inf_{z^\prime\in\cS}\left\{f_0(z^\prime)+L\cdot d_\cS^\alpha(z^\prime,x)\right\}\leq f_0(z)+L\cdot d_\cS^\alpha(z,x)$ for any $z\in\cS$. The second inequality is from the triangle inequality of $\dist_\cS(\cdot,\cdot)$ and the third inequality is from the subadditivity of $h(x)=x^\alpha$ when $x\geq 0$ and $\alpha\in(0,1]$. Similarly we have $f_L(y)-f_L(x)\leq L\cdot d_\cS^\alpha(y,x)$ and hence $f_L$ is $L$-Lipschitz. By truncating the negative part of $f_L$ by $-\norm{f_0}_\infty$, we obtain $\overbar{f}_0$ such that \[
        \overbar{f}_0(x)=\begin{cases}
            f_L(x), & f_L(x)\geq -\norm{f_0}_\infty,\\
            -\norm{f_0}_\infty, & f_L(x)< -\norm{f_0}_\infty.
        \end{cases}
    \] We can easily verify that $\overbar{f}_0$ is $L$-Lipschitz, $\norm{\overbar{f}_0}_\infty\leq\norm{f_0}_\infty$ and $\norm{f_0-\overbar{f}_0}_\infty\leq 2\epsilon$. 
\end{proof}

By \cref{lem:lip-reference}, there exists an $(L_f,\alpha)$-Lipschitz function $\overbar{f}_0$ such that $\norm{\overbar{f}_0}_\infty\leq\norm{f_0}_\infty$ and 
\[
    \norm{\overbar{f}_0-f_0}_\infty\leq 2\epsilon_f.
\]
Therefore, similar to \cref{thm:approx-error-Q}, for any integers $I \in [2, D]$, $\Tilde{M}, \Tilde{J}>0$, 
\begin{align*}
    &M=O(\Tilde{M}), ~ L=O(\log{(\Tilde{M}\Tilde{J})}+D+\log D), ~ J=O(D\Tilde{J}), \\ 
    &\log R_2=O(\log^2 (\Tilde{M}\Tilde{J}) + D\log {(\Tilde{M}\Tilde{J})}), ~ R_1=(8ID)^{-1}\Tilde{M}^{-\frac{1}{L}}=O(1),
\end{align*}
there exists a CNN $f\in\cF(M,L,J,I,R_1,R_2)$ such that 
\begin{align*}
    \norm{f-f_0}_\infty 
    \leq \norm{f-\overbar{f}_0}_\infty + \norm{\overbar{f}_0-f_0}_\infty 
    \leq (L_f+\norm{f_0}_\infty)(\Tilde{M}\Tilde{J})^{-\frac{\alpha}{d}} + 2\epsilon_f.
\end{align*}
where $O(\cdot)$ hides a constant depending on $\log L_f$, $\log \norm{f_0}_\infty$, $d$, $\alpha$, $\omega$, $B$, and the surface area $\Area(\cS)$. 

The rest of the proof is to show that $f$ is uniformly bounded by $\norm{f_0}_\infty$ and is $(L_f, \alpha, \hat{\epsilon}_f)$-approximately Lipschitz with $\hat{\epsilon}_f=(L_f+\norm{f_0}_\infty)(\Tilde{M}\Tilde{J})^{-\frac{\alpha}{d}}$. 
To show the uniform upper bound, we can apply a truncation layer to the components of $\hat{f}$ so that every output will not exceed the range $[-\norm{f_0}_\infty, \norm{f_0}_\infty]$. This can be realized by adding a two-layer ReLU network $g\colon\RR\to\RR$, \[
    g(x)=\ReLU(2\norm{f_0}_\infty - \ReLU(\norm{f_0}_\infty-x))-\norm{f_0}_\infty.
\] 
By Theorem 1 in \citet{oono2019approximation}, such a ReLU network can be expressed by a CNN $\hat{g}$ with constant parameters. By \cref{lem:SCNN-comp}, applying this CNN to the output of $f$ results in a new CNN with the same order of size. In this case, we simply replace $\hat{f}$ with this new CNN. 
By \cref{lem:lip-reference}, $\norm{\overbar{f}_0}_\infty\leq\norm{f_0}_\infty$, so the truncation layer would not affect the approximation error, and we complete the proof for the uniform upper bound.

By \cref{lem:approx-lip}, we conclude that $f$ is $(L_f,\alpha,\hat{\epsilon}_f)$-approximately Lipschitz with $\hat{\epsilon}_f=(L_f+\norm{f_0}_\infty)(\tilde{M}\tilde{J})^{\frac{\alpha}{d}}$.
\end{proof}

\subsection{Supporting Lemmas for CNN Approximation}\label{sec:appendix-approx-lemmas}
In this section, we provide some auxiliary lemmas for CNN approximation. 
\cref{lem:lip-component} shows that each local function $\overbar{f}_i$ defined in \eqref{eq:def-component} is Lipschitz on the low-dimensional Euclidean unit cube $[0,1]^d$. The Lipschitz constant $L_i$ is controlled by the $L^\infty$ norm and the Lipschitz constant of the original function $L_f$
\begin{lemma}\label{lem:lip-component}
    Let $\overbar{f}_i$ be defined as in \eqref{eq:def-component}. Then each function $\overbar{f}_i$ is uniformly bounded by $\norm{f}_\infty$ and is $(L_i,\alpha)$-Lipschitz on $[0, 1]^d$ with $L_i=O(L_f+\norm{f}_\infty)$, where $O(\cdot)$ hides some constant depending on $\alpha,\omega$, $\phi_i$ and $\rho_i$.
\end{lemma}
\begin{proof}
    We only need to show the Lipschitz continuity on the support of $f_i\circ\phi_i^{-1}$. Otherwise, the Lipschitz condition holds trivially since $f_i\circ\phi_i^{-1}$ is bounded and extended to the whole unit cube with $0$. 
    
    Suppose $x, y\in\supp(f_i\circ\phi_i^{-1})$ are two points in the support, then there exist $u,v\in \supp(\rho_i)\subset U_i$ such that $u=\phi_i^{-1}(x), v=\phi_i^{-1}(y)$. 
    By definition of the chart $(U_i,\phi_i)$, we have $\norm{u-v}_2\leq 2\beta<\frac{\omega}{2}$ and that 
    \[
        \norm{x-y}_2
        =\norm{a_i V_i^\top(u-v)}_2
        \leq\norm{u-v}_2
        <\frac{\omega}{2},
    \]
    since $a_i\leq 1$ and $V_i$ is orthonormal. According to Proposition 6.3 in \cite{niyogi2008finding}, the geodesic distance between $u$ and $v$ is upper bounded by the Euclidean distance in $\RR^D$ up to a constant factor:
    \[
        \dist_\cS(u,v)\leq \omega-\omega\sqrt{1-\frac{2\norm{u-v}_2}{\omega}}\leq 2\norm{u-v}_2.
    \]
    By Lemma 2 in \cite{chen2022nonparametric}, $\phi_i$ is a diffeomorphism as long as the Euclidean ball radius satisfies $\beta\leq\frac{\omega}{4}$. By our construction, $\beta<\frac{\omega}{4}$, thus $\phi_i^{-1}$ is differentiable with its Jacobian bounded. Therefore, there exists a constant $L_{i,1}$ such that 
    \begin{align*}
        L_{i,1}\coloneqq \sup_{z\in [0,1]^d}\norm{\nabla\phi_i^{-1}(z)}_{\mathrm{op}}<+\infty.
    \end{align*}
    where $\nabla \phi_i^{-1}(z)$ is the Jacobian of $\phi_i^{-1}$ at $z$ and $\norm{\cdot}_{\mathrm{op}}$ denotes the operator norm. 
    Also notice that $\rho_i$ is $C^\infty$, thus we conclude that there exist another constant $L_{i,2}$ such that 
    \begin{align*}
        &L_{i,2}\coloneqq \sup_{z\in [0,1]^d}\norm{\nabla(\rho_i\circ\phi_i^{-1})(z)}_2<+\infty.
    \end{align*}
    Combine the results together, we have 
    \begin{align*}
        &\abs{f_i\circ\phi_i^{-1}(x)-f_i\circ\phi_i^{-1}(y)} \\ 
        =& \abs{f_i(u) - f_i(v)} \\ 
        =& \abs{f(u)\rho_i(u) - f(v)\rho_i(v)} \\ 
        \leq& \abs{f(u)\rho_i(u) - f(u)\rho_i(v)} + \abs{f(u)\rho_i(v) - f(v)\rho_i(v)} \\ 
        \leq& \norm{f}_\infty\abs{\rho_i\circ\phi_i^{-1}(x) - \rho_i\circ\phi_i^{-1}(y)} + \abs{f(u) - f(v)} \\ 
        \leq& \norm{f}_\infty\sup_{z\in [0,1]^d}\norm{\nabla(\rho_i\circ\phi_i^{-1})(z)}_2\norm{x-y}_2 + L_f\cdot\dist_\cS^\alpha(u,v) \\ 
        \leq& \norm{f}_\infty \left(\frac{\omega}{2}\right)^{1-\alpha} L_{i,2}\cdot\norm{x-y}_2^\alpha + L_f 2^\alpha L_{i,1}^\alpha\cdot\norm{x-y}_2^\alpha \\ 
        \leq & C_i(\norm{f}_\infty + L_f)\norm{x-y}_2^\alpha
    \end{align*}
    where $C_i=\max(2^\alpha L_{i,1}^\alpha, (\frac{\omega}{2})^{1-\alpha} L_{i,2})$ is a constant depending on $\alpha,\omega$, $\phi_i$ and $\rho_i$. Denote $L_i\coloneqq C_i(\norm{f}_\infty + L_f)$, we conclude that $f_i\circ\phi_i^{-1}$ is $(L_i,\alpha)$-Lipschitz.
\end{proof}

\cref{lem:spline-local} further shows that Lipschitz functions on the unit cube $[0,1]^d$ can be arbitrarily approximated by first-order B-splines. The approximation error is $O(N^{-\alpha/d})$. 
\begin{lemma}\label{lem:spline-local}
    Let $f$ be an $(L,\alpha)$-Lipschitz function on the unit cube $[0,1]^d$ and take nonzero value only in the interior of the cube. For any $p\in\NN$, $p\geq 1$, $N=2^{pd}$, there exists a function $\tilde{f}_N$ in the form 
    \begin{align*}
        \tilde{f}_N=\sum_{j\in J(p)}c_{j}M_{p,j}
    \end{align*}
    such that 
    \begin{align*}
        \norm{\tilde{f}_N - f}_\infty\leq 2L dN^{-\alpha/d},
    \end{align*}
    where $\max_{j\in J(p)}c_j=\norm{f}_\infty$.
\end{lemma}
\begin{proof}
    For any $p\in\NN$, $p\geq 1$, the index set $J(p)=\{0, 1, \dots, 2^p-2\}^d$ as defined in Step 2 of \cref{sec:appendix-approx-overview}. We denote $G(p)=\{2^{-k}y\mid y\in \NN^d, 0\leq y_k\leq 2^p, \forall k\in[d]\}$ as the set of all grid points. Let 
    \begin{align*}    
        \tilde{f}_N=\sum_{j\in J(p)}c_{j}M_{p,j},
    \end{align*}
    where $c_{j}=f(2^{-p}(j_1+1), \dots, 2^{-p}(j_d+1))$ for all $j\in J(p)$. 
    By definition of the first-order B-spline and that $f$ takes nonzero value only in the interior of the cube, we have $\tilde{f}_N(x)=0$ if $x_k\in\{0, 1\}$ for some $k\in[d]$ and thus $\tilde{f}_N(x)=f(x)$ for all $x\in G(p)$. Moreover, $\tilde{f}_N$ is a coordinate-wise $(L,\alpha)$-Lipschitz linear function.
    
    For any point $x\in[0, 1]^d$, there exists a grid point $y\in G(p)$ such that $\norm{x-y}_\infty\leq 2^{-p-1}$. Define a sequence of points $\{y^{(t)}\}_{t=0}^d$ as 
    \begin{align*}
        y^{(t)}_k=\begin{cases}
            y_k, & k\leq t,\\
            x_k, & k>t.
        \end{cases}
    \end{align*}
    We have $y^{(0)}=x$ and $y^{(d)}=y$. We can move the point $x$ to $y$ by changing one coordinate at a time following the sequence $\{y^{(t)}\}$. Thus we have 
    \begin{align*}
        \abs{\tilde{f}_N(x)-f(x)}
        =&\abs{\tilde{f}_N(y)-f(y) + f(y) - f(x) + \sum_{t=0}^{d-1}\left(\tilde{f}_N(y^{(t)})-\tilde{f}_N(y^{(t+1)})\right)}\\
        \stackrel{(a)}{\leq} & \abs{f(y) - f(x)} + \sum_{t=0}^{d-1}\abs{\tilde{f}_N(y^{(t)})-\tilde{f}_N(y^{(t+1)})}\\
        \stackrel{(b)}{\leq} & L\norm{x-y}_2^\alpha + L \sum_{t=0}^{d-1}\norm{y^{(t)}-y^{(t+1)}}_2^\alpha\\
        \stackrel{(c)}{\leq} & 2^{1-(p+1)\alpha}Ld,
    \end{align*}
    where (a) uses $\tilde{f}_N(y)=f(y)$ and the triangle inequality, (b) uses the Lipschitz continuity, and (c) is from the upper bound for norms. Since $x$ is arbitrarily chosen from the unit cube, we have 
    \[
        \norm{\tilde{f}_N-f}_\infty\leq 2^{1-(p+1)\alpha}L d.
    \]
    Plugging in $N=2^{pd}$ yields the result.
\end{proof}

\cref{prop:SCNN-spline} is a special case of Lemma 10 in \citet{liu2021besov} for first-order splines (first-order cardinal B-splines). It shows that a single-block CNN can approximate each first-order B-spline to arbitrary accuracy. 
\begin{proposition}[{\citet[Lemma 10]{liu2021besov}}]\label{prop:SCNN-spline}
    Let $p\in \NN$ and $j\in\NN^d$. There exists a constant $C$ depending only on $d$ such that for any $\epsilon\in(0, 1)$ and any $2\leq I\leq d$, there exists a single-block CNN $\hat{M}_{p,j}\in\cF^{\mathrm{SCNN}}(L,J,I,R,R)$ with $L=3+2\lceil\log_2(\frac{3}{C\epsilon})+5\rceil\lceil\log_2 d\rceil$, $J=80d$ and $R=4\vee 2^p$ that satisfies \[
        \norm{M_{p,j} - \hat{M}_{p,j}}_{L^\infty([0,1]^d)}\leq \epsilon,
    \]
    and $\hat{M}_{p,j}(x)=0$ for all $x\notin B_{p,j}\coloneqq\{x\in\RR^d \mid 2^{-p}j_k\leq x_k\leq 2^{-p}(j_k+2)\}$.
\end{proposition}

As a result, we show in \cref{prop:SCNN-local} that CNNs can approximate local functions.
\begin{proposition}\label{prop:SCNN-local}
    Let $\overbar{f}_i$ be defined as in \eqref{eq:def-component}. 
    For $N=2^{pd}$ and any $2\leq I\leq D$, there exists a set of single-block CNNs $\{\hat{f}_{i,j}^\mathrm{SCNN}\}_{j\in J(p)}$ such that \[
        \norm{\sum_{j\in J(p)}\hat{f}_{i,j}^{\mathrm{SCNN}}-\overbar{f}_i}_{L^\infty([0,1]^d)}\leq 4L_idN^{-\alpha/d}.
    \]
    Each single-block CNN $\hat{f}_{i,j}^{\mathrm{SCNN}}$ is in $\cF^{\mathrm{SCNN}}(L,J,I,R,R)$ with \[
        L=O(\log N), ~ J=80d, ~ R=O(N^{\frac{1}{d}}),
    \]
    where $O(\cdot)$ hides some constant depending on $d$ and $\alpha$.
\end{proposition}
\begin{proof}
    By \cref{lem:spline-local} and the $(L_i,\alpha)$-Lipschitzness of $\overbar{f}_i$, for $p\geq 1$, $N=2^{pd}$, there exists a function $\tilde{f}_i$ in the form 
    \begin{align*}
        \tilde{f}_i=\sum_{j\in J(p)}c_{i,j}M_{p,j}
    \end{align*}
    such that 
    \begin{align*}
        \norm{\tilde{f}_i - \overbar{f}_i}_\infty\leq 2L_i dN^{-\alpha/d}.
    \end{align*}
    By \cref{prop:SCNN-spline}, there exists a collection of single-block CNNs $\{\hat{M}_{p,j}\}_{j\in J(p)}$ that approximates the first-order B-splines $\{M_{p,j}\}_{j\in J(p)}$. Suppose $\norm{\hat{M}_{p,j} - M_{p,j}}_{L^\infty}\leq\epsilon_1$ for all $j\in J(p)$ and some $\epsilon_1\in(0, 1)$, we have 
    \begin{align*}
        \norm{\sum_{j\in J(p)}c_{i,j}\hat{M}_{p,j} - \overbar{f}_i}_\infty
        &\leq \abs{J(p)}\norm{\overbar{f}_i}_\infty\epsilon_1 + 2L_i d N^{-\alpha/d}\\
        &\leq N L_i\epsilon_1 + 2L_i d N^{-\alpha/d},
    \end{align*}
    where the second inequality is from $\norm{\overbar{f}_i}_\infty=\norm{f}_\infty\leq L_i$ in \cref{lem:lip-component}.
    By letting $\epsilon_1=2d N^{-\frac{d+\alpha}{d}}$, we obtain
    \begin{align*}
        \norm{\sum_{j\in J(p)}c_{i,j}\hat{M}_{p,j} - \overbar{f}_i}_\infty
        &\leq 4L_idN^{-\alpha/d}.
    \end{align*}
    According to \cref{prop:SCNN-spline}, for any $j\in J(p)$, the single-block CNN $\hat{M}_{p,j}\in\cF^{\mathrm{SCNN}}(L,J,I,R,R)$ with 
    \begin{align*}
        &L=3+2\left\lceil \frac{d+\alpha}{d}\log_2\frac{3^{\frac{d}{d+\alpha}}N}{2C_0 d}\right\rceil\lceil\log_2 d\rceil, ~ J=80d, \\
        &2\leq I\leq d, ~ R=4\vee N^{\frac{1}{d}}.
    \end{align*}
    By letting $\hat{f}_{i,j}^{\mathrm{SCNN}}=c_{i,j}\hat{M}_{p,j}$ we prove the proposition. 
\end{proof}

The rest is to show that the multiplication operator and the indicator function can be approximated by CNNs. \cref{prop:multiplication} shows that CNN can approximate the multiplication of scalars. 
\begin{proposition}\label{prop:multiplication}
    Let $\times$ be the scalar multiplication operator. For any $\eta\in(0, 1)$, there exists a single-block CNN $\hat{\times}$ such that \[
        \norm{a\times b - \hat{\times}(a,b)}_{\infty}\leq \eta,
    \]
    where $a, b$ are functions uniformly bounded by $c_0$.
    The approximated single-block CNN $\hat{\times}$ is in $\cF^{\mathrm{SCNN}}(L,J,I,R,R)$ with $L=O(\log\frac{1}{\eta})+D$ layers, $J=24$ channels and any filter size $I$ such that $2\leq I\leq D$. All parameters are bounded by $R=(c_0^2 \lor 1)$. Furthermore, the weight matrix in the fully connected layer of $\hat{\times}$ has nonzero entries only in the first row.
\end{proposition}
\begin{proof}
    By Proposition 3 in \citet{yarotsky2017error}, there exists a feed-forward ReLU network that can approximate the multiplication operator between values with magnitude bounded by $c_0$ with $\eta$ error. Such a feed-forward network has $O(\log\frac{1}{\eta})$ layers, each layer has its width bounded by $6$, and all parameters are bounded by $c_0^2$. Therefore, such a feed-forward neural network is sufficient to approximate $\times$ with $\eta$ error in $L^\infty$-norm, since the function $a, b$ are uniformly bounded by $c_0$.

    Furthermore, by Lemma 8 in \citet{liu2021besov}, we can express the aforementioned feed-forward network with a single-block CNN in $\cF^{\mathrm{SCNN}}(L,J,I,R,R)$, where $L, J, I, R$ are specified in the statement of the proposition.
\end{proof}

The indicator function $\ind_{U_i}$ can be written as the composition of the indicator function of the closed interval $[0, \beta^2]$ and the squared Euclidean distance function $d_i\colon \cS\to\RR_+$ to the ball center $\cbb_i$: 
\begin{equation}\label{eq:indicator-decompose}
    \ind_{U_i}(x)={\ind_{[0,\beta^2]}}\circ {d_i(x)},
\end{equation}
where $d_i(x)=\norm{x-\cbb_i}_2^2$. As \cref{prop:SCNN-chart} shows, these components can be approximated by CNNs.

\begin{proposition}[{\citet[Lemma 9]{liu2021besov}}]\label{prop:SCNN-chart}
    Let $d_i$ and $\ind_{[0,\beta^2]}$ be defined as in \eqref{eq:indicator-decompose}. For any $\theta\in(0, 1)$ and $\Delta\geq 8B^2D\theta$, there exists a single-block CNN $\hat{d}_i$ approximating $d_i$ such that \[
        \norm{\hat{d}_i - d_i}_{\infty}\leq 4B^2D\theta, 
    \] and a single-block CNN $\hat{\ind}_\Delta$ approximating $\ind_{[0,\beta^2]}$ with 
    \begin{equation*}
        \hat{\ind}_\Delta(x)
        = \begin{cases}
            1, & \text{ if } x\leq (1-2^{-k})(\beta^2 - 4B^2D\theta),\\
            0, & \text{ if } x\geq \beta^2 - 4B^2D\theta,\\
            2^k((\beta^2-4B^2D\theta)^{-1}x - 1), & \text{ otherwise, }
        \end{cases}
    \end{equation*}
    for $x\in\cS$. The single-block CNN for $\hat{d}_i$ has $O(\log(1/\theta)+D)$ layers, $6D$ channels and all weight parameters are bounded by $4B^2$. The single-block CNN $\hat{\ind}_\Delta$ has $\left\lceil \log(\beta^2/\Delta)\right\rceil$ layers, $2$ channels and all weight parameters are bounded by $\max\left(2, \abs{\beta^2 - 4B^2D\theta}\right)$. 

    As a result, for any $x\in\cS$, $\hat{\ind}_\Delta \circ \hat{d}_i(x)$ gives an approximation of $\ind_{U_i}$ satisfying 
    \begin{equation*}
        \hat{\ind}_\Delta \circ \hat{d}_i(x)
        = \begin{cases}
            1, & \text{ if } x\in U_i \text{ and } d_i(x) \leq \beta^2-\Delta,\\
            0, & \text{ if } x\notin U_i,\\
            \text{between $0$ and $1$}, & \text{ otherwise. }
        \end{cases}
    \end{equation*}
\end{proposition}

\subsection{Supporting Lemmas for CNN Architecture}
In this section, we introduce several lemmas for (single-block) CNN architecture. 
\cref{lem:SCNN-comp} states that the composition of two single-block CNNs can be expressed as one single-block CNN with augmented architecture.
\begin{lemma}[{\citet[Lemma 13]{liu2021besov}}]\label{lem:SCNN-comp}
    Let $\cF_1=\cF^{\mathrm{SCNN}}(L_1,J_1,I_1,R_1,R_1)$ be a CNN architecture from $\RR^D\to\RR$ and $\cF_2=\cF^{\mathrm{SCNN}}(L_2,J_2,I_2,R_2,R_2)$ be a CNN architecture from $\RR\to\RR$. Assume the weight matrix in the fully connected layer of $\cF_1$ and $\cF_2$ has nonzero entries only in the first row. Then there exists a CNN architecture $\cF_3=\cF^{\mathrm{SCNN}}(L,J,I,R,R)$ from $\RR^D\to\RR$ with \[
        L=L_1+L_2, ~ J=\max(J_1,J_2), ~ I=\max(I_1, I_2), ~ R=\max(R_1, R_2)
    \]
    such that for any $f_1\in\cF_1$ and $f_2\in\cF_2$, there exists $f\in\cF_3$ such that $f(x)=f_2\circ f_1(x)$. Furthermore, the weight matrix in the fully connected layer of $\cF_3$ has nonzero entries only in the first row.
\end{lemma}

\cref{lem:SCNN-sum-mod} states that the sum of $n_0$ single-block CNNs with the same architecture can be expressed as the sum of $n_1$ single-block CNNs with modified width.
\begin{lemma}[{\citet[Lemma 7]{liu2022benefits}}]\label{lem:SCNN-sum-mod}
    Let $\{f_i\}_{i=1}^{n_0}$ be a set of single-block CNNs with architecture $\cF^{\mathrm{SCNN}}(L_0,J_0,I_0,R_0,R_0)$. For any integers $n$ and $\tilde{J}$ satisfying $1\leq n\leq n_0$, $n\tilde{J}=O(n_0J_0)$ and $\tilde{J}\geq J_0$, there exists an architecture $\cF^{\mathrm{SCNN}}(L,J,I,R,R)$ that gives a set of single-block CNNs $\{g_i\}_{i=1}^n$ such that \[
        \sum_{i=1}^n g_i(x)=\sum_{i=1}^{n_0}f_i(x).
    \]
    Such an architecture has \[
        L=O(L_0), ~ J=O(\tilde{J}), ~ I=I_0, ~ R=R_0.
    \]
    Furthermore, the fully connected layer of $f$ has nonzero elements only in the first row.
\end{lemma}

\cref{lem:SCNN-norm-mod} implies that one can slightly adjust the CNN architecture by re-balancing the weight parameter boundary of the convolutional blocks and that of the final fully connected layer. 
While re-balancing the weight would not affect the approximation power of the CNN, it will change the covering number of the CNN class, which is conducive to a different variance. 
\begin{lemma}[{\citet[Lemma 16]{liu2021besov}}]\label{lem:SCNN-norm-mod}
    Let $\lambda\geq 1$. For any $g\in\cF^{\mathrm{SCNN}}(L,J,I,R_1,R_2)$, there exists $f\in\cF^{\mathrm{SCNN}}(L,J,I,\lambda^{-1}R_1,\lambda^L R_2)$ such that $g(x)\equiv f(x)$.
\end{lemma}

Finally, we prove \cref{lem:SCNN-sum-CNN}, which states that the sum of single-block CNNs can be realized by a CNN of the form \eqref{eq:nn_class_no_mag}.
\begin{lemma}\label{lem:SCNN-sum-CNN}
    Let $\cF^{\mathrm{SCNN}}(L,J,I,R_1,R_2)$ be any CNN architecture from $\RR^D\to\RR$. Assume the weight matrix in the fully connected layer of $\cF^{\mathrm{SCNN}}(L,J,I,R_1,R_2)$ has nonzero entries only in the first row. For any positive integer $M$, there exists a CNN architecture $\cF(M,L,J+4,I,R_1,R_2(1\vee R_1^{-1}))$ such that for any $\{\hat{f}_i(x)\}_{i=1}^M\subset\cF^{\mathrm{SCNN}}(L,J,I,R_1,R_2)$, there exists $\hat{f}\in\cF(M,L,J+4,I,R_1,R_2(1\vee R_1^{-1}))$ with 
    \[
        \hat{f}(x)=\sum_{m=1}^M\hat{f}_m(x).
    \]
\end{lemma}
\begin{proof}
    Denote the architecture of $\hat{f}_m$ with 
    \[
        \hat{f}_m(x)=W_m\cdot\Conv_{\cW_m,\cB_m}(x),
    \]
    where $\cW_m=\{\cW_m^{(l)}\}_{l=1}^L$, $\cB_m=\{B_m^{(l)}\}_{l=1}^L$. Furthermore, denote the weight matrix and bias in the fully connected layer of $\hat{f}$ with $\hat{W}$, $\hat{b}$ and the set of filters and biases in the $m$-th block of $\hat{f}$ with $\hat{\cW}_m$ and $\hat{\cB}_m$ respectively. The padding layer $\hat{P}$ in $\hat{f}$ pads the input $x$ from $\RR^D$ to $\RR^{D\times 4}$ with zeros. Each column denotes a channel.

    Let us first show that for each $m$, there exists some $\Conv_{\hat{\cW}_m,\hat{\cB}_m}\colon\RR^{D\times 4}\to\RR^{D\times 4}$ such that for any $Z\in\RR^{D\times 4}$ with the form
    \begin{align}
        Z=\begin{bmatrix}
            (x)_+ & (x)_- & \star & \star
        \end{bmatrix},
        \label{eq.Zpad}
    \end{align}
    where $(x)_+$ means applying $(\cdot \vee 0)$ to every entry of $x$ and $(x)_-$ means applying $-(\cdot \wedge 0)$ to every entry of $x$, so all entries in $Z$ are non-negative. We have
    \begin{align}
        \Conv_{\widehat{\cW}_m,\widehat{\cB}_m}(Z)=\begin{bmatrix}
            & & \frac{R_1}{R_2}(f_m(\bx) \vee 0) & -\frac{R_1}{R_2}(f_m(\bx) \wedge 0)\\
            \mathbf{0} & \mathbf{0} & \star & \star\\
            & & \vdots & \vdots\\
            & & \star & \star
        \end{bmatrix} + Z
        \label{eq:padZ}
    \end{align}
    where $\star$'s denotes entries that do not affect this result and may take any different value.
	
    For any $m$, the first layer of $f_m$ takes input in $\RR^D$. Thus, the filters in $\cW_m^{(1)}$ are in $\RR^D$. Again, we pad these filters with zeros to get filters in $\RR^{D\times 4}$ and construct $\widehat{\cW}_m^{(1)}$ such that
    \begin{align*}
        &(\widehat{\cW}_m^{(1)})_{1,:,:} =\begin{bmatrix}
            \mathbf{e}_1 & \mathbf{0} & \mathbf{0} & \mathbf{0}
        \end{bmatrix}, \\
        &(\widehat{\cW}_m^{(1)})_{2,:,:} =\begin{bmatrix}
            \mathbf{0} & \mathbf{e}_1 & \mathbf{0} & \mathbf{0}
        \end{bmatrix}, \\
        &(\widehat{\cW}_m^{(1)})_{3,:,:} =\begin{bmatrix}
            \mathbf{0} & \mathbf{0} & \mathbf{e}_1 & \mathbf{0}
        \end{bmatrix}, \\
        &(\widehat{\cW}_m^{(1)})_{4,:,:} =\begin{bmatrix}
            \mathbf{0} & \mathbf{0} & \mathbf{0} & \mathbf{e}_1 
        \end{bmatrix}, \\
        &(\widehat{\cW}_m^{(1)})_{4+j,:,:} =\begin{bmatrix}
            (\cW_m^{(1)})_{j,:,:} & (-\cW_m^{(1)})_{j,:,:} & \mathbf{0} & \mathbf{0}
        \end{bmatrix},
    \end{align*}
    where we use the fact that $\cW_m^{(1)}*(x)_+ - \cW_m^{(1)}*(x)_-=\cW_m^{(1)}*x$. The first four output channels at the end of this first layer are copies of $Z$. For the filters in later layers of $\widehat{f}_m$ and all biases, we simply set
    \begin{align*}
        &(\widehat{\cW}_m^{(l)})_{1,:,:} =\begin{bmatrix}
            \mathbf{e}_1 & \mathbf{0} & \mathbf{0} & \mathbf{0} & \cdots & \mathbf{0}
        \end{bmatrix} &\mbox{ for } l=2,\dots,L,\\
        &(\widehat{\cW}_m^{(l)})_{2,:,:} =\begin{bmatrix}
            \mathbf{0} & \mathbf{e}_1 & \mathbf{0} & \mathbf{0} & \cdots & \mathbf{0}
        \end{bmatrix} &\mbox{ for } l=2,\dots,L,\\
        &(\widehat{\cW}_m^{(l)})_{3,:,:} =\begin{bmatrix}
            \mathbf{0} & \mathbf{0} & \mathbf{e}_1 & \mathbf{0} & \cdots & \mathbf{0}
        \end{bmatrix} &\mbox{ for } l=2,\dots,L-1,\\
        &(\widehat{\cW}_m^{(l)})_{4,:,:} =\begin{bmatrix}
            \mathbf{0} & \mathbf{0} & \mathbf{0} & \mathbf{e}_1 & \cdots & \mathbf{0}
        \end{bmatrix} &\mbox{ for } l=2,\dots,L-1,\\
        &(\widehat{\cW}_m^{(l)})_{4+j,:,:}= \begin{bmatrix}
            \mathbf{0} & \mathbf{0} & \mathbf{0} & \mathbf{0} & (\cW_m^{(l)})_{j,:,:}
        \end{bmatrix} & \mbox{ for } l=2,\dots,L-1,\\
        &(\widehat{\cB}_m^{(l)})_{j,:,:} =\begin{bmatrix}
            \mathbf{0} & \mathbf{0} & \mathbf{0} & \mathbf{0} & (\cB_m^{(l)})_{j,:,:}
        \end{bmatrix} &\mbox{ for } l=1,\dots,L-1.
    \end{align*}
    In $\Conv_{\widehat{\cW}_m,\widehat{\cB}_m}$, an additional convolutional layer is constructed to realize the fully connected layer in $\widehat{f_m}$. By our assumption, only the first row of $W_m$ is nonzero. Furthermore, we set $\widehat{\cB}_m^{(L)}=\mathbf{0}$ and $\widehat{\cW}_m^{L}$ as size-one filters with three output channels in the form of
    \begin{align*}
        &(\widehat{\cW}_m^{(L)})_{3,:,:} =\begin{bmatrix}
            \mathbf{0} & \mathbf{0} & \mathbf{e}_1 & \mathbf{0} & \frac{R_1}{R_2}(W_m)_{1,:}
        \end{bmatrix}, \\
        &(\widehat{\cW}_m^{(L)})_{4,:,:} =\begin{bmatrix}
            \mathbf{0} & \mathbf{0} & \mathbf{0} & \mathbf{e}_1 & -\frac{R_1}{R_2}(W_m)_{1,:}
        \end{bmatrix}.
    \end{align*}
    Under such choices, \eqref{eq:padZ} is proved, and all parameters in $\widehat{\cW}_m,\widehat{\cB}_m$ are bounded by $R_1$.
	
    By composing all convolutional blocks, we have
    \begin{align*}
        &(\Conv_{\widehat{\cW}_M,\widehat{\cB}_M})\circ \cdots \circ (\Conv_{\widehat{\cW}_1,\widehat{\cB}_1})\circ P(x)\\
        = &\begin{bmatrix}
            & & \frac{R_1}{R_2}\sum_{m=1}^M (\widehat{f}_m \vee 0) &  -\frac{R_1}{R_2}\sum_{m=1}^M (\widehat{f}_m \wedge 0)\\
            (x)_+ & (x)_- & \star & \star\\
            & & \vdots & \vdots\\
            & & \star & \star
        \end{bmatrix}.
    \end{align*}
    Lastly, the fully connected layer can be set as
    \[
        \widetilde{W}=\begin{bmatrix}
            0 & 0 & \frac{R_2}{R_1} & -\frac{R_2}{R_1}\\
            \mathbf{0} & \mathbf{0} & \mathbf{0} & \mathbf{0}
        \end{bmatrix}, \ \widetilde{b}=0.
    \]
    Note that the weights in the fully connected layer are bounded by $R_2(1\vee R_1^{-1})$.
    
    The above construction gives
    \begin{align*}
        \widehat{f}(x)=\sum_{m=1}^M (\widehat{f}_m(x) \vee 0)+\sum_{m=1}^M (\widehat{f}_m(x) \wedge 0) =\sum_{m=1}^M \widehat{f}_m(x).
    \end{align*}
\end{proof}

\section{CNN Class Covering Number}\label{sec:appendix-covering}
\newcommand*{\trho}{\Tilde{\rho}}

In this section, we prove a bound on the covering number of the convolutional neural network class used in \cref{alg:NPMD}. The supporting lemmas and their proofs are provided in \cref{sec:appendix-covering-lemmas}

\begin{lemma}\label{lem:covering}
    Given $\delta > 0$, the $\delta$-covering number of the CNN class $\cF(M,L,J,I,R_1,R_2)$ satisfies
    \begin{align*}
        \cN(\delta, \cF(M,L,J,I,R_1,R_2), \norm{\cdot}_\infty) \leq \left(2(R_1 \vee R_2)\Lambda_1\delta^{-1}\right)^{\Lambda_2},
    \end{align*}
    where 
    \begin{align*}
        &\Lambda_1=(M+3)JD(1\vee R_2)(1\vee R_1)\trho\trho^+,\ \Lambda_2=ML(J^2I+J)+JD+1
    \end{align*}
    with $\trho= \rho^M,\trho^+=1+ML\rho^+, \rho=(JIR_1)^L$ and $\rho^+=(1\vee JIR_1)^L$.

    With a network architecture as stated in \cref{thm:approx-error-Q,thm:approx-error-approxlip}, we have
    \begin{align*}
        \log \cN(\delta, \cF(M,L,J,I,R_1,R_2), \norm{\cdot}_\infty) = O\left(\widetilde{M}\widetilde{J}^2D^3\log^5(\widetilde{M}\widetilde{J})\log\frac{1}{\delta}\right),
    \end{align*}
    where $O(\cdot)$ hides a constant depending on $d$, $\alpha$, $\omega$, $B$, and the surface area $\Area(\cS)$.
\end{lemma}

To show \cref{lem:covering}, we first prove a supporting lemma (\cref{lem:F-uniform-bound}) that relates the distance in the function space of CNNs (in the $L^\infty$ sense) to the distance in the parameter space. In this way, we transform the covering of the CNN class into the covering of the parameter space, which is simpler to deal with. We then give a proof of \cref{lem:covering} in \cref{proof:covering}.

\subsection{Supporting Lemmas for Lemma {\ref{lem:covering}}}\label{sec:appendix-covering-lemmas}

\cref{prop:covering1} below provides an upper bound on the $L^\infty$-norm of a series of convolutional neural network blocks in terms of its architecture parameters, e.g. number of layers, number of channels, etc.

Let $J^{(i)}_m$ be the number of channels in $i$-th layer of the $m$-th block, and let $I^{(i)}_m$ be the filter size of $i$-th layer in the $m$-th block. We define $Q_{[i,j]}$ as 
\begin{align*}
    Q_{[i,j]}(x)=&\left(\Conv_{\cW_j,\cB_j}\right)\circ\cdots\circ \left(\Conv_{\cW_i,\cB_i}\right)(x).
\end{align*}

\begin{proposition}\label{prop:covering1}
    {For $m = 1, 2, \cdots, M$ and $x\in[-1,1]^{D}$, we have 
    \begin{align*}
        \norm{Q_{[1,m]}(x)}_\infty \le (1 \vee R_1)\left(\prod_{j=1}^{m}\prod_{i=1}^{L_j}J^{(i-1)}_jI^{(i)}_jR_1 \right)\left(1 + \sum_{k=1}^{m}L_k\prod_{i=1}^{L_k}(1 \vee J^{(i-1)}_kI^{(i)}_kR_1)\right).
    \end{align*}
    }
\end{proposition}

\begin{proof}
    \begin{align*}
        &\quad \norm{Q_{[1,m]}(x)}_\infty\\
        &= \norm{\Conv_{\cW_m,\cB_m}(Q_{[1,m-1]}(x))}_\infty\\
        &\le \prod_{i=1}^{L_m}J^{(i-1)}_mI^{(i)}_mR_1\norm{Q_{[1,m-1]}(x)}_\infty + R_1L_m\prod_{i=1}^{L_m}(1 \vee J^{(i-1)}_mI^{(i)}_mR_1)\\
        &\le \norm{P(x)}_\infty\prod_{j=1}^{m}\prod_{i=1}^{L_j}J^{(i-1)}_jI^{(i)}_jR_1 + R_1\sum_{k=1}^{m}L_k\prod_{i=1}^{L_k}(1 \vee J^{(i-1)}_kI^{(i)}_kR_1)\prod_{l=j+1}^{m}\prod_{i=1}^{L_l}J^{(i-1)}_lI^{(i)}_lR_1\\
        &\le \norm{x}_\infty\prod_{j=1}^{m}\prod_{i=1}^{L_j}J^{(i-1)}_jI^{(i)}_jR_1 + R_1\sum_{k=1}^{m}L_k\prod_{i=1}^{L_k}(1 \vee J^{(i-1)}_kI^{(i)}_kR_1)\prod_{l=j+1}^{m}\prod_{i=1}^{L_l}J^{(i-1)}_lI^{(i)}_lR_1\\
        &\le(1 \vee R_1)\left(\prod_{j=1}^{m}\prod_{i=1}^{L_j}J^{(i-1)}_jI^{(i)}_jR_1 \right)\left(1 + \sum_{k=1}^{m}L_k\prod_{i=1}^{L_k}(1 \vee J^{(i-1)}_kI^{(i)}_kR_1)\right),
    \end{align*}
    where the first two inequalities are obtained by applying Proposition 9 from \citet{oono2019approximation} recursively.
\end{proof}

\cref{lem:F-uniform-bound} quantifies the sensitivity of a CNN with respect to small changes in its weight parameters. This will be used to create a discrete covering for the CNN class.

\begin{lemma}\label{lem:F-uniform-bound}
    Let $\epsilon>0$. For any $f, f^\prime\in \cF(M,L,J,I,R_1,R_2)$ such that $\norm{W-W^\prime}_\infty \le \epsilon$, $\norm{b-b^\prime}_\infty \le \epsilon$, $\norm{\cW_m^{(l)}-{\cW_m^{(l)}}^\prime}_\infty \le \epsilon$ and $\norm{\cB_m^{(l)}-{\cB_m^{(l)}}^\prime}_\infty \le \epsilon$ for all $m$ and $l$, where $(W,b,\{\{(\cW_m^{(l)},\cB_m^{(l)})\}_{l=1}^{L_m}\}_{m=1}^M)$ and $(W^\prime,b^\prime,\{\{({\cW_m^{(l)}}^\prime,{\cB_m^{(l)}}^\prime)\}_{l=1}^{L_m}\}_{m=1}^M)$ are the parameters of $f$ and $f^\prime$ respectively, we have
    \begin{align*}
        \norm{f - f^\prime}_\infty \le \Lambda_1\epsilon,
    \end{align*}
    where $\Lambda_1$ is defined in \cref{lem:covering}.
\end{lemma}

\begin{proof}
    For any $x\in [-1,1]^D$, 
    \begin{align*}
        &\quad \left|f(x) - f^\prime(x)\right|\\
        &= \left|W\otimes Q(x) + b - W^\prime\otimes Q^\prime(x) - b^\prime\right|\\
        &= \left|(W - W^\prime)\otimes Q(x) + b - b^\prime + W^\prime\otimes \left(Q(x) - Q^\prime(x)\right)\right|\\
        &= \left|(W - W^\prime)\otimes Q(x) + b - b^\prime + W^\prime\otimes \left(Q(x) - \Conv_{\cW_M,\cB_M}(Q^\prime(x)) + \Conv_{\cW_M,\cB_M}(Q^\prime(x)) - Q^\prime(x)\right)\right|\\
        &= \left|(W - W^\prime)\otimes Q(x) + b - b^\prime + \sum_{m=1}^{M}W^\prime\otimes Q_{[m+1,M]} \circ \left(\Conv_{\cW_m,\cB_m} - \Conv_{\cW_m^\prime,\cB_m^\prime}\right)\circ Q_{[0,m-1]}^\prime\right|\\
        &\le  \left|(W - W^\prime)\otimes Q(x;\theta) + b - b^\prime\right| + \sum_{m=1}^{M}\left|W^\prime\otimes Q_{[m+1,M]} \circ \left(\Conv_{\cW_m,\cB_m} - \Conv_{\cW_m^\prime,\cB_m^\prime}\right)\circ Q_{[0,m-1]}^\prime\right|\\
        &\stackrel{(a)}{\le} (3+M)JD(1 \vee R_1)(1 \vee R_2)\left(\prod_{j=1}^{M}\prod_{i=1}^{L_j}J^{(i-1)}_jI^{(i)}_jR_1 \right)\left(1 + \sum_{k=1}^{M}L_k\prod_{i=1}^{L_k}(1 \vee J^{(i-1)}_kI^{(i)}_kR_1)\right)\epsilon,
    \end{align*}
    where (a) is obtained through the following reasoning.
    
    The first term in (a) can be bounded as
    \begin{align*}
        &\quad \left|(W - W^\prime)\otimes Q(x) + b - b^\prime\right|\\
        &\le \left(\norm{W}_0 + \norm{W^\prime}_0\right)\norm{W-W^\prime}_\infty\norm{Q(x)}_{\infty} + \norm{b - b^\prime}_{\infty}\\
        &\le 2JD\epsilon\norm{Q(x)}_{\infty}  + \epsilon\\
        &\le 3JD\epsilon\norm{Q(x)}_{\infty}\\
        &\le 3JD\max\{1,R_1\}\left(\prod_{j=1}^{M}\prod_{i=1}^{L_j}J^{(i-1)}_jI^{(i)}_jR_1 \right)\left(1 + \sum_{k=1}^{M}L_k\prod_{i=1}^{L_k}(1 \vee J^{(i-1)}_kI^{(i)}_kR_1)\right)\epsilon,
    \end{align*}
    where the first inequality uses Proposition 8 from \citet{oono2019approximation} and the last inequality is obtained by invoking \cref{prop:covering1}.

    For the second term in (a), it is true that for any $m = 1, \cdots, M$, we have
    \begin{align*}
        &\quad \left|W^\prime\otimes Q_{[m+1,M]} \circ \left(\Conv_{\cW_m,\cB_m} - \Conv_{\cW_m^\prime,\cB_m^\prime}\right)\circ Q_{[1,m-1]}^\prime\right|\\
        &\stackrel{(b)}{\le} \norm{W^\prime}_0 R_2 \norm{Q_{[m+1,M]} \circ \left(\Conv_{\cW_m,\cB_m} - \Conv_{\cW_m^\prime,\cB_m^\prime}\right)\circ Q_{[1,m-1]}^\prime}_\infty\\
        &\stackrel{(c)}{\le} JDR_2 \left(\prod_{j=m+1}^{M}\prod_{i=1}^{L_j}J^{(i-1)}_jI^{(i)}_jR_1\right)\norm{\left(\Conv_{\cW_m,\cB_m} - \Conv_{\cW_m^\prime,\cB_m^\prime}\right)\circ Q_{[1,m-1]}^\prime}_\infty\\
        &\stackrel{(d)}{\le} JDR_2 \left(\prod_{j=m+1}^{M}\prod_{i=1}^{L_j}J^{(i-1)}_jI^{(i)}_jR_1\right)\left(\prod_{i=1}^{L_m}J^{(i-1)}_mI^{(i)}_mR_1 \norm{Q_{[1,m-1]}^\prime}_\infty\epsilon\right)\\
        &\stackrel{(e)}{\le} JDR_2\left(\prod_{j=m+1}^{M}\prod_{i=1}^{L_j}J^{(i-1)}_jI^{(i)}_jR_1\right)\left(\prod_{i=1}^{L_m}J^{(i-1)}_mI^{(i)}_mR_1\right)\\ & \qquad (1 \vee R_1)\left(\prod_{j=1}^{m}\prod_{i=1}^{L_j}J^{(i-1)}_jI^{(i)}_jR_1 \right)\left(1 + \sum_{k=1}^{m}L_k\prod_{i=1}^{L_k}(1 \vee J^{(i-1)}_kI^{(i)}_kR_1)\right)\epsilon\\
        &\le JDR_2\left(\prod_{j=1}^{M}\prod_{i=1}^{L_j}J^{(i-1)}_jI^{(i)}_jR_1\right)(1 \vee R_1)\left(1 + \sum_{k=1}^{M}L_k\prod_{i=1}^{L_k}(1 \vee J^{(i-1)}_kI^{(i)}_kR_1)\right)\epsilon,
    \end{align*}
    where (b) is by Proposition 7 from \citet{oono2019approximation}, (c) is by Proposition 2 and 4 from \citet{oono2019approximation}, (d) is by Proposition 2 and 5 from \citet{oono2019approximation}, and (e) is obtained by invoking \cref{prop:covering1}.
\end{proof}

\subsection{Proof of Lemma {\ref{lem:covering}}}\label{proof:covering}
\begin{proof}
    We grid the range of each parameter into subsets with width $\Lambda^{-1}_1\delta$, so there are at most $2(R_1 \vee R_2)\Lambda_1\delta^{-1}$ different subsets for each parameter. In total, there are $\left(2(R_1 \vee R_2)\Lambda_1\delta^{-1}\right)^{\Lambda_2}$ bins in the grid. For any $f,f^\prime\in \cF(M,L,J,I,R_1,R_2)$ within the same grid, by \cref{lem:F-uniform-bound}, we have $\norm{f-f^\prime}_\infty \le \delta$. We can construct an $\epsilon$-covering with cardinality $\left(2(R_1 \vee R_2)\Lambda_1\delta^{-1}\right)^{\Lambda_2}$ by selecting one neural network from each bin in the grid.

    Taking log and plugging in the network architecture parameters in \cref{thm:approx-error-Q,thm:approx-error-approxlip}, we have
    \begin{align*}
        \log \cN(\delta, \cF(M,L,J,I,R_1,R_2), \norm{\cdot}_\infty) &= O\left(\Lambda_2\log \left(\left(R_1 \vee R_2\right)\Lambda_1\delta^{-1}\right)\right)\\
        &\le O\left(\widetilde{M}DD^2\widetilde{J}^2\log{(\widetilde{M}\widetilde{J})}\log^2(\widetilde{M}\widetilde{J})\log^2(\widetilde{M}\widetilde{J})\log\frac{1}{\delta}\right)\\
        &= O\left(\widetilde{M}\widetilde{J}^2D^3\log^5(\widetilde{M}\widetilde{J})\log\frac{1}{\delta}\right),
    \end{align*}
    where the inequality is due to $\Lambda_2 = O(\widetilde{M}DD^2\widetilde{J}^2\log{(\widetilde{M}\widetilde{J})})$. By plugging in the choice of $R_1$ with sufficiently small integer $\tilde{J}$, we have $\rho = (1/2)^LM^{-1} \le M^{-1}$ and thus $\trho \leq (1+M^{-1})^M \le e$. Moreover, we have $\trho^+ = 1+ML$. 
\end{proof}

\section{Statistical Result of CNN Approximation}\label{sec:appendix-nonparametric}
In this section, we derive the statistical estimation error for using a CNN empirical risk minimizer to estimate an approximately Lipschitz ground truth function over an i.i.d. dataset. We need to choose the appropriate CNN architecture and size in order to balance the approximation error from \cref{thm:approx-error-approxlip} and the variance. This statistical estimation error can be decomposed into the error of using CNN to approximate the target function (\cref{thm:approx-error-approxlip}), terms that grow with the covering number of our CNN class, and the error of using the discrete covering to approximate our CNN class.

\begin{lemma}\label{lem:estimation-error-approxlip}
    Suppose \cref{asm:manifold} holds, $f_0\colon\cS\to\RR$ is a bounded $(L_f,\alpha,\epsilon_f)$-approximately Lipschitz function. We are given samples $\Xi_N=\{x_i,y_i\}_{i=1}^N$ where $x_i$'s are i.i.d. sampled from a distribution $\cD_x$ on $\cS$ and $y_i=f_0(x_i)+\zeta_i$. $\zeta_i$'s are i.i.d. sub-Gaussian random noise with variance proxy $\sigma^2$ and are uncorrelated with $x_i$'s. If $\epsilon_f=\mu D^{\frac{3\alpha}{2\alpha + d}}N^{-\frac{\alpha}{2\alpha+d}}$ for some constant $\mu=O(L_f+\norm{f_0}_\infty)$ and the estimator \[
        \hat{f}_N=\argmin_{f\in\cF_0}\frac{1}{N}\sum_{i=1}^N(f(x_i)-y_i)^2
    \]
    is computed with neural network function class $\cF_0=\cF_{\mathrm{Lip}}(A,L_0,\alpha, \epsilon_0)$ such that
    \begin{align*}
        &M=O(N^\frac{d}{d+2\alpha}), ~ L=O(\log N+D+\log D), J=O(D), ~ I\in[2,D], ~ A=\norm{f_0}_\infty, \\
        &R_1=O(1), ~ \log R_2=O(\log^2 N + D\log N), ~L_0=L_f, ~ \epsilon_0=D^{\frac{3\alpha}{2\alpha + d}}N^{-\frac{\alpha}{2\alpha+d}},
    \end{align*}
    then we have \[
        \EE_{\Xi_N}\left[\int_\cS \left(\hat{f}_N(x)-f_0(x)\right)^2\ud \cD_x(x)\right]\leq c_0 ((L_f + A)^2 + \sigma^2)N^{-\frac{2\alpha}{2\alpha+d}}\log^6N,
    \]
    where $c_0$ is some constant depending on $D^{\frac{6\alpha}{2\alpha+d}}$, $\log L_f$, $\log\norm{f_0}_\infty$, $d$, $\alpha$, $\omega$, $B$, and the surface area $\Area(\cS)$. $O(\cdot)$ hides some constant depending on $\log L_f$, $\log\norm{f_0}_\infty$, $d$, $\alpha$, $\omega$, $B$, and the surface area $\Area(\cS)$.
\end{lemma}

First, note that the nonparametric regression error can be decomposed into two terms:  
\begin{align*}
    &\EE_{\Xi_N}\left[\int_\cS \left(\hat{f}_N(x)-f_0(x)\right)^2\ud \cD_x(x)\right]\\
    =&\underbrace{2\EE_{\Xi_N}\left[\frac{1}{N}\sum_{i=1}^N (\hat{f}_N(x_i)-f_0(x_i))^2\right]}_{T_1}\\
    &\quad +\underbrace{\EE_{\Xi_N}\left[\int_\cS \left(\hat{f}_N(x)-f_0(x)\right)^2\ud \cD_x(x)\right]-2\EE_{\Xi_N}\left[\frac{1}{N}\sum_{i=1}^N (\hat{f}_N(x_i)-f_0(x_i))^2\right]}_{T_2},
\end{align*}
where $T_1$ reflects the squared bias of using CNN to approximate $f_0$, and $T_2$ is the variance term.

\subsection{Supporting Lemmas for Lemma {\ref{lem:estimation-error-approxlip}}}
We introduce two supporting lemmas from \citet{chen2022nonparametric} that show upper bounds for $T_1$ and $T_2$ in terms of the approximation error and covering number.
\begin{lemma}[{\citet[Lemma 5]{chen2022nonparametric}}]\label{lem:T1}
    Fix the neural network class $\cF_{\mathrm{Lip}}(A,L_0,\alpha,\epsilon_0)$. For any constant $\delta\in (0,2A)$, we have 
    \begin{align*}
        {T_1} \leq & 4 \inf_{f\in\cF_{\mathrm{Lip}}(A,L_0,\alpha,\epsilon_0)}\int_\cS (f(x)-f_0(x))^2 \ud \cD_x(x) \\ 
        &\quad + 48\sigma^2\frac{\log\cN(\delta, \cF_{\mathrm{Lip}}(A,L_0,\alpha,\epsilon_0), \norm{\cdot}_\infty) + 2}{N} \\ 
        &\quad + \left(8\sqrt{6}\sqrt{\frac{\log\cN(\delta, \cF_{\mathrm{Lip}}(A,L_0,\alpha,\epsilon_0), \norm{\cdot}_\infty) + 2}{N}} + 8\right)\sigma\delta,
    \end{align*}
    where $\cN(\delta, \cF_{\mathrm{Lip}}(A,L_0,\alpha,\epsilon_0), \norm{\cdot}_\infty)$ denotes the $\delta$-covering number of $\cF_{\mathrm{Lip}}(A,L_0,\alpha,\epsilon_0)$ with respect to the $L^\infty$ norm, that is, there exists a discretization of the class $\cF_{\mathrm{Lip}}(A,L_0,\alpha,\epsilon_0)$ with $\cN(\delta, \cF_{\mathrm{Lip}}(A,L_0,\alpha,\epsilon_0), \norm{\cdot}_\infty)$ distinct elements such that for any $f\in\cF$, there is a $\overbar{f}$ in the discretization satisfying $\norm{\overbar{f}-f}_\infty\leq\delta$.
\end{lemma}

\begin{lemma}[{\citet[Lemma 6]{chen2022nonparametric}}]\label{lem:T2}
    Fix the neural network class $\cF_{\mathrm{Lip}}(A,L_0,\alpha,\epsilon_0)$. For any constant $\delta\in (0,2A)$, we have 
    \begin{align*}
        T_2\leq& \frac{104A^2}{3N}\log\cN(\delta/4A, \cF_{\mathrm{Lip}}(A,L_0,\alpha,\epsilon_0), \norm{\cdot}_\infty) + \left(4+\frac{1}{2A}\right)\delta.
    \end{align*}
\end{lemma}

With \cref{lem:T1,lem:T2}, we can immediately prove \cref{lem:estimation-error-approxlip}.

\subsection{Proof of Lemma {\ref{lem:estimation-error-approxlip}}}
\begin{proof}
    Applying \cref{lem:T1,lem:T2} to the bias and variance decomposition, we derive 
    \begin{align}
        \EE_{\Xi_N}\left[\int_\cS \left(\hat{f}_N(x)-f_0(x)\right)^2\ud \cD_x(x)\right]
        &\leq 4 \inf_{f\in\cF_{\mathrm{Lip}}(A,L_0,\alpha,\epsilon_0)}\int_\cS (f(x)-f_0(x))^2 \ud \cD_x(x) \nonumber\\ 
        &\quad + 48\sigma^2\frac{\log\cN(\delta, \cF_{\mathrm{Lip}}(A,L_0,\alpha,\epsilon_0), \norm{\cdot}_\infty) + 2}{N} \nonumber\\ 
        &\quad + 8\sqrt{6}\sqrt{\frac{\log\cN(\delta, \cF_{\mathrm{Lip}}(A,L_0,\alpha,\epsilon_0), \norm{\cdot}_\infty) + 2}{N}}\sigma\delta \nonumber\\ 
        &\quad + \frac{104A^2}{3N}\log\cN(\delta/4A, \cF_{\mathrm{Lip}}(A,L_0,\alpha,\epsilon_0), \norm{\cdot}_\infty) \nonumber\\ 
        &\quad + \left(4+\frac{1}{2A}+8\right)\delta.\label{eq:nonparam-error}
    \end{align}
    
    By \cref{thm:approx-error-approxlip}, if we set $\Tilde{M}\Tilde{J}=\epsilon^{-\frac{d}{\alpha}}$ and choose $M,L,J,I,R_1,R_2,A$ such that  
    \begin{align*}
        &M=O(\epsilon^{-\frac{d}{\alpha}}), ~ L=O(\log{} (\epsilon^{-\frac{d}{\alpha}})+D+\log D), J=O(D), ~ I\in[2,D], \\
        &R_1=O(1), ~ \log R_2=O(\log^2 (\epsilon^{-\frac{d}{\alpha}}) + D\log{(\epsilon^{-\frac{d}{\alpha}})}), ~ A=\norm{f_0}_\infty, 
    \end{align*}
    for some $\epsilon\in(0,1)$, then there exists an $f\in\cF_{\mathrm{Lip}}(A,L_f,\alpha,\epsilon)$ such that 
    \begin{align*}
        \norm{f-f_0}_\infty\leq (L_f+A)\epsilon+2\epsilon_f. 
    \end{align*}

    Since $\cF_{\mathrm{Lip}}(A,L_f,\alpha,\epsilon)\subseteq\cF(M,L,J,I,R_1,R_2)$, we have \[
        \cM\left(\delta,\cF_{\mathrm{Lip}}(A,L_f,\alpha,\epsilon), \norm{\cdot}_\infty\right)
        \leq 
        \cM\left(\delta,\cF(M,L,J,I,R_1,R_2), \norm{\cdot}_\infty\right),
    \] where $\cM$ denotes the packing number. Combining the relation between covering and packing numbers that $\cN(\delta,\cF)\leq\cM(\delta,\cF)\leq\cN(\delta/2,\cF)$, we have \[
        \cN\left(\delta,\cF_{\mathrm{Lip}}(A,L_f,\alpha,\epsilon), \norm{\cdot}_\infty\right)
        \leq 
        \cN\left(\delta/2,\cF(M,L,J,I,R_1,R_2), \norm{\cdot}_\infty\right).
    \]
    By \cref{lem:covering}, we have 
    \begin{align*}
        \log \cN(\delta^\prime, \cF(M,L,J,I,R_1,R_2) 
        &= O\left(\widetilde{M}\widetilde{J}^2D^3\log^5(\widetilde{M}\widetilde{J})\log\frac{1}{\delta^\prime}\right) \\ 
        &= O\left(\epsilon^{-\frac{d}{\alpha}}D^3\log^5(\epsilon^{-\frac{d}{\alpha}})\log\frac{1}{\delta^\prime}\right). 
    \end{align*}

    Plugging the results back in \eqref{eq:nonparam-error}, we get 
    \begin{multline}
        \EE_{\Xi_N}\left[\int_\cS \left(\hat{f}_N(x)-f_0(x)\right)^2\ud \cD_x(x)\right]\\
        \leq \Tilde{O}\biggl(\left((L_f+A)\epsilon + 2\epsilon_f\right)^2 + \frac{A^2 + \sigma^2}{N}\epsilon^{-\frac{d}{\alpha}}D^3\log^5(\epsilon^{-\frac{d}{\alpha}})\log\frac{1}{\delta} \\ 
        \quad + \sqrt{\frac{\epsilon^{-\frac{d}{\alpha}}D^3\log^5(\epsilon^{-\frac{d}{\alpha}})\log\frac{1}{\delta}}{N}}\sigma\delta + \sigma\delta + \frac{\sigma^2}{N}\biggr).\label{eq:nonparam-error-final}
    \end{multline}

    Finally we choose $\epsilon$ to satisfy $\epsilon^2=D^3N^{-1}\epsilon^{-\frac{d}{\alpha}}$, which gives $\epsilon=D^{\frac{3\alpha}{2\alpha + d}}N^{-\frac{\alpha}{2\alpha+d}}$. It suffices to pick $\delta=\frac{1}{N}$. Since $\epsilon_f=\mu D^{\frac{3\alpha}{2\alpha + d}}N^{-\frac{\alpha}{2\alpha+d}}$ with $\mu=O(L_f+\norm{f_0}_\infty)$, we have $T_1=O(T_2)$, that is, the bias term is dominated by the variance term. Therefore, by substituting both $\epsilon$ and $\delta$ in \eqref{eq:nonparam-error-final}, we get the estimation error bound 
    \begin{align*}
        \EE_{\Xi_N}\left[\int_\cS \left(\hat{f}_N(x)-f_0(x)\right)^2\ud \cD_x(x)\right]
        \leq c_0 ((L_f + A)^2+\sigma^2)N^{-\frac{2\alpha}{2\alpha+d}}\log^6 N,
    \end{align*}
    where $c_0$ is a constant depending on $D^{\frac{6\alpha}{2\alpha+d}}$, $\log L_f$, $\log \norm{f_0}_\infty$, $d$, $\alpha$, $\omega$, $B$, and the surface area $\Area(\cS)$.
\end{proof}









\bibliography{ref}

\end{document}